\documentclass[twoside]{article}

\usepackage[accepted]{aistats2026}
\usepackage{times}  
\usepackage{helvet}  
\usepackage{courier}  
\usepackage[hyphens]{url}  
\usepackage{graphicx} 
\usepackage{natbib}  
\usepackage{caption} 
\usepackage{algorithm}
\usepackage{algorithmic}

\usepackage{booktabs} 
\usepackage{tikz} 

\newcommand{\removed}[1]{}
\usepackage{times}
\usepackage{soul}
\usepackage{url}
\usepackage{hyperref}
\usepackage[utf8]{inputenc}
\usepackage{graphicx}
\usepackage{amsmath}
\usepackage{amsthm}
\usepackage{booktabs}
\usepackage{algorithm}
\usepackage{algorithmic}
\usepackage{stackengine}
\def\defeq{\mathrel{\ensurestackMath{\stackon[1pt]{=}{\scriptscriptstyle\Delta}}}}

\usepackage{algorithm}
\usepackage{algorithmic}

\usepackage{times}

\usepackage{url}

\usepackage{amsmath}

\usepackage{amssymb}
\usepackage{mathtools}
\usepackage{amsthm}

\usepackage{algorithmic}

\usepackage{lscape}
\usepackage[capitalize,noabbrev]{cleveref}

\theoremstyle{plain}

\usepackage[textsize=tiny]{todonotes}
\usepackage{multirow}

\usepackage{ascmac}
\usepackage{float}
\usepackage{perpage}
\MakeSorted{figure}
\MakeSorted{table}

\usepackage{url}
\usepackage{natbib}
\usepackage{chapterbib}

\usepackage{color}
\usepackage{tikz}
\tikzset{%
mynode/.style={circle,minimum width=.5ex, fill=none,draw}, 
myfillnode/.style={circle,minimum width=.5ex, fill=lightgray,draw}, 
}

\usepackage{pgfplots}
\pgfplotsset{compat=newest}

\usepackage{amssymb}
\usepackage{natbib}

\newcommand{\indep}{\perp \!\!\! \perp}
\usepackage{amsmath}               
\usepackage{lscape}
\usepackage{algorithm}
\usepackage{color}
\usepackage{tikz}
\usepackage{amsmath,amsthm}
\newtheorem{theorem}{Theorem}

\newtheorem{assumption}{Assumption}
\newtheorem{lemma}{Lemma}
\newtheorem{proposition}{Proposition}

\usepackage{multirow}
\usepackage{comment}
\usepackage{here}
\allowdisplaybreaks[4]
\usepackage{bbm}
\usepackage{caption}
\usepackage{bbding}
\usepackage{arydshln}
\usepackage{afterpage}

\usepackage{enumitem}

\usepackage{mathrsfs}

\usepackage{soul}

\usepackage{float}

\usepackage{amsmath}
\usepackage{amsthm}
\usepackage{amssymb}
\usepackage{mathtools}

\usepackage{xcolor}

\begin{document}

%
\runningtitle{Bounds and Identification of Joint POs/OVs Probabilities under Monotonicity Assumptions}

%

\runningauthor{Naoya Hashimoyo, Yuta Kawakami, Jin Tian}

\twocolumn[

\aistatstitle{Bounds and Identification of Joint Probabilities of Potential Outcomes and Observed Variables under Monotonicity Assumptions}

\aistatsauthor{ Naoya Hashimoto\NoHyper\footnotemark\endNoHyper \And Yuta Kawakami\NoHyper\footnotemark[\value{footnote}]\endNoHyper \And  Jin Tian }

\aistatsaddress{Mohamed bin Zayed University of Artificial Intelligence}]

\footnotetext{Equal contribution.}

\begin{abstract}

Evaluating  joint probabilities of potential outcomes and observed variables, and their linear combinations, is a fundamental challenge in causal inference. This paper addresses the bounding and identification of these probabilities in settings with discrete treatment and discrete ordinal outcome. We propose new families of monotonicity assumptions  and formulate the bounding problem as a linear programming problem. We further introduce a new monotonicity assumption specifically to achieve identification. 
Finally, we present numerical experiments to validate our methods and demonstrate their application using real-world datasets.
\end{abstract}


\section{INTRODUCTION}

\begin{table*}[tb]
    \centering
    \caption{Summary of monotonicity assumptions (MAs) in causal inference. The symbol ($\checkmark$) indicates that the objective quantity is identified under the MA.
    $\lambda({\cal P})$ represents an arbitrary linear combination of the joint POs/OVs probabilities.}
    \label{tab:my_label1}
    \vspace{-0.2cm}
    \scalebox{0.9}{
    \begin{tabular}{c|ccccc}
    \hline
        Research &  Treatment Type &  Outcome Type & MA  &  Objective & Identification \\
        \hline\hline
        \citet{Pearl1999}; \citet{Tian2000} & Binary &  Binary &   $\mathbb{P}(Y_0=1,Y_1=0)=0$ & PoC &  $\checkmark$ \\
        \citet{Mueller2023} & Binary &  Binary & $0\leq \mathbb{P}(Y_0=1,Y_1=0)\leq \epsilon$ & PoC &   \\
        \citet{Zhang2025} & Binary &  Ordinal & $Y_0 \leq Y_1$ & PN &   \\
        \citet{Zhang2025} & Binary &  Ordinal & $0\leq Y_1 - Y_0 \leq 1$ & PN & $\checkmark$  \\
        \citet{Manski1997} & Discrete &  General & $Y_0\leq Y_1 \leq \dots \leq Y_{d_X-1}$ & ACE &   \\
        \hline
        This paper & Discrete &  Ordinal & Assumptions~\ref{weak_mono-bi} $\sim$  \ref{assum_mono_with_prob_s} & $\lambda({\cal P})$  &   \\
        This paper & Discrete &  Ordinal & Assumption \ref{MITE_D} & $\lambda({\cal P})$  &  $\checkmark$ \\
         \hline
    \end{tabular}
    }
\end{table*}

One of the fundamental challenges in causal inference is the evaluation of joint probabilities of potential outcomes (POs) and observed variables (OVs) and their linear combinations.
Let the treatment ($X$) take values in $\{0, 1, \ldots, d_X-1\}$ and the outcome ($Y$) take values in $\{0, 1, \ldots, d_Y-1\}$, and let $Y_x$ denote the PO under treatment $x$. 
The joint POs/OVs probabilities $\mathbb{P}(Y_0=y_0,Y_1=y_1,\dots ,Y_{d_X-1}=y_{d_X-1},X=x,Y=y)$
 and their linear combinations 
encompass various causal quantities of practical interest, including, among others,   
(i) the probabilities of causation (PoC) \citep{Pearl1999,Tian2000,ALi2024}, which measure the necessity and sufficiency of the treatment;
(ii) the posterior causal effects \citep{Lu2022}, which measure causal effects given evidence on post-treatment variables; and
(iii) the moments of causal effects \citep{Kawakami2025_moments}, which capture the shape of the distribution of causal effects.

However, the joint POs/OVs probabilities 
are not identifiable from observational data without additional assumptions; while bounds can be derived, they are often too wide to be informative.
Their evaluation positions the problem at the highest third 
layer of Pearl’s causal hierarchy \citep{Pearl2018,Bareinboim2022}.
To address this issue, various monotonicity assumptions (MAs) on POs have been actively studied.

\citet{Pearl1999} and \citet{Tian2000} 
provided an identification theorem for PoC under the MA that no subjects satisfy $(Y_0=1,Y_1=0)$ (combined with the exogeneity assumption)  when the treatment and outcome are binary.
This MA is a natural and reasonable assumption in a wide range of empirical studies \citep{Hernan2024}.
In the instrumental variable (IV) literature, this MA is often invoked to describe the relationship between the binary IV ($Z$) and treatment ($X$) for the identification of local average treatment effects (LATE) \citep{Imbens1994,Angrist1995}. 
Recently, \citet{Mueller2023,Mueller2025} 
introduced a probability-limited MA, 
termed $\epsilon$-limited harm, 
that allows subjects with $(Y_0=1,Y_1=0)$ with probability at most $\epsilon$, under which 
they provided closed-form bounds for 
PoC.

For binary treatment and discrete ordinal outcome, \citet{Zhang2025} studied the bounds for the probability of necessity (PN) under an extended MA for discrete outcome, $Y_0\leq Y_1$, and further proposed a new increment-limited  
MA $0 \leq Y_1-Y_0 \leq 1$, termed the monotonic incremental treatment effect (MITE), for identifying PN. 
For discrete treatment settings, \citet{Manski1997} has introduced ``Monotone Treatment Response'' (MTR) assumption ($Y_0\leq Y_1 \leq \dots \leq Y_{d_X-1}$) to bound the average causal effect (ACE) $\mathbb{E}[Y_i-Y_j]$.

In this paper, we introduce new families of MAs 
in settings with discrete treatment and ordinal outcome.
Our work unifies the existing MAs 
into a single general MA, which accommodates multiple probabilistic and increment-limited relationships.
Table \ref{tab:my_label1} summarizes the MAs in the literature. 

Under our MAs, we address the problem of bounding the joint POs/OVs probabilities and their linear combinations.
We formulate the bounding problem using a linear programming (LP) approach \citep{Balke1995,Balke1997,Tian2000,Duarte2024},  incorporating new linear constraints implied by our new MAs.
We discuss the estimation of the bounds from finite-sample data and show numerical experiments to validate our method.

We further provide an MA assumption for identifying joint POs/OVs probabilities (and their linear combinations) from observational or experimental data. We show numerical experiments to validate their estimation from finite-sample data.

Finally, we demonstrate the application of our methods on real-world datasets.

\section{MOTIVATING EXAMPLE}
\label{sec2}

\label{sec: motivation}

We motivate our study using a real-world example.

\textbf{Dataset.} 
We consider an open-access dataset on student performance in two Portuguese schools, available at (\url{https://archive.ics.uci.edu/dataset/320/student+performance}). Students in secondary education are tested annually over three years, yielding three total assessments.
The observational dataset includes attributes related to demographics, social factors, school-related features, and student academic performance.
The sample size is 649, with no missing values. 

\textbf{Variables.}
We define study hours as the treatment variable $X$, where $X=0,1,2,3$ denotes studying less than 2 hours,  2 to 5 hours, 5 to 10 hours, or more than 10 hours weekly, respectively.
We consider the Portuguese language score in the final period  as the outcome variable ($Y$), and categorize it as $Y=0$ if the score is between 0 and 8, $Y=1$ if the score is between 9 and 15, and $Y=2$ if the score is between 16 and 20, corresponding to low, medium, and high score categories, respectively.

To assess the causal effect of study hours on student performance, existing work typically computes ACEs, such as $\mathbb{E}[Y_1-Y_0]$, $\mathbb{E}[Y_2-Y_1]$.
However, ACEs capture only one aspect of the causal relationship.
Researchers are often interested in more  nuanced causal questions that necessitate the use of the joint POs/OVs probabilities. 
For example, consider the following two causal questions:
\vspace{-0.1cm}
\begin{center}
``{\it 
For a student who studied more than 10 hours and obtained a  score of more than 15 in reality, would their score be less than 16 had they studied less than 10 hours?}"
\end{center}
\vspace{-0.1cm}
This is a situation in which studying more than 10 hours is necessary to obtain a score above 15.
\vspace{-0.1cm}
\begin{center}
``{\it 
For a student who studied less than 2 hours and obtained a  score of no more than 8 in reality, would their score be more than 8 had they studied at least 2 hours?}"
\end{center}
\vspace{-0.1cm}
This is a situation in which studying at least 2 hours is sufficient to obtain a score above 8.

Probabilities representing such two situations can be formally expressed in terms of the joint POs/OVs probabilities as follows, respectively:
\begin{equation}\label{eq-question}
\begin{gathered}
\mathbb{P}(Y_0\leq 1,Y_1\leq 1,Y_2\leq 1 | X=3,Y=2),\\
\mathbb{P}(Y_1\geq 1,Y_2\geq 1,Y_3\geq 1 | X=0,Y=0).
\end{gathered}
\end{equation}

However, counterfactual probabilities in \eqref{eq-question} are not identifiable from observational data. In this paper, we show how to derive bounds for these probabilities using linear programming. We illustrate that the bounds can be tightened by incorporating plausible MAs based on domain knowledge. For example, one might assume that ``a student would obtain a higher score had they studied 2 to 5 hours weekly rather than less than 2 hours" ($Y_0\leq Y_1$) or ``99.5\% of the students would have obtained a higher score had they studied more than 10 hours weekly rather than 5 to 10 hours" ($\mathbb{P}(Y_2\leq Y_3)\geq 0.995$). Further, we show that the probabilities in \eqref{eq-question} are identifiable if we impose stronger MAs $0\leq Y_i-Y_j\leq 1$ for all $i>j$.

\section{NOTATIONS AND BACKGROUND}

We represent each variable with a capital letter $(X)$ and its realized value with a small letter $(x)$.
Let $\mathbb{I}(x)$ be an indicator function that takes $1$ if $x$ is true and $0$ if $x$ is false.
The symbol $\land$ represents logical conjunction.

{\bf Structural Causal Models.} 
We use the language of Structural Causal Models (SCM) as our basic semantic and inferential framework \citep{Pearl09}.
An SCM ${\cal M}$ is a tuple $\left<{\boldsymbol U},{\boldsymbol V}, {\cal F}, \mathbb{P}_{\boldsymbol U} \right>$, where ${\boldsymbol U}$ is a set of exogenous (unobserved) variables following joint probabilities $\mathbb{P}_{\boldsymbol U}$, and ${\boldsymbol V}$ is a set of endogenous (observable) variables whose values are determined by structural functions ${\cal F}=\{f_{V_i}\}_{V_i \in {\boldsymbol V}}$ such that $v_i:= f_{V_i}({\mathbf{pa}}_{V_i},{\boldsymbol u}_{V_i})$ where ${\mathbf{PA}}_{V_i} \subseteq {\boldsymbol V}$ and $U_{V_i} \subseteq {\boldsymbol U}$. 
An atomic intervention of setting a set of endogenous variables $X$ to constants $x$, denoted by $do(x)$, replaces the original equations of $X$ by $X :=x$ and induces a \textit{sub-model}  ${\cal M}_x$.
We denote the potential outcome~(PO) $Y$ under intervention $do({x})$ by $Y_{{x}}({\boldsymbol u})$, which is the solution of $Y$ with ${\boldsymbol U}={\boldsymbol u}$ in the sub-model ${\cal M}_{x}$.

In this paper, we consider a discrete treatment variable $X$ and an \emph{ordinal} outcome variable $Y$. $Y$ must be ordinal since monotonicity assumptions presuppose an ordering among the POs. By contrast, $X$ need not be ordinal.

{\bf Assumptions in Previous Works.}
The exogeneity  is a commonly used assumption stated as follows:
\begin{assumption}[Exogeneity]
\label{exo}
$Y_x \indep X$ for any $x$.
\end{assumption}
Assumption \ref{exo} implies $\mathbb{P}(Y_x)=\mathbb{P}(Y|X=x)$ for any $x$.

Let $X$ and $Y$ be discrete random variables with supports $\{0,1,\dots,d_X-1\}$ and $\{0,1,\dots,d_Y-1\}$, respectively.
We list the MAs used in previous works.
\begin{assumption}[Monotonicity for Binary Treatment and Outcome]\citep{Pearl1999,Tian2000}
\label{monobin_s}
$\mathbb{P}(Y_0=1,Y_1=0)=0$.
\end{assumption}
\begin{assumption}[Monotonicity for Discrete Outcome]\citep{Zhang2025}
\label{mono_dc}
$Y_0 \leq Y_1$, almost surely (\text{a.s.}).
\end{assumption}
\begin{assumption}[Monotonic Incremental Treatment Effect (MITE)]\citep{Zhang2025}
\label{mono_it}
$0 \leq Y_1 - Y_0 \leq 1, \text{ a.s.}$
\end{assumption}
\begin{assumption}[$\epsilon$-limited harm]\citep{Mueller2023}
\label{monobin_w}
$0\leq \mathbb{P}(Y_0=1,Y_1=0)\leq \epsilon$.
\end{assumption}
\begin{assumption}[Monotone Treatment Response (MTR)]\citep{Manski1997}
\label{monoAS}
$Y_0\leq Y_1 \leq \dots \leq Y_{d_X-1}, \text{ a.s.}$
\end{assumption}
We note that Assumptions \ref{monobin_s} and \ref{monobin_w} apply to binary treatment and binary outcome, while Assumptions \ref{mono_dc} and \ref{mono_it} apply to binary treatment. 
A detailed discussion of these MAs is provided in Appendix \ref{app_mono}.

\section{MONOTONICITY ASSUMPTIONS FOR DISCRETE TREATMENT AND OUTCOME}

\label{sec:mono_assumptions}

We introduce new families of MAs by extending the existing MAs to discrete treatment and discrete outcome settings and generalizing them to accommodate more complex PO relationships.

{\bf MAs for Binary Treatment and Discrete Outcome.}
We introduce a new assumption for binary treatment that captures both probability-limited and increment-limited forms of monotonicity.
\begin{assumption}[Probability-limited and Increment-limited Monotonicity for Binary Treatment]
\label{weak_mono-bi}
There exist known constants ${\underline{D}}$, ${\overline{D}}$, $L$, and $U$,
such that the following condition holds:
\begin{equation}
L\leq \mathbb{P}\left({\underline{D}} \leq Y_1 - Y_0 \leq {\overline{D}} \right)\leq U.
\end{equation}
\end{assumption}
These values ${\underline{D}}$, ${\overline{D}}$, $L$, and $U$ are pre-specified by the researchers as prior knowledge. 
This MA generalizes 
Assumption \ref{mono_it} by allowing both the increment condition and the probability condition simultaneously.
It reduces to Assumption \ref{mono_dc} when $L=U=1$, ${\underline{D}}=0$, ${\overline{D}}=\infty$, and equals Assumption \ref{mono_it} when $L=U=1$, ${\underline{D}}=0$, ${\overline{D}}=1$.

{\bf MAs for Discrete Treatment and Outcome.}
Next, we present MAs for  discrete treatment and discrete outcome.

We introduce a relaxation of 
Assumption \ref{monoAS} by allowing for probability-limited restrictions.
\begin{assumption}[Probability-limited MTR]
\label{monoLB}
There exist known constants  $L$ and $U$, such that
\begin{equation}
L \leq \mathbb{P}(Y_0\leq Y_1 \leq \dots \leq Y_{d_X-1})\leq U.
\end{equation}
\end{assumption}
This condition implies that the POs for each subject increase monotonically with the treatment level  with probability in $[L,U]$.
It coincides with Assumption \ref{monoAS} when $L=U=1$.

We introduce a relaxation of the MITE assumption (Assumption \ref{mono_it})  that allows for increment-limited restrictions in the case of discrete treatments.
\begin{assumption}[Increment-limited Monotonicity for Discrete Treatment]
\label{monomatAS}
There exist known constants  ${\underline{D}}_{st}$ and ${\overline{D}}_{st}$ for $s,t=0,\dots,d_X-1$,
such that the following condition holds:
\begin{equation}
\bigwedge_{s,t=0,\dots,d_X-1;s>t}({\underline{D}}_{st}\leq Y_s- Y_t\leq {\overline{D}}_{st}), \text{ a.s.}
\end{equation}
\end{assumption}
This condition requires that the difference between any two potential outcomes $Y_s$ and $Y_t$ lies within the interval $[{\underline{D}}_{st},{\overline{D}}_{st}]$ a.s. whenever $s>t$. 
Note that we can avoid imposing a restrictive condition on $Y_s - Y_t$ by setting $\underline{D}_{st}=-\infty$ and $\overline{D}_{st}=\infty$.

We relax Assumption~\ref{monomatAS} with probability-limited restrictions.
\begin{assumption}[Probability-limited and Increment-limited Monotonicity with a Single Term]
\label{monomatASS}
There exist known constants  $L$, $U$, and ${\underline{D}}_{st}$ and ${\overline{D}}_{st}$ for $s,t=0,\dots,d_X-1$, 
such that the following condition holds:
\begin{equation}
L \leq \mathbb{P}\left(\bigwedge_{s,t=0,\dots,d_X-1;s>t}({\underline{D}}_{st}\leq Y_s- Y_t\leq {\overline{D}}_{st})\right)\leq U.
\end{equation}
\end{assumption}
The above condition bounds the probability that 
the difference between any two potential outcomes $Y_s$ and $Y_t$ with $s>t$ lies within the interval $[{\underline{D}}_{st},{\overline{D}}_{st}]$ by $[L,U]$.
Assumption \ref{monomatAS} is the special case with $L=U=1$.

Researchers may impose multiple probability-limited and increment-limited MAs simultaneously, e.g.,
$0.5 \leq \mathbb{P}(0\leq Y_1-Y_0\leq 1)$, $0.5 \leq \mathbb{P}(0\leq Y_2-Y_1\leq 1)$, and $0.5 \leq \mathbb{P}(0\leq Y_2-Y_0\leq 1)$, each in the form of Assumption \ref{monomatASS}.  
For convenience, we introduce the following single MA to represent multiple simultaneous requirements in the form of Assumption \ref{monomatASS}:

\begin{assumption}[Probability-limited and Increment-limited Monotonicity with Multiple Terms]
\label{assum_mono_with_prob_s}
There exist known constants  ${\underline{D}}_{st}^{(w)}$ and ${\overline{D}}_{st}^{(w)}$ for $s,t=0,\dots,d_X-1$ and $w=1,\dots,W$, and $L^w$ and $U^w$ for $w=1,\dots,W$,
such that the following condition holds, for $w=1,\dots,W$:

{\vspace{-0.5cm}\small
\begin{equation}
 L^w  \leq \mathbb{P}\left(\bigwedge_{s,t=0,\dots,d_X-1;s>t}({\underline{D}}_{st}^{(w)}\leq Y_s- Y_t\leq {\overline{D}}_{st}^{(w)})\right) \leq U^w.  
\end{equation}
}
\end{assumption}
In Assumption~\ref{assum_mono_with_prob_s}, $W$ represents the number of simultaneous monotonicity conditions imposed.

{\bf Relationships among MAs.}
\begin{proposition}
\label{prop1}
We have the following  relationships among the MAs presented thus far:
(1) Assumptions~\ref{monobin_s} through~\ref{monomatASS} are all special cases of Assumption~\ref{assum_mono_with_prob_s}. 
(2) Assumption \ref{mono_it} is  a special case of Assumptions \ref{weak_mono-bi}, \ref{monomatAS}, and \ref{monomatASS}. 
(3) Assumption \ref{monobin_w} is  a special case of Assumptions \ref{weak_mono-bi}, \ref{monoLB}, and \ref{monomatASS}. 
(4) Assumption \ref{monoAS} is  a special case of Assumptions \ref{monoLB}, \ref{monomatAS}, and \ref{monomatASS}.
\end{proposition}
Assumption~\ref{assum_mono_with_prob_s} encompasses existing MAs in the literature and represents the most general class of monotonicity among all the MAs considered in this paper.

\section{BOUNDING WITH MONOTONICITY ASSUMPTIONS}

\label{sec5}

In this section, we formulate the bounding problem for joint POs/OVs probabilities given observed distributions as a linear programming (LP), incorporating additional MAs. 

We denote $[d]\coloneqq \{0,\dots ,d-1\}$,
$Y_{[d_X]}=(Y_0,\dots ,Y_{d_X-1})$, and  $y_{[d_X]}=(y_0,\dots ,y_{d_X-1})$.

\subsection{Objective of the Study}
\label{subsec:objective_of_study}

We denote the set of the joint POs/OVs probabilities as 
\begin{equation}
\begin{aligned}
    &{\cal P}\defeq\{\mathbb{P}(Y_{[d_X]}=y_{[d_X]},X=x,Y=y):  
    \\
    & \hspace{10mm}y_{[d_X]}\in [d_Y]^{d_X},x\in [d_X],y\in [d_Y]\} .
\end{aligned}
\end{equation}

The objective of this study is a linear combination of the joint POs/OVs probabilities, expressed as
$\lambda({\cal P})$,
where $\lambda$ is an arbitrary linear function mapping the parameters ${\cal P}$ to $\mathbb{R}$.
$\lambda({\cal P})$ encompasses a wide range of causal quantities studied in the literature, including the PoC families in \citep{ALi2024}, the moments of causal effect \citep{Kawakami2025_moments}, and the posterior causal effects \citep{Lu2022}. A detailed illustration is provided in Appendix \ref{app AB}.

\subsection{Bounding Problem Setup}
\label{subsec:setup_body}
We set up the bounding problem by LP.  
We denote the following parameters
\begin{equation}\label{eq-para}
p_{y_{[d_X]},x}
=\mathbb{P}(Y_{[d_X]}=y_{[d_X]},X=x).
\end{equation}
We denote the set of parameters as $\mathfrak{p}\defeq\{p_{y_{[d_X]},x}:y_{[d_X]}\in [d_Y]^{d_X},x\in [d_X]\}$.  
The total number of parameters is $d_Y^{d_X}d_X$. 
The number of parameters grows exponentially in $d_X$ and polynomially in $d_Y$. 
Note that we do not include $Y$ in the parameterization of the bounding problem 
since we have
$\mathbb{P}(Y_{[d_X]}=y_{[d_X]},X=x,Y=y)=\mathbb{P}(Y_{[d_X]}=y_{[d_X]},X=x)\mathbb{I}(y_x=y)$
from the counterfactual consistency property \citep{Pearl09}, i.e., 
$X=x \Rightarrow Y_x =Y$.
Then, $\lambda({\cal P})$ can be represented as a linear combination of the parameters in $\mathfrak{p}$  and denote it $\phi(\mathfrak{p})$. 

{\bf Constraints from Observed Distributions.}
The general constraints 
are  
\begin{equation}
\begin{aligned}
\label{eqsum}
&p_{y_{[d_X]},x} \geq 0\text{ for all } y_{[d_X]}\text{ and }x,\\
&\sum_{y_0=0}^{d_Y-1}\dots\sum_{y_{d_X-1}=0}^{d_Y-1}\sum_{x=0}^{d_X-1}p_{y_{[d_X]},x}=1.
\end{aligned}
\end{equation}

Next, we present the constraints imposed by three types of distributions:
(i) experimental distribution only, (ii) observational distribution only, and (iii) both experimental and observational distributions.
First, if we have the experimental distribution $\mathbb{P}(Y_x)$, it imposes 
\begin{equation}
\label{exp.constraints}
\begin{aligned}
    &\sum_{y_0=0}^{d_Y-1}\dots\sum_{y_{k-1}=0}^{d_Y-1}\sum_{y_{k+1}=0}^{d_Y-1}\dots\sum_{y_{d_X-1}=0}^{d_Y-1}\sum_{x=0}^{d_X-1}\\
   &\hspace{10mm} p_{y_0,\dots y_{k-1},j,y_{k+1}\dots ,y_{d_X-1},x}
    =\mathbb{P}(Y_k=j)
\end{aligned}
\end{equation}
for $k\in [d_X]$ and $j\in [d_Y-1]$.
We do not include $j=d_Y-1$ since Eq. \eqref{exp.constraints} for $j=d_Y-1$ can be expressed as a linear combination of Eqs. \eqref{eqsum} and \eqref{exp.constraints}.

Next, if we have the observational distribution $\mathbb{P}(X,Y)$, it imposes 
\begin{equation}
\label{Obs.constraints}
\begin{aligned}
    &\sum_{y_0=0}^{d_Y-1}\dots\sum_{y_{l-1}=0}^{d_Y-1}\sum_{y_{l+1}=0}^{d_Y-1}\cdots \sum_{y_{d_X-1}=0}^{d_Y-1}\\
    & \hspace{10mm}p_{y_0,\dots ,y_{l-1},m,y_{l+1},\dots ,y_{d_X-1}, l}
    =\mathbb{P}(X=l,Y=m)
\end{aligned}
\end{equation}
for $l\in[d_X]$,  $m\in[d_Y]$, $(l,m)\ne (d_X-1,d_Y-1)$. We excluded $(l,m)=(d_X-1,d_Y-1)$ since Eq. \eqref{Obs.constraints} in that case can be expressed as a linear combination of Eqs. \eqref{eqsum} and \eqref{Obs.constraints}. 

Finally, when we have both experimental and observational distributions, they impose the constraints in both Eqs.~\eqref{exp.constraints} and \eqref{Obs.constraints}.

All of the above constraints from distributions are linear in $\mathfrak{p}$.

{\bf Remark.}
{Introducing exogeneity  (Assumption \ref{exo}) can be helpful in deriving tighter bounds. $Y_x \indep X$ imposes the constraints $\mathbb{P}(Y_x=y_x,X=l)=\mathbb{P}(Y_x=y_x)\mathbb{P}(X=l)$ on $p_{y_{[d_X]},x}$, which are linear given either observational or experimental distributions. }

\textbf{Additional Linear Constraints by MAs.}
Next, we present the additional linear constraints imposed by our introduced MAs.

\begin{proposition}
\label{prop2}
The constraints implied by Assumption~\ref{assum_mono_with_prob_s} are given by
\begin{equation}
\begin{aligned}
\label{eq: mono_prop_p}
&L^w\leq  \sum\nolimits_{y_0=0}^{d_Y-1}\dots \sum\nolimits_{y_{d_X-1}=0}^{d_Y-1}\sum\nolimits_{x=0}^{d_X-1} \\
&\hspace{0cm}\left(\prod_{s,t\in [d_X];s>t}\mathbb{I}({\underline{D}}_{st}^{(w)}\leq y_s- y_t\leq {\overline{D}}_{st}^{(w)})\right)p_{y_{[d_X]},x}\leq U^w
\end{aligned}
\end{equation}
for $w=1,\dots,W$.
The above constraints are all linear constraints in ${\cal P}$.
\end{proposition}
The explicit constraints implied by 
{Assumptions \ref{monoAS}-\ref{monomatASS}} are provided in Appendix~\ref{APPsubsec:LP_formulation_nonmatrix}.

{\bf Bounding Problems Under MA.}
Under Assumption \ref{assum_mono_with_prob_s}, the bounds of $\lambda({\cal P})$ are obtained by solving the following constrained optimization problems:
\begin{equation}\label{eq-LP}
\begin{gathered}
\underline{\lambda}({\cal P}) = \min_{\mathfrak{p}} \phi(\mathfrak{p}),\ \ 
\overline{\lambda}({\cal P}) = \max_{\mathfrak{p}} \phi(\mathfrak{p}),\\
\text{subject to the constraints from observed distributions }\\
\text{and Eq.~\eqref{eq: mono_prop_p}} 
\end{gathered}
\end{equation}
The resulting interval
$[\underline{\lambda}({\cal P}),\overline{\lambda}({\cal P})]$
represents the bounds of $\lambda({\cal P})$ under Assumption \ref{assum_mono_with_prob_s}.
These bounds are guaranteed to be sharp under Assumption \ref{assum_mono_with_prob_s}. The bound $[\underline{\lambda}({\cal P}),\overline{\lambda}({\cal P})]$ is said to be 
\emph{sharp}  if there exist  $\underline{\mathfrak{p}},\overline{\mathfrak{p}}\in \mathcal{P}$
that satisfy the imposed constraints and achieve 
$\phi(\underline{\mathfrak{p}})=\underline{\lambda}({\cal P})$ and $\phi(\overline{\mathfrak{p}})=\overline{\lambda}({\cal P})$.  
We provide the explicit formulation of the LPs in Appendix~\ref{Appsubsec:LP_in_matrix_form}.

Therefore, we have the following theorem:
\begin{theorem}
\label{thoe_lp}
{
Given observed distributions (experimental $\mathbb{P}(Y_x)$, observational $\mathbb{P}(X,Y)$, or both),   under Assumption \ref{assum_mono_with_prob_s}, 
solving the  LP problem in \eqref{eq-LP} yields sharp bounds for $\lambda({\cal P})$.}
\end{theorem}

{\bf Remark.}
There may be a compatibility issue in LPs, that is, the constraints implied by the MAs 
conflict with those supported by the observed distributions. 
If the feasible region of the LP—that is, the set of all points satisfying all constraints simultaneously—is empty, then the LP is infeasible.
In such cases, the researcher should withdraw or revise the assumptions.

\subsection{Bounding Joint POs/OVs Probabilities and Their Linear Combinations From Finite Samples}

We use a plug-in estimator to compute the bounds from finite samples. That is, we use the empirical probabilities of $\mathbb{P}(Y_x)$ and $\mathbb{P}(X,Y)$ in solving the LP in \eqref{eq-LP}. The resulting plug-in estimators of the lower and upper bounds of $\lambda({\cal P})$ are denoted by 
$[\hat{\underline{\lambda}}({\cal P}),\hat{\overline{\lambda}}({\cal P})]$. $\hat{\underline{\lambda}}({\cal P})$ and $\hat{\overline{\lambda}}({\cal P})$ are consistent estimators of  $\underline{\lambda}({\cal P})$ and $\overline{\lambda}({\cal P})$. A detailed description is given in Appendix \ref{app_est}.

A discussion of the computational complexity is provided in Appendix~\ref{app_comp}.
The computational complexity of solving our bounding problems is polynomial when using interior-point methods \citep{Karmarkar1984}.
If we use the simplex method to solve LP, the worst-case computational complexity of solving our bounding problems is exponential \citep{RinnooyKan1981}.
However, in practice, the simplex method rarely exhibits its exponential worst-case behavior and is often faster than interior-point algorithms \citep{Vanderbei2001}.

{\bf Numerical Experiments.}
We conduct numerical experiments to illustrate how MAs can tighten the bounds.
All details are stated in Appendix \ref{subsec: numerical experiment of bounding}.
We simulated i.i.d. samples of experimental and observational datasets and computed bounds of the joint POs probability $\mathbb{P}(Y_0=0,Y_1=0,Y_2=1)$, the posterior causal effect $\mathbb{E}[Y_1-Y_0|X=2,Y=2]$, and the second moment of causal effect $\mathbb{E}[(Y_1-Y_0)^2]$ under various MAs.
The results are reported in Tables~\ref{tab:results}-\ref{tab:L1L2L3_2moment_obs} in Appendix \ref{subsec: numerical experiment of bounding}. The results show that our methods can provide informative bounds. 
We observe that introducing stronger assumptions leads to narrower bounds,  
and that the 95\% confidence intervals (CIs) for the estimates become narrower as the sample size increases.

\section{IDENTIFICATION WITH MONOTONICITY ASSUMPTIONS}

\label{sec6}

In this section, we introduce an MA that leads to the identification of the joint POs/OVs probabilities and their linear combinations.

\subsection{Identifying Joint POs/OVs Probabilities and Their Linear Combinations}
We  present an assumption for identifying the joint POs/OVs probabilities, which is a special case of Assumption~\ref{assum_mono_with_prob_s} and
 an extension of MITE (Assumption \ref{mono_it}).

\begin{assumption}[MITE for Discrete Treatment and Outcome]
\label{MITE_D}
For any $s>t$,
\begin{equation}
\label{assump_differs_only1}
    0\leq Y_s-Y_t\leq 1, \text{ a.s.}
\end{equation}
\end{assumption}
This assumption means that raising any treatment level can never make the outcome worse and can only improve it slightly, by at most one level.

Assumption~\ref{MITE_D} is a special case of  Assumption \ref{monomatAS}, where ${\underline{D}}_{st}=0$ and ${\overline{D}}_{st}=1$ for any $s>t$.
When $Y$ is binary, Assumption~\ref{MITE_D} is equivalent to Assumption~\ref{monoAS} and is further  equivalent to $\mathbb{P}(Y_t = 1, Y_s = 0)=0$ for any $s > t$.

{\bf Remark.}
Similar to the bounding problem, the feasible region defined by the linear constraints from data and the restriction imposed by Assumption~\ref{MITE_D} must be non-empty.

{\bf Experimental Data.}
We discuss the identification of the joint POs probabilities, i.e., $\mathbb{P}(Y_0,\dots,Y_{d_X-1})$,  and their linear combinations, using experimental data.
Because the experimental data does not contain $(X,Y)$, we focus solely on $\mathbb{P}(Y_0,\dots,Y_{d_X-1})$.

We denote the set of the joint probabilities of POs as 
\begin{equation}
{\cal P}'\defeq\hspace{0cm}\{\mathbb{P}(Y_{[d_X]}=y_{[d_X]})|y_{[d_X]}\in [d_Y]^{d_X}\}.
\end{equation}
The objective is a linear combination of the joint POs probabilities, expressed as $\lambda'({\cal P}')$,
where $\lambda'$ is an arbitrary linear function mapping the parameters $\cal P'$ to $\mathbb{R}$.

We have the following identification theorem.
\begin{theorem}[Identifiability under MA]
\label{thm: identification_of_joint_PO}
Under Assumption \ref{MITE_D}, 
given experimental distributions $\mathbb{P}(Y_x)$,
$\mathbb{P}(Y_0,\dots ,Y_{d_X-1})$ is identifiable as
\begin{equation}
\begin{aligned}
&\mathbb{P}(Y_0=\cdots =Y_{d_X-1}=y_0) \\
&= \sum\nolimits_{j=0}^{y_0} \mathbb{P}(Y_{d_X-1}=j)-\sum\nolimits_{j=0}^{y_0 -1}\mathbb{P}(Y_0=j)
\end{aligned}
\end{equation}
for $y_0\in [d_Y]$, and 
\begin{equation}
\begin{aligned}
&\mathbb{P}(Y_0=\cdots =Y_k=y_0,\ Y_{k+1}=\cdots =Y_{d_X-1}=y_0 +1)\\
&=\sum\nolimits_{j=0}^{y_0} \mathbb{P}(Y_k=j)-\sum\nolimits_{j=0}^{y_0} \mathbb{P}(Y_{k+1}=j)
\end{aligned}
\end{equation}
for $k\in [d_X-1]$ and $y_0\in[d_Y-1]$; otherwise,
\begin{equation}
\mathbb{P}(Y_0=y_0,\dots ,Y_{d_X-1}=y_{d_X-1})=0.
\label{eq: jointPO_others}
\end{equation}
Then, any $\lambda'({\cal P}')$ is identifiable.
\end{theorem}

{\bf Observational Data.}
Next, we discuss the identification of the joint POs/OVs probabilities, i.e., $\mathbb{P}(Y_0,\dots,Y_{d_X-1},X,Y)$, and their linear combinations, using observational data.
This requires exogeneity (Assumption \ref{exo}), i.e. $Y_x \indep X$, which implies $\mathbb{P}(Y_x) = P(Y|x)$.

We have the following:
\begin{theorem}[Identifiability under MA]
\label{thm: identification_jointPO_full}
Under Assumptions \ref{exo} and \ref{MITE_D}, 
given observational distributions $\mathbb{P}(X, Y)$,
$\mathbb{P}(Y_0,\dots ,Y_{d_X-1},X,Y)$ is identifiable as
\begin{equation}
\begin{aligned}
&\mathbb{P}(Y_0=\cdots =Y_{d_X-1}=y_0,X=x,Y=y)=\\
&\left(\sum_{m=0}^{y_0} \mathbb{P}(Y=m|X=d_X-1)-\sum_{m=0}^{y_0 -1}\mathbb{P}(Y=m|X=0)\right)\\
&\hspace{6cm}\times\mathbb{P}(X=x)
\end{aligned}
\end{equation}
for $y_0\in [d_Y]$, $x\in [d_X]$, and $y=y_0$, and
\begin{equation}
\begin{aligned}
&\mathbb{P}(Y_0=\cdots =Y_k=y_0,\\
&\hspace{1cm} Y_{k+1}=\cdots =Y_{d_X-1}=y_0+1,X=x,Y=y)\\
&=\left(\sum_{l=0}^{y_0} \mathbb{P}(Y=l|X=k)-\sum_{l=0}^{y_0} \mathbb{P}(Y=l|X=k+1)\right)\\
&\hspace{6cm}\times\mathbb{P}(X=x)  
\end{aligned}
\end{equation}
for $k\in [d_X-1]$, $y_0\in [d_Y-1]$, $y=y_0$ for $x\leq k$, and  $y=y_0+1$ for $x=k+1,\dots,d_X-1$; 
otherwise,
\begin{equation}
\mathbb{P}(Y_0=y_0,\dots ,Y_{d_X-1}=y_{d_X-1},X=x,Y=y)=0.
\end{equation}
Then, any $\lambda({\cal P})$ is identifiable. 
\end{theorem}

\textbf{Remark.} We have  the following  characterization of identification (when restricted to linear constraints):
\begin{itemize}
    \item Given experimental distributions $\mathbb{P}(Y_x)$, $\mathbb{P}(Y_0,\dots ,Y_{d_X-1})$ is identifiable if and only if we have $d_Y^{d_X}-d_X(d_Y-1)-1$ additional linearly independent constraints.
\item Under Assumption 1, given observational distributions $\mathbb{P}(X, Y)$, $\mathbb{P}(Y_0,\dots ,Y_{d_X-1},X,Y)$ is identifiable if and only if we have $d_X(d_Y^{d_X}-\{d_Y+(d_X-1)(d_Y-1)\})$ additional linearly independent constraints.
\end{itemize}
Assumption~\ref{MITE_D} imposes the {minimal number} of additional linear constraints needed to achieve identification in the sense that it provides exactly the required number of additional linearly independent constraints in both the experimental and observational settings.

\subsection{Estimating Joint POs/OVs Probabilities and Their Linear Combinations}
 We estimate the joint POs/OVs probabilities and their linear combinations from finite samples by plugging the empirical probabilities of $\mathbb{P}(Y_x)$ and $\mathbb{P}(X,Y)$
into the identification results in Theorems \ref{thm: identification_of_joint_PO} and \ref{thm: identification_jointPO_full}. 
{These are consistent estimators. A detailed discussion is given in Appendix \ref{app_est}.}


{\bf Numerical Experiments.}
We conduct numerical experiments to illustrate properties of the estimates under identification assumptions.
All details are stated in Appendix \ref{appC2}.
We simulated experimental and observational datasets and estimated $\mathbb{P}(Y_0=0,Y_1=0,Y_2=1)$, $\mathbb{E}[Y_1-Y_0|X=2,Y=2]$, and  $\mathbb{E}[(Y_1-Y_0)^2]$.
The results are reported in 
Tables \ref{tab:iden_results}-\ref{tab:iden_results_2ndmoment}  in Appendix \ref{appC2}.
We observe that all 95\% CIs of the estimates 
become narrower as the sample size increases, and that the means of the estimators are very close to the ground truth when the sample size exceeds 100.

\section{APPLICATION TO REAL-WORLD DATASETs}

\subsection{Experimental Data}
\label{subsec: experimental data}

{\bf Dataset.}
We use the Project STAR randomized trial \citep{Stock2007}, publicly available at (\url{https://search.r-project.org/CRAN/refmans/AER/html/STAR.html}). 
The experiment was conducted in Tennessee in the late 1980s across 79 schools, and the experiment randomized students into three class-size treatments:
regular classes with one teacher (``regular"; $X=0$), regular classes with a teacher and a teacher's aide (``regular+aide”; $X=1$), and small classes (``small"; $X=2$).
Teachers and students were randomly assigned, and outcomes were standardized math and reading scores. 
After removing missing values, the sample size is 5967.
Previous analyses of STAR by \citet{Stock2007} reported 
that smaller classes have positive effects on academic performance.

\textbf{Analysis.}
We take the sum of third-grade math and reading scores as the outcome $Y$. 
We group the scores into quartiles based on their empirical distribution, $[1009,1180]$, $[1181,1230]$, $[1231,1282]$, and $[1283,1527]$, and label them $0$, $1$, $2$, and $3$, respectively. 
We conduct 100 times bootstrapping \citep{Efron1979} 
to provide the mean and 95\% confidence interval (CI) for the estimator.

{\bf Assumptions.}
We assume the following MAs: 
\begin{enumerate}
\vspace{-0.2cm}
  \setlength\itemsep{3pt}
  \setlength\parskip{0pt}
  \setlength\parsep{0pt}
\item[] (A1). No MAs,
\item[] (A2). $Y_0\leq Y_1$ holds a.s. (Assumption \ref{monomatAS}),
\item[] (A3). $Y_1\leq Y_2$ holds a.s. (Assumption \ref{monomatAS}),
\item[] (A4).  $\mathbb{P}(Y_0\leq Y_1)\geq 0.95$ and $\mathbb{P}(Y_1\leq Y_2)\geq 0.95$ (Assumption \ref{assum_mono_with_prob_s}),
\item[] (A5). $0\leq Y_i-Y_j\leq 1$ holds for all $i>j$ a.s. (Assumption \ref{MITE_D}).
\vspace{-0.2cm}
\end{enumerate}
(A2) states that, for almost all students, the potential score if assigned to a regular class with both a teacher and a teacher's aide is greater than the potential score if assigned to a regular class with only a teacher.
(A3) states that, for almost all students, the potential score if assigned to a small class is greater than the potential score if assigned to a regular class with both a teacher and a teacher's aide.
(A4) combines (A2) and (A3) and permits a 5\% violation of each condition.
(A5) states that, for almost all students, increasing the treatment level never decreases the score level, and can raise it by at most one level.
By Theorem \ref{thm: identification_of_joint_PO}, the joint probabilities of POs are identifiable under Assumption (A5).
These MAs may be regarded as reasonable.

{\bf Results.}
We first estimate the mean POs (directly from the experimental data) 
with 95\% CIs and obtain 
\begin{center}
$\mathbb{E}[Y_0]$: 1228.725 [1225.898,1231.136]\vspace{0.1cm}\\
$\mathbb{E}[Y_1]$: 1228.357 [1225.792,1231.461]\vspace{0.1cm}\\
$\mathbb{E}[Y_2]$: 1244.199 [1240.400,1248.306]
\end{center}
The results show that small classes lead to slightly higher average scores. Then one may be motivated to implement small class sizes.

We next evaluate the probability $\mathbb{P}(Y_0=0,Y_1=0,Y_2\geq 1)$, which denotes the probability of the following  counterfactual scenario: 
\vspace{-0cm}
\begin{center}
``{\it A student would score in the bottom 25\% if placed in a regular class with either a teacher alone or with both a teacher and a teacher’s aide, and would score above the bottom 25\% if enrolled in a small class.}"
\end{center}
\vspace{-0cm}
This is a situation in which being assigned to a small class is necessary and sufficient to obtain a score above the bottom 25\%.

Table~\ref{tab:app1} shows that this probability lies between 0\% and 26.1\% (without making any assumptions), has a higher lower bound 6.4\% under assumption (A2), a tighter upper bound 6.4\% under (A3), a narrower interval of [1.4\%,11.4\%] under (A4), and is point-identified at 6.4\% under (A5). 
The results mean that, for instance, if one believes that assumption (A3) or (A5) is plausible, then only for a small number of subjects  (no more than 6.4\%), assignment to a small class is both necessary and sufficient to achieve a score above the bottom 25\%.
This finding may discourage one from implementing small classes.  
Note that in experimental results such as Table~\ref{tab:app1}, the bootstrap distribution of the estimator may become distorted when the true probability is close to 0 or 1 due to boundary problems.

\begin{table}[!tb]
\centering
\caption{
The means and 95\% CIs of the estimates of $\mathbb{P}(Y_0=0,Y_1=0,Y_2\geq 1)$.
LB stands for ``lower bound" and UB stands for ``upper bound."
}
\label{tab:app1}
\vspace{-0.25cm}
\scalebox{0.9}{
\begin{tabular}{c|cc}
\hline
Assumption & LB & UB\\
\hline
(A1) &  0.000 [0.000,0.000] & 0.261 [0.247,0.274] \\
(A2) &  0.064 [0.042,0.089] & 0.261 [0.247,0.274] \\
(A3) &  0.000 [0.000,0.000] & 0.064 [0.042,0.089] \\
(A4) & 0.014 [0.000,0.039] & 0.114 [0.092,0.139] \\
\hline 
\hline 
Assumption &\multicolumn{2}{c}{Identification}\\
\hline 
(A5) &  \multicolumn{2}{c}{0.064 [0.042,0.089]}  \\
\hline
\end{tabular}
}
\end{table}

\subsection{Observational Data}
\label{subsec: observational data}

{\bf Dataset.}
We use the dataset introduced in Section \ref{sec: motivation}.
\citet{Cortez2008, Helwig2017} focused on predicting students' academic performance from their attributes.
\citet{Kawakami2024} assessed the causal relationship between student performance, study time, and extra paid classes by estimating PoC.
They treat the outcome as a continuous variable and focus only on identifying the PoC.
\citet{Kawakami2025_multmed} further performed mediation analysis of the PoC.

{\bf Analysis.}
The treatment  ($X$) and outcome ($Y$) variables are  defined in Section \ref{sec: motivation}. 
We restrict our analysis to the subpopulation defined by the covariates (sex = F, schoolsup = no, famsup = yes).
{We assume that these covariates cover all the confounders such that the exogeneity (Assumption \ref{exo}) holds in this subpopulation.} 
The sample size of this subpopulation is 219.
We conduct 100 times bootstrapping 
to provide the mean and 95\% CI for the estimator.

{\bf Assumptions.}
We assume the following MAs:
\begin{enumerate}
\vspace{-0.2cm}
  \setlength\itemsep{3pt}
  \setlength\parskip{0pt}
  \setlength\parsep{0pt}
\item[] (B1). No MAs,
\item[] (B2). $Y_0\leq Y_1$ holds a.s. (Assumption \ref{monomatAS}),
\item[] (B3). $Y_2\leq Y_3$ holds a.s. (Assumption \ref{monomatAS}),
\item[] (B4).  $\mathbb{P}(Y_0\leq Y_1)\geq 0.995$ and $\mathbb{P}(Y_2\leq Y_3)\geq 0.995$ (Assumption \ref{assum_mono_with_prob_s}),
\item[] (B5). $0\leq Y_i-Y_j\leq 1$ holds for all $i>j$ a.s. (Assumption \ref{MITE_D}).
\vspace{-0.2cm}
\end{enumerate}
(B2) (resp. (B3)) states that, for almost all students, the potential score if they studied 2 to 5 hours (resp. more than 10 hours) is greater than the potential score if they studied less than 2 hours (resp. 5 to 10 hours).
(B4) combines (B2) and (B3) and permits a 0.5\% violation of each condition.
(B5) states that, for almost all students, increasing the study time alone never decreases the score level and can raise it by at most one level.
By Theorem \ref{thm: identification_jointPO_full}, 
the joint POs/OVs probabilities and their linear combinations are identifiable under (B5) and the exogeneity assumption.
These MAs may be regarded as reasonable.

{\bf Results.}
First, we estimate the POs based on $\mathbb{E}[Y_x]=\mathbb{E}[Y|x]$ by the exogeneity assumption. 
The mean estimates with 95\% CIs are as follows:
\begin{center}
$\mathbb{E}[Y_0]$: 11.154 [10.662,11.663]\vspace{0.1cm}\\
$\mathbb{E}[Y_1]$: 12.457 [11.940,12.959]\vspace{0.1cm}\\
$\mathbb{E}[Y_2]$: 12.979 [12.357,13.708]\vspace{0.1cm}\\
$\mathbb{E}[Y_3]$: 15.732 [14.220,16.909]
\end{center}
These estimates indicate that 
the average score increases as the level of study time increases.  
The results may 
motivate students to study more in order to obtain higher scores.

We then evaluate $\mathbb{P}(Y_0\leq 1,Y_1\leq 1,Y_2\leq 1|X=3,Y=2)$, whose interpretation is stated in Section \ref{sec: motivation}.
This is a situation in which studying more than 10 hours is necessary to obtain a score above 15.

Table \ref{tab:app2_1112} reports the estimates under exogeneity.
(B2), (B3), and (B4) yield tighter bounds than without MAs (B1).
A combination of (B5) and exogeneity yields a point estimate of 57.5\%, indicating that, for 57.5\% of the students, studying more than 10 hours is necessary to obtain a score above 15. 
This may be strong evidence that studying long hours is necessary to obtain a high score for half of the students, and may support the view that students’ study time is not wasted.
We note that this finding is not captured by the average results about $\mathbb{E}[Y_x]$.

We also evaluate $\mathbb{P}(Y_1\geq 1,Y_2\geq 1,Y_3\geq 1|X=0,Y=0)$,
whose interpretation is  stated in Section \ref{sec: motivation}.
This is a situation in which studying at least 2 hours is sufficient to obtain a score above 8.

Table \ref{tab:app2_0111} reports the estimates  under exogeneity.
(B2) and (B4) yield tighter bounds than without MAs (B1), while the assumption (B3) does not tighten the bounds.  
A combination of (B5) and exogeneity yields a point estimate of 25.1\%, indicating that 25.1\% of the students who studied less than 2 hours and obtained a score of no more than 8 would have obtained a score above 8 had they studied more than 2 hours.
That is, assuming (B5) and exogeneity, studying at least 2 hours is sufficient to obtain a score above 8 for a quarter of the students.
This finding is not captured by the average results about $\mathbb{E}[Y_x]$.

{Additionally, we conduct evaluations under each MA without  the exogeneity assumption. All bounds without exogeneity are uninformative, ranging from 0\% to 100\%.}

\begin{table}[!tb]
\centering
\caption{
The means and 95\% CIs of the estimates of $\mathbb{P}(Y_0\leq 1,Y_1\leq 1,Y_2\leq 1|X=3,Y=2)$ under exogeneity.
}
\label{tab:app2_1112}
\vspace{-0.25cm}
\scalebox{0.9}{
\begin{tabular}{c|cc}
\hline
Assumption & LB & UB\\
\hline
\hline
(B1)&  0.203 [0.000,0.489] & 0.997 [0.956,1.000] \\
(B2)& 0.293 [0.000,0.561] & 0.997 [0.956,1.000] \\
(B3)&  0.203 [0.000,0.489] & 0.575 [0.192,0.740] \\
(B4)&  0.207 [0.000,0.489] & 0.783 [0.531,1.000] \\
\hline 
\hline 
Assumption &\multicolumn{2}{c}{Identification}\\
\hline 
(B5)  & \multicolumn{2}{c}{0.575 [0.192,0.740]} \\
\hline
\end{tabular} 
}
\end{table}

\begin{table}[!tb]
\centering
\caption{
The means and 95\% CIs of the estimates of $\mathbb{P}(Y_1\geq 1,Y_2\geq 1,Y_3\geq 1|X=0,Y=0)$ under exogeneity.
}
\label{tab:app2_0111}
\vspace{-0.25cm}
\scalebox{0.9}{
\begin{tabular}{c|cc}
\hline
Assumption & LB & UB\\
\hline
\hline
(B1) & 0.044 [0.000,0.280] & 1.000 [1.000,1.000] \\
(B2) & 0.044 [0.000,0.280] & 0.251 [0.013,0.612] \\
(B3) & 0.044 [0.000,0.280] & 1.000 [1.000,1.000] \\
(B4) & 0.044 [0.000,0.280] & 0.404 [0.165,0.665] \\

\hline 
\hline 
Assumption & \multicolumn{2}{c}{Identification}\\
\hline 
(B5) &  \multicolumn{2}{c}{0.251 [0.013,0.612]} \\
\hline
\end{tabular} 
}
\end{table}

\section{CONCLUSION}

{We provide methods for the bounding and identification of
joint probabilities of potential outcomes and observed variables in settings involving a discrete treatment and an ordinal outcome. We introduce a broad class of flexible monotonicity assumptions, which empower researchers to  integrate nuanced domain knowledge into the analysis. 
Our approach  derives bounds that are provably as tight as the available domain knowledge allows. This bounding methodology is particularly valuable in practice because it yields informative results under plausible assumptions based on prior knowledge.
The tools developed here move beyond standard Average Causal Effects (ACEs), providing a rigorous framework for evaluating a broad class of nuanced causal quantities. This includes families of effects like the Probability of Causation (PoC) and moments of causal effects, which ultimately offer crucial, valuable support for informed decision-making across various applied fields.
}
 
{\citet{Balke1995,Sachs2023} derived closed-form symbolic bounds for causal effects or joint POs probabilities involving binary variables. However, their algorithm becomes computationally infeasible when the domain sizes $d_X$ and $d_Y$ get larger. 
We therefore adopt a numerical approach to solve the bounding problem. 
Our source code is available at 
(\url{https://github.com/Nmath997/AISTATS2026_bounding.git}).}

If additional variables such as covariates or mediators are included, tighter bounds can be obtained, as shown in \citep{Kuroki2011,Dawid2016}. 
Extending our results to incorporate covariates or mediators represents an interesting direction for future work. 
However, such an extension would involve a larger number of parameters and may become computationally difficult to solve.

One limitation of this work is that we do not incorporate constraints implied by causal graphs.  
The problem of bounding counterfactual probabilities in a general causal graph with discrete variables has been studied in \citep{cozman2000credal, Zhang2022_Partial,zaffalon2020structural, zaffalon2024efficient,Duarte2024}, including multi-linear programming for causal bounds in quasi-Markovian models \citep{shridharan2023causal,shridharan2023scalable,Arroyo2025}. Additionally,  \citet{Maiti2025} extended the binary MA to identify counterfactual probabilities in causal graphs. Investigating MAs for bounding counterfactual probabilities in a general causal graph  is a future research direction.

\section*{Acknowledgements}

The authors thank the anonymous reviewers for their time
and thoughtful comments.

\bibliographystyle{plainnat}
\bibliography{citation}

\section*{Checklist}



 \begin{enumerate}

 \item For all models and algorithms presented, check if you include:
 \begin{enumerate}
   \item A clear description of the mathematical setting, assumptions, algorithm, and/or model. [Yes]
   \item An analysis of the properties and complexity (time, space, sample size) of any algorithm. [Yes]
   \item (Optional) Anonymized source code, with specification of all dependencies, including external libraries. [Yes] The source code is available at 
(\url{https://github.com/Nmath997/AISTATS2026_bounding.git}).
 \end{enumerate}

 \item For any theoretical claim, check if you include:
 \begin{enumerate}
   \item Statements of the full set of assumptions of all theoretical results. [Yes] All theorems are presented with their corresponding assumptions.
   \item Complete proofs of all theoretical results. [Yes]
   \item Clear explanations of any assumptions. [Yes]  In Appendix \ref{app_proof}.    
 \end{enumerate}

 \item For all figures and tables that present empirical results, check if you include:
 \begin{enumerate}
   \item The code, data, and instructions needed to reproduce the main experimental results (either in the supplemental material or as a URL). [Yes] 
   \item All the training details (e.g., data splits, hyperparameters, how they were chosen). [Not Applicable]
         \item A clear definition of the specific measure or statistics and error bars (e.g., with respect to the random seed after running experiments multiple times). [Yes] We have provided 95\% confidence intervals for all reported estimates.
         \item A description of the computing infrastructure used. (e.g., type of GPUs, internal cluster, or cloud provider). [Yes] We use a computer with AMD Ryzen 7 8840U 3.30GHz processor and 16GB of RAM. 
 \end{enumerate}

 \item If you are using existing assets (e.g., code, data, models) or curating/releasing new assets, check if you include:
 \begin{enumerate}
   \item Citations of the creator If your work uses existing assets. [Not Applicable]
   \item The license information of the assets, if applicable. [Not Applicable]
   \item New assets either in the supplemental material or as a URL, if applicable. [Not Applicable]
   \item Information about consent from data providers/curators. [Not Applicable]
   \item Discussion of sensible content if applicable, e.g., personally identifiable information or offensive content. [Not Applicable]
 \end{enumerate}

 \item If you used crowdsourcing or conducted research with human subjects, check if you include:
 \begin{enumerate}
   \item The full text of instructions given to participants and screenshots. [Not Applicable]
   \item Descriptions of potential participant risks, with links to Institutional Review Board (IRB) approvals if applicable. [Not Applicable]
   \item The estimated hourly wage paid to participants and the total amount spent on participant compensation. [Not Applicable]
 \end{enumerate}

 \end{enumerate}

\newpage
\appendix
\onecolumn

\clearpage
\appendix
\thispagestyle{empty}

\onecolumn
\aistatstitle{Appendix to ``Bounds and Identification of Joint Probabilities of Potential Outcomes and Observed Variables under Monotonicity Assumptions"}

\section{DETAILS OF MONOTONICITY ASSUMPTIONS IN THE PREVIOUS WORKS}
\label{app_mono}

In this appendix, we provide details of the monotonicity assumptions used in previous works.

({\bf Assumption \ref{monobin_s}}.)
For binary treatment and outcome ($d_X=d_Y=2$), \citet{Pearl1999} and \citet{Tian2000} established bounds for PoC without additional assumptions and provided an identification theorem under the exogeneity assumption, i.e., $Y_x\indep X$ for any $x \in \{0,1\}$, combined with the monotonicity assumption, i.e.,
\begin{equation}
\begin{aligned}
\mathbb{P}(Y_0=1,Y_1=0)=0.
\end{aligned}
\end{equation}
The monotonicity is equivalently stated as
\begin{equation}
\mathbb{P}(Y_0 \leq Y_1)=1,
\end{equation} 
or
\begin{equation}
Y_0 \leq Y_1, \text{ almost surely},
\end{equation} 
where a statement holds almost surely if it is true with probability one.
This condition represents that the POs are monotonically increasing with the treatment level for almost every subject.

The monotonicity assumption is a natural and reasonable assumption in a wide range of empirical studies \citep{Hernan2024}.
In the instrumental variables (IV) literature, the monotonicity assumption is often invoked to describe the relationship between binary IV ($Z$) and treatment ($X$) in the identification of local average treatment effects (LATE) \citep{Imbens1994,Angrist1995}.
They provide an identification theorem for PoC under the proposed monotonicity condition together with the exogeneity assumption.

({\bf Assumptions \ref{mono_dc} and \ref{mono_it}}.)
\citet{Zhang2025} studied PN for a binary treatment and ordinal outcome, i.e., $\mathbb{P}(Y_0<y|X=1,Y=y)$, and derived bounds for PN under
the following monotonicity assumption:
\begin{equation}
Y_0 \leq Y_1, \text{ almost surely}.
\end{equation}
This is a straightforward extension of the monotonicity assumption in \citep{Pearl1999,Tian2000}.
They provided closed-form bounds for PN under this assumption.

This paper also introduced an assumption termed the monotonic incremental treatment effect (MITE) assumption for identifying PN, which states
\begin{equation}
0 \leq Y_1 - Y_0 \leq 1, \text{ almost surely}.
\end{equation}
This assumption is stronger than Assumption \ref{mono_dc}.
This assumption serves as the identification condition for PN.
This assumption states that the treatment $X$ does not decrease the level of $Y$ and increases it by at most one level.
They established an identification lemma for $\mathbb{P}(Y_1, Y_0|X=1)$ given $\mathbb{P}(Y_0|X=1)$ under the above monotonicity condition. 
Then, they established the identification theorem for PN through the lemma.

({\bf Assumption \ref{monobin_w}}.)
The relaxed assumption, i.e., $\epsilon$-limited harm assumption, for binary treatment and outcome by \citet{Mueller2023} states that:
\begin{equation}
\begin{aligned}
0\leq \mathbb{P}(Y_0=1,Y_1=0)\leq \epsilon.
\end{aligned}
\end{equation}
This assumption relaxes the monotonicity assumption in \citep{Pearl1999,Tian2000} by allowing violations in which subjects have POs $Y_0 = 1$ and $Y_1 = 0$ for at most an $\epsilon$ proportion of the population.
This is equivalent to $l\leq \mathbb{P}(Y_0 \leq Y_1)\leq 1$, where $l=1-\epsilon$.
They provided closed-form bounds for PNS with this assumption.
When $\epsilon=0$ or $l=1$, Assumption \ref{monobin_w} reduces to Assumption \ref{monobin_s}.

({\bf Assumption \ref{monoAS}}.)
The above studies have been limited to settings with binary treatment. 
\citet{Manski1997} has introduced the following assumption:
\begin{equation}
Y_0\leq Y_1 \leq \dots \leq Y_{d_X-1}, \text{ almost surely}.
\end{equation}
This condition implies that the POs are monotonically increasing with the treatment level for almost every subject.
This is equivalently stated as $\mathbb{P}(Y_0\leq Y_1 \leq \dots \leq Y_{d_X-1})=1$.
Their target is the average causal effect (ACE) for discrete, continuous, or mixed outcomes.
The joint POs/OVs probabilities are not their target.

\section{CAUSAL QUANTITIES REPRESENTED BY LINEAR COMBINATIONS OF JOINT POs/OVs PROBABILITIES}

\label{app AB}
In this appendix, we show more examples that can be expressed as linear combinations of $\mathcal{P}$ (Section \ref{subsec:objective_of_study}).

{\bf PoC for discrete treatment and outcome.}
Recently, \citet{ALi2024} have introduced new families of PoC for non-binary discrete treatment—including the probability of preservation, the probability of replacement, the probability of substitution, the probability of necessity, and the probability of necessity and sufficiency—extending the classical PoC framework to nonbinary and discrete treatments and outcomes, using the joint probabilities of POs and OVs $\mathbb{P}(Y_0=y_0,Y_1=y_1,\dots ,Y_{d_X-1}=y_{d_X-1},X=x,Y=y)$.

PoC with a single hypothetical term for non-binary treatment and outcome is given as 

\hspace{0.2cm}(i) The probability of preservation $\mathbb{P}(Y_j=y,Y=y)$,

\hspace{0.2cm}(ii) The probability of replacement $\mathbb{P}(Y_j=y_j,Y=y)$ ($y_j \ne y$),

\hspace{0.2cm}(iii) The probability of substitute $\mathbb{P}(Y_j=y_j,X=x)$ ($x \ne j$), and

\hspace{0.2cm}(iv) The probability of necessity $\mathbb{P}(Y_j=y_j,Y=y,X=x)$ ($x \ne j$).

PoC with multiple hypothetical terms for non-binary treatment and outcome are given as

\hspace{0.2cm}(i) The probability of necessity and sufficiency $\mathbb{P}(Y_0=y_0,Y_1=y_1,\dots ,Y_{d_X-1}=y_{d_X-1})$,

\hspace{0.2cm}(ii) The probability of substitute $\mathbb{P}(Y_0=y_0,Y_1=y_1,\dots ,Y_{d_X-1}=y_{d_X-1},X=x)$, 

\hspace{0.2cm}(iii) The probability of replacement $\mathbb{P}(Y_0=y_0,Y_1=y_1,\dots ,Y_{d_X-1}=y_{d_X-1},Y=y)$, and 

\hspace{0.2cm}(iv) The probability of necessity $\mathbb{P}(Y_0=y_0,Y_1=y_1,\dots ,Y_{d_X-1}=y_{d_X-1},X=x,Y=y)$.

They are given as marginals of the joint POs/OVs probabilities.
\citet{ALi2024} provided closed forms of their bounds without additional assumptions, and the bounds for PoC with a single hypothetical term are sharp.

{\bf Joint expectation of a function  over POs and OVs.}
The joint expectation of a function  over POs and OVs can be expressed as linear combinations of ${\cal P}$:
\begin{equation}
\begin{aligned}
&\mathbb{E}[g(Y_0,\dots,Y_{d_X-1},X,Y)]\\
&=\sum_{y_0=0}^{d_Y-1}\dots\sum_{y_{d_X-1}=0}^{d_Y-1}\sum_{x=0}^{d_X-1}\sum_{y=0}^{d_Y-1}g(y_0,\dots,y_{d_X-1},x,y)\times\mathbb{P}(Y_0=y_0,Y_1=y_1,\dots ,Y_{d_X-1}=y_{d_X-1},X=x,Y=y).
\end{aligned}
\end{equation}
For example, the moments of causal effects \citep{Kawakami2025_moments}, such as
\begin{equation}
\begin{aligned}
&\mathbb{E}[(Y_1-Y_0)^m]\\
&=\sum_{y_0=0}^{d_Y-1}\dots\sum_{y_{d_X-1}=0}^{d_Y-1}\sum_{x=0}^{d_X-1}\sum_{y=0}^{d_Y-1}(y_1-y_0)^m\times\mathbb{P}(Y_0=y_0,Y_1=y_1,\dots ,Y_{d_X-1}=y_{d_X-1},X=x,Y=y),    
\end{aligned}
\end{equation}
and the product moments of causal effects \citep{Kawakami2025_moments}, such as 
\begin{equation}
\begin{aligned}
&\mathbb{E}[(Y_1-Y_0)(Y_2-Y_1)]\\ 
&=\sum_{y_0=0}^{d_Y-1}\dots\sum_{y_{d_X-1}=0}^{d_Y-1}\sum_{x=0}^{d_X-1}\sum_{y=0}^{d_Y-1}(y_1-y_0)(y_2-y_1)\times\mathbb{P}(Y_0=y_0,Y_1=y_1,\dots ,Y_{d_X-1}=y_{d_X-1},X=x,Y=y).
\end{aligned}  
\end{equation}
can be expressed as linear combinations of ${\cal P}$.

When $\mathbb{P}(X=l,Y=m)$ is given,
the conditional probabilities given $X$ and $Y$ can be expressed as the following linear relationship:
\begin{equation}
\begin{aligned}
&\mathbb{P}(Y_{[d_X]}=y_{[d_X]}|X=l,Y=m)=\frac{\mathbb{P}(Y_{[d_X]}=y_{[d_X]},X=l,Y=m)}{\mathbb{P}(X=l,Y=m)},
\end{aligned}
\end{equation}
and this quantity is also bounded if the joint POs/OVs probabilities are bounded, e.g., PN \citep{Zhang2025}.

Furthermore, when $\mathbb{P}(X=l,Y=m)$ is given,
the joint expectation of a function $\lambda$ over POs, conditional on $X$ and $Y$, i.e., $\mathbb{E}[g(Y_0,\dots,Y_{d_X-1},X,Y)|X=l,Y=m]$, can be expressed as the following linear relationship:
\begin{equation}
\begin{aligned}
&\mathbb{E}[g(Y_0,\dots,Y_{d_X-1},X,Y)|X=l,Y=m]\\
&=\sum_{y_0=0}^{d_Y-1}\dots\sum_{y_{d_X-1}=0}^{d_Y-1}g(y_0,\dots,y_{d_X-1},l,m)\times
\hspace{0mm}\frac{\mathbb{P}(Y_0=y_0,Y_1=y_1,\dots ,Y_{d_X-1}=y_{d_X-1},X=l,Y=m)}{\mathbb{P}(X=l,Y=m)}.
\end{aligned}
\end{equation}
The causal effect given $(X=l,Y=m)$,
\begin{equation}
\begin{aligned}
&\mathbb{E}[Y_x-Y_{x'}|X=l,Y=m]\\
&=\sum_{y_0=0}^{d_Y-1}\dots\sum_{y_{d_X-1}=0}^{d_Y-1}(y_x-y_{x'})\times
\hspace{0mm}\frac{\mathbb{P}(Y_0=y_0,Y_1=y_1,\dots ,Y_{d_X-1}=y_{d_X-1},X=l,Y=m)}{\mathbb{P}(X=l,Y=m)}.
\end{aligned}
\end{equation}
is a member of the class of posterior causal effects, a concept introduced by \citet{Lu2022}, who define such effects as expectations conditional on evidence observed for a post-treatment variable.
It evaluates the counterfactual average causal effect of changing the intervention from $X = x'$ to $X = x$ for subjects who, in reality, had $Y = m$ under $X = l$.

\section{PROOFS }
\label{app_proof}

\label{appB}

In this section, we provide the proofs of the propositions and theorems stated in the main body of the paper.

{\bf Proposition \ref{prop1}.}
{\it
We have the following  relationships among the MAs presented thus far:

(1) Assumptions~\ref{monobin_s} through~\ref{monomatASS} are all special cases of Assumption~\ref{assum_mono_with_prob_s}. 

(2) Assumption \ref{mono_it} is  a special case of Assumptions \ref{weak_mono-bi}, \ref{monomatAS}, and \ref{monomatASS}. 

(3) Assumption \ref{monobin_w} is  a special case of Assumptions \ref{weak_mono-bi}, \ref{monoLB}, and \ref{monomatASS}. 

(4) Assumption \ref{monoAS} is  a special case of Assumptions \ref{monoLB}, \ref{monomatAS}, and \ref{monomatASS}.
}

\begin{proof}
First, we recall that Assumptions \ref{monobin_s} and \ref{mono_dc} are special cases of Assumption \ref{mono_it} (Appendix \ref{app_mono}), and that  Assumption \ref{monomatASS} is a special case of Assumption \ref{assum_mono_with_prob_s} (Section \ref{sec:mono_assumptions}).
We also remark that Assumptions \ref{weak_mono-bi}, \ref{monoLB}, \ref{monomatAS} are special cases of Assumption \ref{monomatASS}.
Indeed, Assumption \ref{monomatASS} reduces to Assumption \ref{weak_mono-bi} (resp.\ref{monoLB}, \ref{monomatAS}) when $d_X=2$ and  $\underline{D}_{st}=\underline{D}$ (resp. when $\underline{D}_{st}=-\infty$, $\overline{D}_{st}=\infty$ for all $s>t$ with $s=t+1$, when $L=U=1$).

Then we can prove each assertion of the proposition as follows:

(2): We obtain Assumption \ref{mono_it} from Assumption \ref{weak_mono-bi} (resp. \ref{monomatAS}) by putting $\underline{D}=0,\overline{D}=1,L=-\infty,U=\infty$ (resp. by setting $d_X=2$ and putting $\underline{D}_{10}=0,\overline{D}_{10}=1$). Since these are special cases of Assumption \ref{monomatASS}, the assertion (2) holds.

(3): We obtain Assumption \ref{monobin_w} from Assumption \ref{weak_mono-bi} (resp. \ref{monoLB}) by setting  $d_Y=2$ and putting $\underline{D}=0,\overline{D}=1,L=0,U=\epsilon$ (resp. by setting $d_X=d_Y=2$ and putting $L=1-\epsilon, U=1$). Since these are special cases of Assumption \ref{monomatASS}, the assertion (3) holds.

(4): We obtain Assumption \ref{monoAS} from Assumption \ref{monoLB} (resp. \ref{monomatAS}) by putting $L=-\infty,U=\infty$ (resp. by putting $\underline{D}_{st}=-\infty$, $\overline{D}_{st}=\infty$ for all $s>t$ with $s=t+1$). Since these are special cases of Assumption \ref{monomatASS}, the assertion (4) holds.

Now we also have the assertion (1) from (2)--(4) and remarks in the beginning of this proof.
\end{proof}

{\bf Proposition \ref{prop2}.}
{\it
The constraints implied by Assumption~\ref{assum_mono_with_prob_s} are given by
\begin{equation}
\begin{aligned}
&L^w\leq  \sum\nolimits_{y_0=0}^{d_Y-1}\dots \sum\nolimits_{y_{d_X-1}=0}^{d_Y-1}\sum\nolimits_{x=0}^{d_X-1}\left(\prod_{s,t\in [d_X];s>t}\mathbb{I}({\underline{D}}_{st}^{(w)}\leq y_s- y_t\leq {\overline{D}}_{st}^{(w)})\right)p_{y_{[d_X]},x}\leq U^w
\end{aligned}
\end{equation}
for $w=1,\dots,W$.
The above constraints are all linear constraints in ${\cal P}$.
}

{\bf Remark.} $[d_X]$ represents the set $\{0,1,\dots ,d_X-1\}$.

\begin{proof}
We have the following relationships:
\begin{equation}
\begin{aligned}
&\bigwedge_{s,t\in [d_X];s>t}({\underline{D}}_{st}^{(w)}\leq y_s- y_t\leq {\overline{D}}_{st}^{(w)})\\
&\iff {\underline{D}}_{st}^{(w)}\leq y_s- y_t\leq {\overline{D}}_{st}^{(w)} \text{ holds for any } s,t(s>t)\\
& \iff \mathbb{I}({\underline{D}}_{st}^{(w)}\leq y_s- y_t\leq {\overline{D}}_{st}^{(w)})=1 \text{ for any } s,t(s>t)\\
& \iff \prod_{s>t}\mathbb{I}({\underline{D}}_{st}^{(w)}\leq y_s- y_t\leq {\overline{D}}_{st}^{(w)})=1.
\end{aligned}
\end{equation}
Then, the following relationship holds:
\begin{equation}
\begin{aligned}
&\mathbb{P}\left(\bigwedge_{s,t\in [d_X];s>t}({\underline{D}}_{st}^{(w)}\leq y_s- y_t\leq {\overline{D}}_{st}^{(w)})\right)\\
&=\mathbb{P}\left(\prod_{s>t}\mathbb{I}({\underline{D}}_{st}^{(w)}\leq y_s- y_t\leq {\overline{D}}_{st}^{(w)})=1\right)\\
&=\sum_{y_0=0}^{d_Y-1}\dots \sum_{y_{d_X-1}=0}^{d_Y-1}\sum_{x=0}^{d_X-1}\left(\prod_{s>t}\mathbb{I}({\underline{D}}_{st}^{(w)}\leq y_s- y_t\leq {\overline{D}}_{st}^{(w)})\right)\times p_{y_{[d_X]},x}.
\end{aligned}
\end{equation}
They are all linear constraints in ${\cal P}$ since $\prod_{s>t}\mathbb{I}({\underline{D}}_{st}^{(w)}\leq y_s- y_t\leq {\overline{D}}_{st}^{(w)})$ is a constant $0$ or $1$.
Then, we have Proposition \ref{prop2}.
\end{proof}

{\bf Theorem \ref{thoe_lp}.}
{\it
{Given data (experimental, observational, or both),   under Assumption \ref{assum_mono_with_prob_s}, 
solving the  LP problem in \eqref{eq-LP} yields sharp bounds for $\lambda({\cal P})$.}
}

\begin{proof}
In general, optima obtained by solving LP problems are global optima. In addition, since constraints from data and Eq. \eqref{eq: mono_prop_p} restricts $\mathfrak{p}$ to a closed subset of $[0,1]^{d_Y^{d_X}{d_X}}$, LP \eqref{eq-LP} has the minimum $\underline{\lambda}(\mathcal{P})$ and the maximum $\overline{\lambda}(\mathcal{P})$. 
Then $[\underline{\lambda}(\mathcal{P}),\overline{\lambda}(\mathcal{P})]$ obtained by solving \eqref{eq: mono_prop_p} is a bound of $\lambda(\mathcal{P})$, and it is sharp since $\underline{\lambda}(\mathcal{P})$ and $\overline{\lambda}(\mathcal{P})$ are achieved by some $\mathfrak{p}$ which satisfies constraints from data and Eq. \eqref{eq: mono_prop_p}.
\end{proof}

{\bf Theorem \ref{thm: identification_of_joint_PO}.}
{\it
Under Assumption \ref{MITE_D}, 
given experimental distributions $\mathbb{P}(Y_x)$,
$\mathbb{P}(Y_0,\dots ,Y_{d_X-1})$ is identifiable as
\begin{equation}
\begin{aligned}
&\mathbb{P}(Y_0=\cdots =Y_{d_X-1}=y_0)= \sum\nolimits_{j=0}^{y_0} \mathbb{P}(Y_{d_X-1}=j)-\sum\nolimits_{j=0}^{y_0 -1}\mathbb{P}(Y_0=j)
\end{aligned}
\end{equation}
for $y_0\in [d_Y]$, and 
\begin{equation}
\begin{aligned}
&\mathbb{P}(Y_0=\cdots =Y_k=y_0,\ Y_{k+1}=\cdots =Y_{d_X-1}=y_0 +1)=\sum\nolimits_{j=0}^{y_0} \mathbb{P}(Y_k=j)-\sum\nolimits_{j=0}^{y_0} \mathbb{P}(Y_{k+1}=j)
\end{aligned}
\end{equation}
for $k\in [d_X-1]$ and $y_0\in[d_Y-1]$; otherwise,
\begin{equation}
\label{eq: jointPO_others_appendix}
\mathbb{P}(Y_0=y_0,\dots ,Y_{d_X-1}=y_{d_X-1})=0.
\end{equation}
Then, any $\lambda'({\cal P}')$ is identifiable.
}

\begin{proof}
Eq.~\eqref{eq: jointPO_others_appendix} follows from Assumption \ref{MITE_D}.
Denote
\begin{equation}
\begin{aligned}
&\theta _{y_0}  =\mathbb{P}(Y_0=\cdots =Y_{d_X-1}=y_0),
\end{aligned}
\end{equation} 
for $y_0\in[d_Y]$, and
\begin{equation}
\begin{aligned}
&\varphi_{y_0,k}  =\mathbb{P}(Y_0=\cdots =Y_k=y_0,Y_{k+1}=\cdots =Y_{d_X-1}=y_0 +1)
\end{aligned}
\end{equation}
for $y_0\in[d_Y-1], k\in[d_X-1]$.
Then, we have
\begin{equation}
\mathbb{P}(Y_i=j)=\theta _j +\sum_{k=0}^{i-1}\varphi_{j-1,k}+\sum_{k=i}^{d_X-2}\varphi_{j,k},
\end{equation}
for $i\in[d_X]$ and $j\in[d_Y]$, where $\varphi_{-1,k}=\varphi_{d_Y-1,k}=0$. This implies
\begin{equation}
    \mathbb{P}(Y_{i+1}=j)-\mathbb{P}(Y_i=j)=\varphi_{j-1,i}-\varphi_{j,i}
\end{equation}
for $i\in[d_X-1]$. Then, for each $k\in[d_X-1]$, we can identify $\varphi_{y_0,k}$ as 
\begin{equation}
    \varphi_{y_0,k}=\sum_{j=0}^{y_0}(\varphi_{j,k}-\varphi_{j-1,k})=\sum_{j=0}^{y_0} (\mathbb{P}(Y_k=j)-\mathbb{P}(Y_{k+1}=j)).
\end{equation}
$\theta _{y_0}$ are identified as
\begin{equation}
\begin{aligned}
    \theta _{y_0}
    & =\mathbb{P}(Y_0=y_0)-\sum_{k=0}^{d_X-2}\varphi_{y_0 ,k}\\
    & =\mathbb{P}(Y_0=y_0)-\sum_{l=0}^{y_0}\sum_{k=0}^{d_X-2}\left( \mathbb{P}(Y_k=l)- \mathbb{P}(Y_{k+1}=l)\right)\\
    & = \mathbb{P}(Y_0=y_0)-\sum_{l=0}^{y_0} (\mathbb{P}(Y_0=l)-\mathbb{P}(Y_{d_X-1}=l))\\
    & = \sum_{l=0}^{y_0} \mathbb{P}(Y_{d_X-1}=l)-\sum_{l=0}^{y_0 -1}\mathbb{P}(Y_0=l).
\end{aligned}
\end{equation}
for $y_0\in [d_Y]$. Then, we have Theorem \ref{thm: identification_of_joint_PO}.
\end{proof}

We have the following three lemmas for the proof of Theorem \ref{thm: identification_jointPO_full}. 

\begin{lemma}
\label{lem: obtain_fulljointPO_from_obsdata}
$\mathbb{P}(Y_0,\dots,Y_{d_X-1},X,Y)$ can be represented by
$\mathbb{P}(Y_0,\dots,Y_{d_X-1}|X)$ and $\mathbb{P}(X,Y)$ as
\begin{equation}
\begin{aligned}
&\mathbb{P}(Y_0,\dots,Y_{x-1},Y_x=y_x,Y_{x+1},\dots,Y_{d_X-1},X=x,Y=y)\\
&=\frac{\mathbb{P}(Y_0,\dots,Y_{x-1},Y_x=y_x,Y_{x+1},\dots,Y_{d_X-1}|X=x)}{\mathbb{P}(Y=y|X=x)}\times\mathbb{I}(y_x=y)\mathbb{P}(X=x,Y=y).
\end{aligned}
\end{equation}
\end{lemma}

\begin{proof}
The joint probabilities
$\mathbb{P}(Y_0,\dots,Y_{d_X-1},X,Y)$ can be represented by
$\mathbb{P}(Y_0,\dots,Y_{d_X-1}|X=x)$ and $\mathbb{P}(X,Y)$ as follows:
\begin{equation}
\begin{aligned}
&\mathbb{P}(Y_0,\dots,Y_{x-1},Y_x=y_x,Y_{x+1},\dots,Y_{d_X-1},X=x,Y=y)\\
&=\mathbb{P}(Y_0,\dots,Y_{x-1},Y_x=y_x,Y_{x+1},\dots,Y_{d_X-1}|X=x,Y=y)\times\mathbb{P}(X=x,Y=y)\\
&=\frac{\mathbb{P}(Y_0,\dots,Y_{x-1},Y_x=y_x,Y_{x+1},\dots,Y_{d_X-1},Y=y|X=x)}{\mathbb{P}(Y=y|X=x)}\times\mathbb{P}(X=x,Y=y)\\
&=\frac{\mathbb{P}(Y_0,\dots,Y_{x-1},Y_x=y_x,Y_{x+1},\dots,Y_{d_X-1},Y_x=y|X=x)}{\mathbb{P}(Y=y|X=x)}\times\mathbb{P}(X=x,Y=y)\\
&=\frac{\mathbb{P}(Y_0,\dots,Y_{x-1},Y_x=y_x,Y_{x+1},\dots,Y_{d_X-1}|X=x)}{\mathbb{P}(Y=y|X=x)}\times\mathbb{I}(y_x=y)\mathbb{P}(X=x,Y=y).
\end{aligned}
\end{equation}
Then, we have Lemma \ref{lem: obtain_fulljointPO_from_obsdata}.
\end{proof}

\begin{lemma}
\label{lem: identification_jointPO_cond-pre}
Under Assumptions \ref{MITE_D}, 
with observational data,
$\mathbb{P}(Y_0,\dots ,Y_{d_X-1}|X=x)$ is given as 
\begin{equation}
\begin{aligned}
&\mathbb{P}(Y_0=\cdots =Y_{d_X-1}=y_0|X=x)= \sum_{l=0}^{y_0} \mathbb{P}(Y_{d_X-1}=l|X=x)-\sum_{l=0}^{y_0 -1}\mathbb{P}(Y_0=l|X=x)
\end{aligned}
\end{equation}
for $y_0\in[d_Y]$ and $x\in[d_X]$, and 
\begin{equation}
\begin{aligned}
&\mathbb{P}(Y_0=\cdots =Y_k=y_0,Y_{k+1}=\cdots =Y_{d_X-1}=y_0 +1|X=x)=\sum_{j=0}^{y_0} (\mathbb{P}(Y_k=j|X=x)-\mathbb{P}(Y_{k+1}=j|X=x))
\end{aligned}
\end{equation}
for $k\in[d_X-1]$, $y_0\in[d_Y-1]$, and $x\in[d_X]$; otherwise,
\begin{equation}
\label{eq: cond_jointPO_others-pre}
\mathbb{P}(Y_0=y_0,\dots ,Y_{d_X-1}=y_{d_X-1}|X=x)=0.
\end{equation}
\end{lemma}

\begin{proof}
    
Eq.~\eqref{eq: cond_jointPO_others-pre} holds from Assumption \ref{MITE_D}.
Denote
\begin{equation}
\begin{aligned}
\theta^x _{y_0} & =\mathbb{P}(Y_0=\cdots =Y_{d_X-1}=y_0|X=x),
\end{aligned}
\end{equation}
\begin{equation}
\begin{aligned}
\varphi^x_{y_0,k} & =\mathbb{P}(Y_0=\cdots =Y_k=y_0,Y_{k+1}=\cdots =Y_{d_X-1}=y_0 +1|X=x)
\end{aligned}
\end{equation}
for $k\in[d_X-1]$, $y_0\in[d_Y-1]$, and $x\in[d_X]$.
Then we have
\begin{equation}
    \mathbb{P}(Y_i=j|X=x)=\theta^x _j +\sum_{k=0}^{i-1}\varphi^x_{j-1,k}+\sum_{k=i}^{d_X-2}\varphi^x_{j,k}
\end{equation}
for $i\in [d_X]$ and $j\in[d_Y]$, where $\varphi^x_{-1,k}=\varphi^x_{d_Y-1,k}=0$. This implies
\begin{equation}
    \mathbb{P}(Y_{i+1}=j|X=x)-\mathbb{P}(Y_i=j|X=x)=\varphi^x_{j-1,i}-\varphi^x_{j,i}
\end{equation}
for $i\in[d_X-1]$. Then, for each $k\in[d_X-1]$, we can rewrite $\varphi^x_{y_0,k}$ as 
\begin{equation}
    \varphi^x_{y_0 ,k}=\sum_{j=0}^{y_0}(\varphi^x_{j,k}-\varphi^x_{j-1,k})=\sum_{j=0}^{y_0} (\mathbb{P}(Y_k=j|X=x)-\mathbb{P}(Y_{k+1}=j|X=x)).
\end{equation}
$\theta^x _{y_0}$ are rewritten as
\begin{equation}
\begin{aligned}
    \theta^x _{y_0}
    & =\mathbb{P}(Y_0=y_0|X=x)-\sum_{k=0}^{d_X-2}\varphi^x_{y_0 ,k}\\
    & =\mathbb{P}(Y_0=y_0|X=x)-\sum_{l=0}^{y_0}\sum_{k=0}^{d_X-2}\left( \mathbb{P}(Y_k=l|X=x)- \mathbb{P}(Y_{k+1}=l|X=x)\right)\\
    & = \mathbb{P}(Y_0=y_0|X=x)-\sum_{l=0}^{y_0} (\mathbb{P}(Y_0=l|X=x)-\mathbb{P}(Y_{d_X-1}=l|X=x))\\
    & = \sum_{l=0}^{y_0} \mathbb{P}(Y_{d_X-1}=l|X=x)-\sum_{l=0}^{y_0 -1}\mathbb{P}(Y_0=l|X=x)
\end{aligned} 
\end{equation}
for $y_0\in[d_Y]$.
\end{proof}

\begin{lemma}
\label{lem: identification_jointPO_cond}
Under Assumptions \ref{exo} and \ref{MITE_D}, 
with observational data,
$\mathbb{P}(Y_0,\dots ,Y_{d_X-1}|X=x)$ is identifiable as 
\begin{equation}
\begin{aligned}
&\mathbb{P}(Y_0=\cdots =Y_{d_X-1}=y_0|X=x)= \sum_{m=0}^{y_0} \mathbb{P}(Y=m|X=d_X-1)-\sum_{m=0}^{y_0 -1}\mathbb{P}(Y=m|X=0)
\end{aligned}
\end{equation}
for $y_0\in[d_Y]$ and $x\in[d_X]$, and 
\begin{equation}
\begin{aligned}
&\mathbb{P}(Y_0=\cdots =Y_k=y_0,Y_{k+1}=\cdots =Y_{d_X-1}=y_0 +1|X=x)=\sum_{m=0}^{y_0} \mathbb{P}(Y=m|X=k)-\sum_{m=0}^{y_0} \mathbb{P}(Y=m|X=k+1)
\end{aligned}
\end{equation}
for $k\in[d_X-1]$, $y_0\in[d_Y-1]$, and $x\in[d_X]$; otherwise,
\begin{equation}
\label{eq: cond_jointPO_others}
\mathbb{P}(Y_0=y_0,\dots ,Y_{d_X-1}=y_{d_X-1}|X=x)=0.
\end{equation}
\end{lemma}

{\bf Remark.} This lemma reduces to Lemma 1 in \citep{Zhang2025} when $d_X=2$ and $X=1$.

\begin{proof}
Assumption \ref{exo} implies
    \begin{equation}
        \mathbb{P}(Y_j=l|X=x)=\mathbb{P}(Y_j=l|X=j)=\mathbb{P}(Y=l|X=j)
    \end{equation}
    for arbitrary $j\in[d_X],\ x\in[d_X],\ l\in[d_Y]$. Hence, from Lemma \ref{lem: identification_jointPO_cond-pre}, we have
    \begin{equation}
        \theta_{y_0} ^x = \sum_{l=0}^{y_0} \mathbb{P}(Y=l|X=d_X-1)-\sum_{l=0}^{y_0 -1}\mathbb{P}(Y=l|X=0)
    \end{equation}
    and
    \begin{equation}
        \varphi^x_{y_0,k}=\sum_{l=0}^{y_0} \mathbb{P}(Y=l|X=k)-\sum_{l=0}^{y_0} \mathbb{P}(Y=l|X=k+1).
    \end{equation}
     Then, we have Lemma \ref{lem: identification_jointPO_cond}.
\end{proof}

{\bf Theorem \ref{thm: identification_jointPO_full}.}
{\it
Under Assumptions \ref{exo} and \ref{MITE_D}, 
given observational distributions $\mathbb{P}(X, Y)$,
$\mathbb{P}(Y_0,\dots ,Y_{d_X-1},X,Y)$ is identifiable as
\begin{equation}
\begin{aligned}
&\mathbb{P}(Y_0=\cdots =Y_{d_X-1}=y_0,X=x,Y=y)=\left(\sum_{m=0}^{y_0} \mathbb{P}(Y=m|X=d_X-1)-\sum_{m=0}^{y_0 -1}\mathbb{P}(Y=m|X=0)\right)\\
&\hspace{6cm}\times\mathbb{P}(X=x)
\end{aligned}
\end{equation}
for $y_0\in [d_Y]$, $x\in [d_X]$, and $y=y_0$, and
\begin{equation}
\begin{aligned}
&\mathbb{P}(Y_0=\cdots =Y_k=y_0,Y_{k+1}=\cdots =Y_{d_X-1}=y_0+1,X=x,Y=y)\\
&=\left(\sum_{l=0}^{y_0} \mathbb{P}(Y=l|X=k)-\sum_{l=0}^{y_0} \mathbb{P}(Y=l|X=k+1)\right)\mathbb{P}(X=x)  
\end{aligned}
\end{equation}
for $k\in [d_X-1]$, $y_0\in [d_Y-1]$, $y=y_0$ for $x\leq k$, and  $y=y_0+1$ for $x=k+1,\dots,d_X-1$; 
otherwise,
\begin{equation}
\mathbb{P}(Y_0=y_0,\dots ,Y_{d_X-1}=y_{d_X-1},X=x,Y=y)=0.
\end{equation}
Then, any $\lambda({\cal P})$ is identifiable. 
}

\begin{proof}
    $\mathbb{P}(Y_0,\dots,Y_{d_X-1},X,Y)$ is categorized into the following cases:
    \begin{equation}
        \begin{cases}
        \text{(A): }Y_0=\cdots =Y_{d_X-1}=y_0 &\text{ for some } y_0\in[d_Y].\\
        \text{(B): }Y_0=\cdots =Y_k=y_0 ,Y_{k+1}=\cdots =Y_{d_X-1}=y_0+1 & \text{ for some } y_0\in[d_Y-1]\text{ and }k\in[d_X-1]. \\
        \text{(C): }\text{Otherwise}.
    \end{cases}
    \end{equation}
    In case \text{(C)}, we have $\mathbb{P}(Y_0,\dots,Y_{d_X-1},X,Y)=0$. 
    In case \text{(A)}, we have $\mathbb{P}(Y_0=\cdots =Y_{d_X-1}=y_0,X=x,Y=y)=0$ if $y\ne y_0$.
    When $y=y_0$, by Lemmas \ref{lem: obtain_fulljointPO_from_obsdata} and \ref{lem: identification_jointPO_cond} , we have
    \begin{equation}
        \begin{aligned}
            &\quad  \mathbb{P}(Y_0=\cdots =Y_{d_X-1}=y_0,X=x,Y=y)\\
            & = \frac{\mathbb{P}(Y_0=\cdots =Y_{d_X-1}=y_0|X=x)}{\mathbb{P}(Y=y|X=x)}\times\mathbb{I}(y_0=y)\mathbb{P}(X=x,Y=y)\\
            & = \mathbb{P}(Y_0=\cdots =Y_{d_X-1}=y_0|X=x)\mathbb{P}(X=x)\\
            & =  \left(\sum_{l=0}^{y_0} \mathbb{P}(Y=l|X=d_X-1)-\sum_{l=0}^{y_0 -1}\mathbb{P}(Y=l|X=0)\right)\mathbb{P}(X=x).
        \end{aligned} 
    \end{equation}
    In case \text{(B)}, by Lemmas \ref{lem: obtain_fulljointPO_from_obsdata} and \ref{lem: identification_jointPO_cond}, we have 
    \begin{equation}
        \begin{aligned}
            &\quad  \mathbb{P}(Y_0=\cdots =Y_k=y_0,Y_{k+1}=\cdots =Y_{d_X-1}=y_0 +1,X=x,Y=y)\\
            & = \frac{\mathbb{P}(Y_0=\cdots =Y_k=y_0,Y_{k+1}=\cdots =Y_{d_X-1}=y_0 +1|X=x)}{\mathbb{P}(Y=y|X=x)}\times\mathbb{I}(Y_x=y)\mathbb{P}(X=x,Y=y).
        \end{aligned}
    \end{equation}
    If $x\leq k, y=y_0$ or $x\geq k+1,y=y_0+1$, it can be calculated as
    \begin{align}
        &  \mathbb{P}(Y_0=\cdots =Y_k=y_0,Y_{k+1}=\cdots =Y_{d_X-1}=y_0 +1|X=x)\mathbb{P}(X=x)\\
        & = \left(\sum_{l=0}^{y_0} \mathbb{P}(Y=l|X=k)-\sum_{l=0}^{y_0} \mathbb{P}(Y=l|X=k+1)\right)\mathbb{P}(X=x).
    \end{align}
    Otherwise, we have $\mathbb{P}(Y_0=\cdots =Y_k=y_0,Y_{k+1}=\cdots =Y_{d_X-1}=y_0 +1,X=x,Y=y)=0$.
    Then, we have Theorem \ref{thm: identification_jointPO_full}.
\end{proof}

\section{TECHNICAL DETAILS}
\label{appA}

This section presents the technical details that support the method introduced in the main body of the paper.

\subsection{Preliminaries on LPs}

We provide some important properties of LP.

{\bf Formulation of LPs.}
LPs in our paper are formulated as
\begin{equation}
\label{eq:LP_general_nonmatrix}
\begin{gathered}
    \text{maximize / minimize }c_1x_1+\cdots +c_nx_n\\\text{ subject to }
    \begin{cases}
        a_{11}x_1+\cdots +a_{1n}x_{n}= b_1\\
        \vdots\\
        a_{t1}x_1+\cdots +a_{tn}x_{n}= b_{t}        
    \end{cases},\quad
    \begin{cases}
        a_{11}'x_1+\cdots +a_{1n}'x_{n}\leq b_1'\\
        \vdots\\
        a_{t'1}'x_1+\cdots +a_{t'n}'x_{n}\leq b_{t'}'        
    \end{cases},\quad x_1,\dots ,x_n\geq 0
\end{gathered}
\end{equation}
We can represent the same LP in one line by using matrices.
Defining
\begin{equation}
\begin{gathered}
    c=(c_1\,\dots\,c_n)^T,\quad x=(x_1\,\dots \,x_n)^T,\\
    A=\begin{pmatrix}
        a_{11}&\cdots&a_{1n}\\\vdots&\ddots & \vdots  \\ a_{t1}&\cdots &a_{tn}
    \end{pmatrix},
    \quad A'=\begin{pmatrix}
        a_{11}'&\cdots&a_{1n}'\\\vdots&\ddots & \vdots  \\ a_{t'1}'&\cdots &a_{t'n}'
    \end{pmatrix},
    \quad b=\begin{pmatrix}
        b_1\\ \vdots \\ b_t
    \end{pmatrix},
    \quad b=\begin{pmatrix}
        b_1'\\ \vdots \\ b_t'
    \end{pmatrix},
\end{gathered}
\end{equation}
we can represent Eq.~\eqref{eq:LP_general_nonmatrix} by
\begin{equation}
\text{maximize /minimize }c^T x \text{ subject to } Ax = b, A'x\leq b', x \geq 0.
\end{equation}
Note that, for vectors $x=(x_1,\dots x_n)$ and $y=(y_1,\dots ,y_n)$,  $x\leq y$ represents $x_1\leq y_1,\dots ,x_n\leq y_n$.

{{\bf Reformulation of Linear Programming.}}
Some references on LPs used in the proof discuss only equality constraints or only inequality constraints.
Then, we show that mixed constraints can be represented as only equality constraints or only inequality constraints.

We can reformulate the inequality and equality constraints in our bounding problems as equality constraints.
We can rewrite
\begin{equation}
\label{eq: LP(inequality form)}
    \begin{cases}
        a_{11}'x_1+\cdots +a_{1u}'x_{n}\leq b_1'\\
        \vdots\\
        a_{t'1}'x_1+\cdots +a_{tu}'x_{n}\leq b_{t'}'        
    \end{cases},\;x_1,\dots ,x_n\geq 0
\end{equation}
into an equivalent constraint system
\begin{equation}
\label{eq: LP(equality form)}
    \begin{cases}
        a_{11}'x_1+\cdots +a_{1n}'x_{n}+w_1= b_1'\\
        a_{21}'x_1+\cdots +a_{2n}'x_{n}+w_2=b_2'\\
        \vdots\\
        a_{t'1}'x_1+\cdots +a_{tn}'x_{n}+w_t= b_{t'}'
    \end{cases},\;x_1,\dots ,x_n,w_1,\dots ,w_{t'}\geq 0.
\end{equation}
We can also reformulate equality constraints
\begin{equation}
    \begin{cases}
        a_{11}x_1+\cdots +a_{1n}x_{n}= b_1\\
        \vdots\\
        a_{t1}x_1+\cdots +a_{tn}x_{n}= b_{t}        
    \end{cases},\;x_1,\dots ,x_n\geq 0
\end{equation}
into
\begin{equation}
    \begin{cases}
        a_{11}x_1+\cdots +a_{1n}x_{n}\leq b_1\\
        \vdots\\
        a_{t1}x_1+\cdots +a_{tn}x_{n}\leq b_{t}        
    \end{cases},
    \quad\begin{cases}
        -a_{11}x_1-\cdots -a_{1n}x_{n}\leq -b_1\\
        \vdots\\
        -a_{t1}x_1+\cdots -a_{tn}x_{n}\leq -b_{t}        
    \end{cases},\quad x_1,\dots ,x_n\geq 0,
\end{equation}
since we have $x=b\iff x\leq b,\,b\leq x \iff x\leq b,\,-x\leq -b$.

{\bf Dual problem of Linear Programming.}
We use the dual problem of LPs in the proof and provide a brief explanation of LP duality.
The dual problem of the LP
\begin{equation}
\label{eq: primalLP}
    \text{maximize }c^Tx\text{ subject to }Ax=b,A'x\leq b',x\geq0
\end{equation}
is an LP 
\begin{equation}
\label{eq: dualLP}
    \text{minimize }b^Ty+(b')^Tz\text{ subject to }A^Ty+(A')^Tz\geq c, z\geq 0
\end{equation}
where $y$ and $z$ are vectors.
If \eqref{eq: primalLP} is feasible and has an optimum (maximum), \eqref{eq: dualLP} is also feasible, the minimum of \eqref{eq: dualLP} exists and is equal to the maximum of \eqref{eq: primalLP}.

\subsection{Explicit constraints implied by Assumptions \ref{monoAS}-\ref{monomatASS}}
\label{APPsubsec:LP_formulation_nonmatrix}

The constraint by strong non-decreasing monotonicity (Assumption \ref{monoAS}) is given as

\begin{equation}
\sum_{y_0=0}^{d_Y-1}\dots \sum_{y_{d_X-1}=0}^{d_Y-1}\sum_{x=0}^{d_X-1} \\
\mathbb{I}(y_0 \leq y_1 \leq \dots \leq y_{d_X-1})\times p_{y_0,\dots ,y_{d_X-1},x}=1,
\end{equation}

the constraint by weak non-decreasing monotonicity for binary treatment (Assumption \ref{weak_mono-bi}) is given as
\begin{equation}
L\leq  \sum_{y_0=0}^{d_Y-1}\sum_{y_1=0}^{d_Y-1}\sum_{x=0}^{d_X-1} \\
\mathbb{I}(y_0 \leq y_1)\times p_{y_0,y_1,x}\leq U,
\end{equation}

the constraint by weak non-decreasing monotonicity (Assumption \ref{monoLB}) is given as
\begin{equation}
L\leq  \sum_{y_0=0}^{d_Y-1}\dots \sum_{y_{d_X-1}=0}^{d_Y-1}\sum_{x=0}^{d_X-1} \\
\mathbb{I}(y_0 \leq y_1 \leq \dots \leq y_{d_X-1})\times p_{y_0,\dots ,y_{d_X-1},x}\leq U,
\end{equation}
the constraint by single strong monotonicity (Assumption \ref{monomatAS}) is given as
\begin{equation}
 \sum_{y_0=0}^{d_Y-1}\dots \sum_{y_{d_X-1}=0}^{d_Y-1}\sum_{x=0}^{d_X-1} \\
\left( \prod_{s>t}\mathbb{I}({\underline{D}}_{st}\leq y_s- y_t\leq {\overline{D}}_{st})\right)\times p_{y_0,\dots ,y_{d_X-1},x}=1,
\end{equation}
and the constraint by single weak monotonicity (Assumption \ref{monomatASS}) is given as
\begin{equation}
L\leq  \sum_{y_0=0}^{d_Y-1}\dots \sum_{y_{d_X-1}=0}^{d_Y-1}\sum_{x=0}^{d_X-1} \\
\left(\prod_{s>t}\mathbb{I}({\underline{D}}_{st}\leq y_s- y_t\leq {\overline{D}}_{st})\right)\times p_{y_0,\dots ,y_{d_X-1},x}\leq U.
\end{equation}

\subsection{Formulation of Our LPs in Matrix Form}
\label{Appsubsec:LP_in_matrix_form}

Since the implementation of LP is formulated in matrix form, we discuss the formulation of our bounding problems in matrix form.

We denote the tensor of parameters as ${\cal T} \in \mathbb{R}^{d_Y^{d_X}\times d_X}$, where the element of ${\cal T}$ at index $(y_0,\dots, y_{d_X-1},x)$ is given by $p_{y_0,\dots, y_{d_X-1},x}$, i.e.,
\begin{equation}
{\cal T}=\{p_{y_0,\dots, y_{d_X-1},x}\}_{y_0,\dots, y_{d_X-1},x}.
\end{equation}
We denote $\text{vec}({\cal T})$ as the vectorization of tensor ${\cal T}$ in lexicographic order, i.e., $(p_{0,\dots, 0,0,0},\dots,)$.
The length of $\text{vec}(\mathcal{T})$ is $d_Y^{d_X}d_X$.
Let ${\boldsymbol 1}$ be a $d_Y^{d_X}d_X$ vector, whose elements are all $1$.
Let ${\cal C}_{\alpha} \in \mathbb{R}^{d_Y^{d_X}\times d_X}$, where $\alpha$ is a function from $\mathbb{R}^{d_Y^{d_X}\times d_X}$ to $\{0,1\}$, denote the tensor whose element at $(y_0,\dots,y_{d_X-1},x)$ is given by the function $\alpha(y_0,\dots,y_{d_X-1},x)$, i.e.,
\begin{equation}
{\cal C}_{\alpha}=\{\alpha(y_0,\dots,y_{d_X-1},x)\}_{y_0,\dots, y_{d_X-1},x}.
\end{equation}

{\bf (1) Total of Probabilities.}
The constraint
\begin{equation}
\begin{aligned}
&\sum_{y_0=0}^{d_Y-1}\dots\sum_{y_{d_X-1}=0}^{d_Y-1}\sum_{x=0}^{d_X-1}p_{y_0,\dots ,y_{d_X-1},x}=1.
\end{aligned}
\end{equation}
can be represented as
\begin{equation}
{\boldsymbol 1}^T \text{vec}({\cal T})=1.
\end{equation}
This is a single constraint.

{\bf (2) Experimental Data.}
The constraint by experimental data
\begin{equation}
\begin{aligned}
    &\sum_{y_0=0}^{d_Y-1}\dots\sum_{y_{l-1}=0}^{d_Y-1}\sum_{y_{l+1}=0}^{d_Y-1}\dots\sum_{y_{d_X-1}=0}^{d_Y-1}\sum_{x=0}^{d_X-1}p_{y_0,\dots y_{l-1},m,y_{l+1}\dots ,y_{d_X-1},x}
    =\mathbb{P}(Y_l=m)
\end{aligned}
\end{equation}
for $l=0,\dots,d_X-1,m=0,\dots,d_Y-2$.
This can be represented as
\begin{equation}
\text{vec}({\cal C}_{\mathbb{I}(y_l=m)})^T\text{vec}({\cal T})=\mathbb{P}(Y_l=m)
\end{equation}
for $l=0,\dots,d_X-1,m=0,\dots,d_Y-2$.
These are $d_X(d_Y-1)$ constraints.

{\bf (3) Observational Data.}
The constraint by observational data
\begin{equation}
    \sum_{y_0=0}^{d_Y-1}\dots\sum_{y_{l-1}=0}^{d_Y-1}\sum_{y_{l+1}=0}^{d_Y-1}\cdots \sum_{y_{d_X-1}=0}^{d_Y-1}p_{y_0,\dots ,y_{l-1},m,y_{l+1},\dots ,y_{d_X-1}, l}
    = \mathbb{P}(X=l,Y=m)
\end{equation}
for $l=0,\dots,d_X-1,m=0,\dots,d_Y-1$.
This can be represented as
\begin{equation}
\text{vec}({\cal C}_{\mathbb{I}(y_l=m,x=l)})^T\text{vec}({\cal T})=\mathbb{P}(X=l,Y=m)
\end{equation}
for $l=0,\dots,d_X-1,m=0,\dots,d_Y-1$. However, we need to exclude the case $(l,m)=(d_X-1,d_Y-1)$ since we have $\mathbb{P}(X=d_X-1,Y=d_Y-1)=1-\sum_{(l,m)\ne(d_X-1,d_Y-1)}\mathbb{P}(X=l,Y=m)$. 
Therefore, observational data derive $d_Xd_Y-1$ constraints.

{\bf (4) Monotonicity Assumption.}
The constraints derived from Assumption \ref{assum_mono_with_prob_s}, i.e., 
\begin{equation}
    L^w\leq \sum_{y_0=0}^{d_Y-1}\dots \sum_{y_{d_X-1}=0}^{d_Y-1}\sum_{x=0}^{d_X-1} \bigg(\Big(\prod_{s,t=0,\dots,d_X-1;s>t}\mathbb{I}({\underline{D}}_{st}^{(w)}\leq y_s- y_t\leq {\overline{D}}_{st}^{(w)})\Big)\times p_{y_0,\dots ,y_{d_X-1},x}\bigg)\leq U^w
\end{equation}
for $w=1,\dots ,W$. This can be expressed as 
\begin{equation}
\begin{aligned}
     \mathrm{vec}(\mathcal{C}_{\prod_{s>t}\mathbb{I}({\underline{D}}_{st}^{(w)}\leq y_s-y_t\leq {\overline{D}}_{st}^{(w)})})^T\cdot \text{vec}(\mathcal{T}) & \geq L^w\,(w=1,\dots ,W),\\
     \mathrm{vec}(\mathcal{C}_{\prod_{s>t}\mathbb{I}({\underline{D}}_{st}^{(w)}\leq y_s-y_t\leq {\overline{D}}_{st}^{(w)})})^T\cdot \text{vec}(\mathcal{T}) & \leq U^w\,(w=1,\dots ,W).
\end{aligned}
\end{equation}
These are $2W$ inequality constraints.

Let
\begin{equation}
\begin{aligned}
    A' & =\Big(\mathrm{vec}(\mathcal{C}_{\prod_{s>t}\mathbb{I}({\underline{D}}_{st}^{(1)}\leq y_s-y_t\leq {\overline{D}}_{st}^{(1)})})\hspace{2mm}\cdots \hspace{2mm} \mathrm{vec}(\mathcal{C}_{\prod_{s>t}\mathbb{I}({\underline{D}}_{st}^{(W)}\leq y_s-y_t\leq {\overline{D}}_{st}^{(W)})})\\
    & \hspace{10mm} -\mathrm{vec}(\mathcal{C}_{\prod_{s>t}\mathbb{I}({\underline{D}}_{st}^{(1)}\leq y_s-y_t\leq {\overline{D}}_{st}^{(1)})})\hspace{2mm}\cdots \hspace{2mm} -\mathrm{vec}(\mathcal{C}_{\prod_{s>t}\mathbb{I}({\underline{D}}_{st}^{(W)}\leq y_s-y_t\leq {\overline{D}}_{st}^{(W)})})\Big)^T,\\
    b' & =(U^1,\dots ,U^W,-L^1,\dots ,-L^W)^T.
\end{aligned}
\end{equation}
$A'$ is a $(2W)\times(d_Xd_Y^{d_X}$)-matrix.
Then our LP is formulated in a matrix form in the following way with respect to the marginal data we have:

{\emph{Formulation of our LP when using experimental data.}}
Let 
\begin{equation}
\begin{aligned}
    A_{\text{exp}} & =\bigg(\boldsymbol{1}\hspace{2mm}
    \hspace{2mm}\text{vec}({\cal C}_{\mathbb{I}(y_0=0)}) \hspace{2mm} \text{vec}({\cal C}_{\mathbb{I}(y_0=1)})\hspace{2mm} \cdots \hspace{2mm} \text{vec}({\cal C}_{\mathbb{I}(y_0=d_Y-1)})\hspace{2mm} \text{vec}({\cal C}_{\mathbb{I}(y_1=0)}) \cdots \hspace{2mm} \text{vec}({\cal C}_{\mathbb{I}(y_{d_X-1}=d_Y-1)})\bigg)^T,\\
   b & = (1\hspace{2mm} \mathbb{P}(Y_0=0)\hspace{2mm} \mathbb{P}(y_0=1) \hspace{2mm} \cdots \hspace{2mm} \mathbb{P}(y_0=d_Y-1)\hspace{2mm} \mathbb{P}(y_1=0) \cdots \hspace{2mm} \mathbb{P}(y_{d_X-1}=d_Y-1))^T. 
\end{aligned}
\end{equation}

$A_{\text{exp}}$ is a $(d_Xd_Y-d_X+1)\times(d_Xd_Y^{d_X})$-matrix.

Then our LP is formulated as
\begin{equation}
\label{eq: our_LP_matrixform_expdata2}
    \text{maximize / minimize $\lambda({\cal P})$ subject to } A_{\text{exp}}\text{vec}(\mathcal{T})=b, A'\text{vec}(\mathcal{T})\leq b',\text{vec}(\mathcal{T})\geq 0.
\end{equation}

We can reformulate it only using equality constraints or only using inequality constraints. Let $\mathcal{W}=(w_1,\dots ,w_{2W})^T$ be a sequence of slack variables. Then, the problem \eqref{eq: our_LP_matrixform_expdata2} is equivalent to
\begin{equation}
    \text{maximize / minimize $\lambda({\cal P})$ subject to } 
    \begin{bmatrix}
        A_{\text{exp}} & O\\
        A' & E_{2W}
    \end{bmatrix}
    \begin{bmatrix}
        \text{vec}(\mathcal{T}) \\ \mathcal{W}
    \end{bmatrix}=
    \begin{bmatrix}
        b \\b'
    \end{bmatrix},\quad \text{vec}(\mathcal{T})\geq 0,\ \mathcal{W}\geq 0
\end{equation}
or
\begin{equation}
    \text{maximize / minimize $\lambda({\cal P})$ subject to } 
    \begin{bmatrix}
        A_{\text{exp}} \\ -A_{\text{exp}} \\ A'
    \end{bmatrix}
        \text{vec}(\mathcal{T})\ \leq
    \begin{bmatrix}
        b \\ -b \\ b'
    \end{bmatrix},\quad \text{vec}(\mathcal{T})\geq 0.
\end{equation}
where $\mathcal{W}=(w_1,\dots ,w_{2W})$ is a vector consisting of slack variables.

{\emph{Formulation of our LP when using observational data.}}
Let
\begin{equation}
\begin{aligned}
    A_{\text{obs}} &=\bigg(\boldsymbol{1}
   \hspace{2mm}\text{vec}({\cal C}_{\mathbb{I}(X=0,Y=0)}) \hspace{2mm} \text{vec}({\cal C}_{\mathbb{I}(X=0,Y=1)})\hspace{2mm} \cdots \hspace{2mm} \text{vec}({\cal C}_{\mathbb{I}(X=0,Y=d_Y-1)})\\
   & \hspace{10mm} \text{vec}({\cal C}_{\mathbb{I}(X=1,Y=0)}) \cdots \hspace{2mm} \text{vec}({\cal C}_{\mathbb{I}(X={d_X-1},Y=d_Y-1)})\bigg)^T,\\
   b & = \bigg(1\hspace{2mm}  
   \hspace{2mm} \mathbb{P}(X=0,Y=0) \hspace{2mm} \mathbb{P}(X=0,Y=1)\hspace{2mm} \cdots \hspace{2mm} \mathbb{P}(X=0,Y=d_Y-1)\hspace{2mm} \\
   & \hspace{10mm} \mathbb{P}(X=1,Y=0) \cdots \hspace{2mm} \mathbb{P}(X={d_X-1},Y=d_Y-1)\bigg)^T.
\end{aligned}
\end{equation}

$A_\text{obs}$ is a $(d_Xd_Y)\times(d_Xd_Y^{d_X})$-matrix.

Then our LP is formulated as
\begin{equation}
\label{eq: our_LP_matrixform3}
    \text{maximize / minimize $\lambda({\cal P})$ subject to } A_{\text{obs}}\text{vec}(\mathcal{T})=b, A'\text{vec}(\mathcal{T})\leq b',\text{vec}(\mathcal{T})\geq 0.
\end{equation}

We can reformulate it only using equality constraints or only using inequality constraints. 
Let $\mathcal{W}=(w_1,\dots ,w_{2W})^T$ be a sequence of slack variables. Then, the problem \eqref{eq: our_LP_matrixform3} is equivalent to
\begin{equation}
    \text{maximize / minimize $\lambda({\cal P})$ subject to } 
    \begin{bmatrix}
        A_{\text{obs}} & O\\
        A' & E_{2W}
    \end{bmatrix}
    \begin{bmatrix}
        \text{vec}(\mathcal{T}) \\ \mathcal{W}
    \end{bmatrix}=
    \begin{bmatrix}
        b \\b'
    \end{bmatrix},\quad \text{vec}(\mathcal{T})\geq 0,\ \mathcal{W}\geq 0
\end{equation}
or
\begin{equation}
    \text{maximize / minimize $\lambda({\cal P})$ subject to } 
    \begin{bmatrix}
        A_{\text{obs}} \\ -A_{\text{obs}} \\ A'
    \end{bmatrix}
        \text{vec}(\mathcal{T})\ \leq
    \begin{bmatrix}
        b \\ -b \\ b'
    \end{bmatrix},\quad \text{vec}(\mathcal{T})\geq 0.
\end{equation}
where $\mathcal{W}=(w_1,\dots ,w_{2W})$ is a vector consisting of slack variables.

{\emph{Formulation of our LP when using both experimental and observational data.}}
Let 
\begin{equation}
\begin{aligned}
    A_{\text{both}} & =\bigg(\boldsymbol{1} \hspace{2mm}\text{vec}({\cal C}_{\mathbb{I}(y_0=0)}) \hspace{2mm} \text{vec}({\cal C}_{\mathbb{I}(y_0=1)})\hspace{2mm} \cdots \hspace{2mm} \text{vec}({\cal C}_{\mathbb{I}(y_0=d_Y-1)})\hspace{2mm} \text{vec}({\cal C}_{\mathbb{I}(y_1=0)}) \cdots \hspace{2mm} \text{vec}({\cal C}_{\mathbb{I}(y_{d_X-1}=d_Y-1)})\\
   &\hspace{10mm}\text{vec}({\cal C}_{\mathbb{I}(X=0,Y=0)}) \hspace{2mm} \text{vec}({\cal C}_{\mathbb{I}(X=0,Y=1)})\hspace{2mm} \cdots \hspace{2mm} \text{vec}({\cal C}_{\mathbb{I}(X=0,Y=d_Y-1)})\\
   & \hspace{10mm} \text{vec}({\cal C}_{\mathbb{I}(X=1,Y=0)}) \cdots \hspace{2mm} \text{vec}({\cal C}_{\mathbb{I}(X={d_X-1},Y=d_Y-1)})\bigg)^T,\\
   b & = (1\hspace{2mm} \mathbb{P}(Y_0=0)\hspace{2mm} \mathbb{P}(y_0=1) \hspace{2mm} \cdots \hspace{2mm} \mathbb{P}(y_0=d_Y-1)\hspace{2mm} \mathbb{P}(y_1=0) \cdots \hspace{2mm} \mathbb{P}(y_{d_X-1}=d_Y-1)\\
   &\hspace{10mm} \mathbb{P}(X=0,Y=0) \hspace{2mm} \mathbb{P}(X=0,Y=1)\hspace{2mm} \cdots \hspace{2mm} \mathbb{P}(X=0,Y=d_Y-1)\hspace{2mm} \\
   & \hspace{10mm} \mathbb{P}(X=1,Y=0) \cdots \hspace{2mm} \mathbb{P}(X={d_X-1},Y=d_Y-1)\bigg)^T.
\end{aligned}
\end{equation}

$A_\text{both}$ is a $(2d_Xd_Y-d_X)\times(d_Xd_Y^{d_X})$-matrix.

Then our LP is formulated as
\begin{equation}
\label{eq: our_LP_matrixform}
    \text{maximize / minimize $\lambda({\cal P})$ subject to } A_{\text{both}}\text{vec}(\mathcal{T})=b, A'\text{vec}(\mathcal{T})\leq b',\text{vec}(\mathcal{T})\geq 0.
\end{equation}

We can reformulate it only using equality constraints or only using inequality constraints. Let $\mathcal{W}=(w_1,\dots ,w_{2W})^T$ be a sequence of slack variables. Then, the problem \eqref{eq: our_LP_matrixform} is equivalent to
\begin{equation}
    \text{maximize / minimize $\lambda({\cal P})$ subject to } 
    \begin{bmatrix}
        A_{\text{both}} & O\\
        A' & E_{2W}
    \end{bmatrix}
    \begin{bmatrix}
        \text{vec}(\mathcal{T}) \\ \mathcal{W}
    \end{bmatrix}=
    \begin{bmatrix}
        b \\b'
    \end{bmatrix},\quad \text{vec}(\mathcal{T})\geq 0,\ \mathcal{W}\geq 0
\end{equation}
or
\begin{equation}
    \text{maximize / minimize $\lambda({\cal P})$ subject to } 
    \begin{bmatrix}
        A_{\text{both}} \\ -A_{\text{both}} \\ A'
    \end{bmatrix}
        \text{vec}(\mathcal{T})\ \leq
    \begin{bmatrix}
        b \\ -b \\ b'
    \end{bmatrix},\quad \text{vec}(\mathcal{T})\geq 0,
\end{equation}
where $\mathcal{W}=(w_1,\dots ,w_{2W})$ is a vector consisting of slack variables.

\subsection{Details of estimation problem}
\label{app_est}

We provide the details of plug-in estimators. 

{\bf Bounding Joint POs/OVs Probabilities and Their Linear Combinations From Finite Samples}
We consider two types of datasets:
(i) $d_X$ experimental datasets
${\cal D}^{k}=\{Y_{(1)}^k,Y_{(2)}^k,\dots,Y_{({N^k})}^k\}$,
for $k=0,\dots,d_X-1$,
where $N^k$ denotes the sample size of the $k$-th experimental dataset ${\cal D}^{k}$, obtained under the intervention $do(X = k)$.
(ii) one observational dataset
${\cal D}=\{(X_{(1)},Y_{(1)}),(X_{(2)},Y_{(2)}),\dots,(X_{(N)},Y_{(N)})\}$,
where $N$ denotes the sample size for the observational dataset ${\cal D}$.
To bound the joint POs/OVs probabilities, we use the empirical probabilities of $\mathbb{P}(Y_x)$ and $\mathbb{P}(X,Y)$, i.e.,
\begin{equation}
\label{emp1}
\begin{aligned}
\hat{\mathbb{P}}(Y_k=j)={N^k}^{-1}\sum_{i=1}^{N^k}\mathbb{I}(Y_{(i)}^k=j),
\end{aligned}
\end{equation}
\begin{equation}
\label{emp2}
\begin{aligned}
\hat{\mathbb{P}}(X=l,Y=m)={N}^{-1}\sum_{i=1}^{N}\mathbb{I}(X_{(i)}=l,Y_{(i)}=m).
\end{aligned}
\end{equation}

Substituting the empirical probabilities in Eqs.~\eqref{emp1} and \eqref{emp2} into the LP constraints 
yields the plug-in estimators of the lower and upper bounds of $\lambda({\cal P})$, i.e.,
$[\hat{\underline{\lambda}}({\cal P}),\hat{\overline{\lambda}}({\cal P})]$.

Note that even if Assumption~\ref{assum_mono_with_prob_s} holds, sampling errors in the empirical probabilities may render the LPs infeasible.
Throughout, we assume that the LPs constructed from empirical probabilities are feasible.

{\bf Consistency.}
The estimates of the bounds on $\lambda({\cal P})$ are consistent estimators as follows.

\begin{theorem}[Consistency]
\label{thm: convergence}
Under Assumption~\ref{assum_mono_with_prob_s},
if the plug-in LPs are feasible,
$\hat{\underline{\lambda}}({\cal P})$ and $\hat{\overline{\lambda}}({\cal P})$ converge in probability to $\underline{\lambda}({\cal P})$ and $\overline{\lambda}({\cal P})$.
The convergence rate is ${\cal O}_p(\sum_{k=0}^{d_X-1}{N^k}^{-1/2})$ for experimental data, ${\cal O}_p({N}^{-1/2})$ for observational data, and ${\cal O}_p(\sum_{k=0}^{d_X-1}{N^k}^{-1/2}+{N}^{-1/2})$ for combined experimental and observational data.
\end{theorem}

{\bf Estimating Joint POs/OVs Probabilities and Their Linear Combinations.}
By plugging the empirical probabilities in Eqs.~\eqref{emp1} and \eqref{emp2} into the identification results in Theorems \ref{thm: identification_of_joint_PO} and \ref{thm: identification_jointPO_full}, we can estimate the joint POs/OVs probabilities and their linear combinations from finite samples. 
Explicitly, from Theorem \ref{thm: identification_of_joint_PO},
the estimates of $\mathbb{P}(Y_0,\dots ,Y_{d_X-1})$ is 
\begin{equation}
\begin{aligned}
&\hat{\mathbb{P}}(Y_0=\cdots =Y_{d_X-1}=y_0)= \sum_{j=0}^{y_0} \hat{\mathbb{P}}(Y_{d_X-1}=j)-\sum_{j=0}^{y_0 -1}\hat{\mathbb{P}}(Y_0=j)
\end{aligned}
\end{equation}
for $y_0\in [d_Y]$, and 
\begin{equation}
\begin{aligned}
&\hat{\mathbb{P}}(Y_0=\cdots =Y_k=y_0,\ Y_{k+1}=\cdots =Y_{d_X-1}=y_0 +1)=\sum_{j=0}^{y_0} \hat{\mathbb{P}}(Y_k=j)-\sum_{j=0}^{y_0} \hat{\mathbb{P}}(Y_{k+1}=j)
\end{aligned}
\end{equation}
for $k\in [d_X-1]$ and $y_0\in[d_Y-1]$; otherwise,
\begin{equation}
\hat{\mathbb{P}}(Y_0=y_0,\dots ,Y_{d_X-1}=y_{d_X-1})=0.
\end{equation}
The estimate of $\lambda'({\cal P}')$ is given by the linear combination of them.
From Theorem \ref{thm: identification_jointPO_full}, the estimates of $\mathbb{P}(Y_0,\dots ,Y_{d_X-1},X,Y)$ are 
\begin{equation}
\begin{aligned}
&\hat{\mathbb{P}}(Y_0=\cdots =Y_{d_X-1}=y_0,X=x,Y=y)=\left(\sum_{m=0}^{y_0} \hat{\mathbb{P}}(Y=m|X=d_X-1)-\sum_{m=0}^{y_0 -1}\hat{\mathbb{P}}(Y=m|X=0)\right)\\
&\hspace{6cm}\times\hat{\mathbb{P}}(X=x)
\end{aligned}
\end{equation}
for $y_0\in [d_Y]$, $x\in [d_X]$, and $y=y_0$, and
\begin{equation}
\begin{aligned}
&\hat{\mathbb{P}}(Y_0=\cdots =Y_k=y_0,Y_{k+1}=\cdots =Y_{d_X-1}=y_0+1,X=x,Y=y)\\
&=\left(\sum_{l=0}^{y_0} \hat{\mathbb{P}}(Y=l|X=k)-\sum_{l=0}^{y_0} \hat{\mathbb{P}}(Y=l|X=k+1)\right)\hat{\mathbb{P}}(X=x)  
\end{aligned}
\end{equation}
for $k\in [d_X-1]$, $y_0\in [d_Y-1]$, $y=y_0$ for $x\leq k$, and  $y=y_0+1$ for $x=k+1,\dots,d_X-1$; 
otherwise,
\begin{equation}
\mathbb{P}(Y_0=y_0,\dots ,Y_{d_X-1}=y_{d_X-1},X=x,Y=y)=0.
\end{equation}
The estimate of $\lambda({\cal P})$ is given by the linear combination of them.

These are consistent estimators.

\begin{theorem}[Consistency]
\label{theor5}
Under Assumption \ref{MITE_D}, $\lambda'(\hat{{\cal P}'})$ converges in probability to $\lambda'({\cal P}')$ at the rate ${\cal O}_p(\sum_{k=0}^{d_X-1}{N^k}^{-1/2})$. 
\end{theorem}

\begin{theorem}[Consistency]
\label{theor6}
Under Assumptions \ref{exo} and \ref{MITE_D}, $\lambda(\hat{{\cal P}})$ converges in probability to $\lambda({\cal P})$ at the rate ${\cal O}_p(N^{-1/2})$. 
\end{theorem}

{\bf Proof.}
We provide proofs of theorems in this subsection.

We first show the following lemma to prove Theorem \ref{thm: convergence}.
\begin{lemma}
\label{lem: finite_candidate}
    Let $V=\{x\in\mathbb{R}^d:Ax\geq c\}$ be a nonempty polyhedron.
    For any $h\in\mathbb{R}^d$ such that $\alpha:=\max_{x\in V} h^T x$ is finite and attained on $V$, there exists a finite set $S\subseteq V$, depending only on $V$ (not on $h$), such that $\max_{x\in V} h^T x=\max_{a\in S}h^T a$.
\end{lemma}
\begin{proof}
    Fix $h$ with finite attained optimum $\alpha$. Define the maximizer set
    \begin{equation}
    F(h):=\{x\in V\mid h^T x=\alpha\}.
    \end{equation}
    Then $F(h)$ is the nonempty intersection of $V$ with the hyperplane $\{x\mid  h^T x=\alpha\}$; hence $F(h)$ is an exposed face of $V$.
    Choose one representative point $r(F)$ from each nonempty face $F$ of $V$, and set $S=\{r(F)\mid F\text{ is a nonempty face of }V\}$.
    Then $S$ is finite since a polyhedron has only finitely many faces, and $S$ is independent of $h$.
    $F(h)\cap V(\ne \emptyset)$ is one of those faces, since $V$ is contained in the closed half-space $\{x\mid h^T x\leq \alpha\}$. Then $z\coloneqq r(F(h))\in S$ and $\alpha =h^Tz$. This proves the lemma.
\end{proof}

\begin{proof}[Proof of Theorem \ref{thm: convergence}]
We remark that the feasible regions of our LPs are subsets of $[0,1]^{\mathcal{X}}$, where $\mathcal{X}$ is the number of parameters. As a consequence, since we assumed the plug-in LPs are feasible, their optima always exist and are bounded.

Passing from the mixed system in \eqref{eq: LP(inequality form)} to the equality form in \eqref{eq: LP(equality form)} by introducing nonnegative slack variables preserves feasibility, the right-hand-side vector and optimum. Therefore, we may assume that the  plug-in LPs are represented in the equality form.

Let us denote the corresponding LP whose optima are $\underline{\lambda}(\mathcal{P})$ and $\overline{\lambda}(\mathcal{P})$ by 
\begin{equation}
    \label{eq: LP_b (in the proof of thm2)}
   \text{maximize/minimize } c^Tx\text{ with respect to }Ax=b,x\geq 0.
\end{equation}
We also denote the corresponding LP whose optima are $\hat{\underline{\lambda}}({\cal P})$ and $\hat{\overline{\lambda}}({\cal P})$ by 
\begin{equation}
    \label{eq: LP_b_hat (in the proof of thm2)}
   \text{maximize/minimize } c^Tx\text{ with respect to } Ax=\hat{b},x\geq 0.
\end{equation}
Note that $A,c$ are common in \eqref{eq: LP_b (in the proof of thm2)} and \eqref{eq: LP_b_hat (in the proof of thm2)}.

First, we will show the statement on UB.
The feasible region of the dual of \eqref{eq: LP_b (in the proof of thm2)} and that of \eqref{eq: LP_b_hat (in the proof of thm2)} are $\{y\mid A^Ty\geq c\}$ when the primal LP is the maximization, and do not depend on $b$ or $\hat{b}$. 
It is not empty, since primal LPs have optima.
We remark that any dual LP of \eqref{eq: LP_b (in the proof of thm2)} or \eqref{eq: LP_b_hat (in the proof of thm2)} has an optimum.
Therefore, by Lemma \ref{lem: finite_candidate}, we can take a \emph{finite} set $\mathcal{C}$ such that the optima of any dual LP of \eqref{eq: LP_b (in the proof of thm2)} or \eqref{eq: LP_b_hat (in the proof of thm2)} are achieved at some point of $\mathcal{C}$. We take $M=\max_{a\in \mathcal{C}}\Vert a\Vert$.

Let $u(b)\in \mathcal{C}$ be a solution of the dual of \eqref{eq: LP_b (in the proof of thm2)}. By properties of dual LP, we have $\overline{\lambda}({\cal P})=b^Tu(b)$. We also have $\hat{\overline{\lambda}}({\cal P})=\min_{y; A^Ty\geq c}\hat{b}^Ty\leq \hat{b}^Tu(b)$.
Then we obtain
\begin{equation}
        \hat{\overline{\lambda}}({\cal P})-\overline{\lambda}({\cal P})\leq (\hat{b}-b)^Tu(b)\leq\Vert \hat{b}-b\Vert\Vert u(b)\Vert\leq M\Vert \hat{b}-b\Vert .
\end{equation}
Changing the roles of $b$ and $\hat{b}$, we also obtain $\overline{\lambda}({\cal P})-\hat{\overline{\lambda}}({\cal P})\leq M\Vert \hat{b}-b\Vert$.
Summing up, we have 
\begin{equation}
    \label{eq: Lip_continuous_UB}
    |\overline{\lambda}({\cal P})-\hat{\overline{\lambda}}({\cal P})|\leq M\Vert \hat{b}-b\Vert.
\end{equation}
Then, for any $\varepsilon>0$, we have
\begin{equation}
    \mathbb{P}\big(|\hat{\overline{\lambda}}({\cal P})- \overline{\lambda}(\mathcal{P})|>\varepsilon \big)\leq \mathbb{P}\big(\Vert \hat{b}-b\Vert>\varepsilon/M)\big).
\end{equation}
Since $\hat{b}$ converges in probability to $b$, $\hat{\overline{\lambda}}({\cal P})$ converges in probability to $\overline{\lambda}(\mathcal{P})$.
The convergence rate of $\hat{\overline{\lambda}}({\cal P})$ is the same as that of $\hat{b}$ since \eqref{eq: Lip_continuous_UB} holds. 
In our case, the convergence rate is ${\cal O}_p(\sum_{k=0}^{d_X-1}\frac{1}{\sqrt{N^k}})$ for experimental data, ${\cal O}_p(\frac{1}{\sqrt{N}})$ for observational data, and ${\cal O}_p(\sum_{k=0}^{d_X-1}\frac{1}{\sqrt{N^k}}+\frac{1}{\sqrt{N}})$ when both are combined.

Next we show the statement on LB. 
For any vector $z$, minimizing $c^Tx$ with respect to $Ax=b,x\geq 0$ is equivalent to maximizing $-c^Tx$ with respect to $Ax=b,x\geq 0$. 
Therefore, we see that the feasible region of the dual of \eqref{eq: LP_b (in the proof of thm2)} and that of \eqref{eq: LP_b_hat (in the proof of thm2)} are $\{y\mid A^Ty\geq -c\}$ when the primal LP is the maximization, and do not depend on $b$ or $\hat{b}$. 
It is not empty, since primal LPs have optima.
We remark that any dual LP of \eqref{eq: LP_b (in the proof of thm2)} or \eqref{eq: LP_b_hat (in the proof of thm2)} has an optimum.
Therefore, by Lemma \ref{lem: finite_candidate}, we can take a \emph{finite} set $\mathcal{C'}$ such that the optima of any dual LP of \eqref{eq: LP_b (in the proof of thm2)} or \eqref{eq: LP_b_hat (in the proof of thm2)} are achieved at some point of $\mathcal{C'}$. We take $M'=\max_{a\in \mathcal{C'}}\Vert a\Vert$.

Let $v(b)\in \mathcal{C'}$ be a solution of the dual of \eqref{eq: LP_b (in the proof of thm2)}. By properties of dual LP, we have $\underline{\lambda}({\cal P})=b^Tv(b)$. We also have $\hat{\underline{\lambda}}({\cal P})=\max_{y; A^Ty\geq -c}\hat{b}^Ty\geq \hat{b}^Tv(b)$.
Then we obtain
\begin{equation}
        \underline{\lambda}({\cal P})-\hat{\underline{\lambda}}({\cal P})\leq (b-\hat{b})^Tv(b)\leq\Vert \hat{b}-b\Vert\Vert v(b)\Vert\leq M'\Vert \hat{b}-b\Vert .
\end{equation}
Changing the roles of $b$ and $\hat{b}$, we also obtain $\hat{\underline{\lambda}}({\cal P})-\underline{\lambda}({\cal P})\leq M'\Vert \hat{b}-b\Vert$.
Summing up, we have 
\begin{equation}
    \label{eq: Lip_continuous_LB}
    |\underline{\lambda}({\cal P})-\hat{\underline{\lambda}}({\cal P})|\leq M'\Vert \hat{b}-b\Vert.
\end{equation}
Then, for any $\varepsilon>0$, we have
\begin{equation}
    \mathbb{P}\big(|\hat{\underline{\lambda}}({\cal P})- \underline{\lambda}(\mathcal{P})|>\varepsilon \big)\leq \mathbb{P}\big(\Vert \hat{b}-b\Vert>\varepsilon/M')\big).
\end{equation}
Since $\hat{b}$ converges in probability to $b$, $\hat{\underline{\lambda}}({\cal P})$ converges in probability to $\underline{\lambda}(\mathcal{P})$.
The convergence rate of $\hat{\underline{\lambda}}({\cal P})$ is the same as that of $\hat{b}$ since \eqref{eq: Lip_continuous_LB} holds. 
In our case, the convergence rate is ${\cal O}_p(\sum_{k=0}^{d_X-1}\frac{1}{\sqrt{N^k}})$ for experimental data, ${\cal O}_p(\frac{1}{\sqrt{N}})$ for observational data, and ${\cal O}_p(\sum_{k=0}^{d_X-1}\frac{1}{\sqrt{N^k}}+\frac{1}{\sqrt{N}})$ when both are combined.
\end{proof}

\begin{proof}[Proof of Theorem \ref{theor5}]
$\lambda'(\hat{{\cal P}'})$ is a linear function of $\mathbb{P}(Y|X)$.
Then, by the delta method, $\lambda'(\hat{{\cal P}'})$ converges in probability to $\lambda'({\cal P}')$ at the rate ${\cal O}_p\left(\sum_{k=0}^{d_X-1}\frac{1}{\sqrt{N^k}}\right)$. 
Then, we have Theorem \ref{theor5}.
\end{proof}

\begin{proof}[Proof of Theorem \ref{theor6}]
$\lambda(\hat{{\cal P}})$ is a linear function of $\mathbb{P}(Y|X)$ and $\mathbb{P}(X)$, respectively.
Then, by the delta method, $\lambda(\hat{{\cal P}})$ converges in probability to $\lambda({\cal P})$  at the rate ${\cal O}_p\left(\frac{1}{\sqrt{N}}\right)$. 
Then, we have Theorem \ref{theor6}.
\end{proof}

\subsection{Computational complexity to solve our bounding problems}
\label{app_comp}
Next, we discuss the computational complexity.
There are two main approaches for solving LPs: the interior-point methods \citep{Karmarkar1984} and the simplex method \citep{RinnooyKan1981}.

\begin{theorem}
\label{thm: interior-point method}
The computational complexity of solving our bounding problems is polynomial when using interior-point methods, on the order of 
$\mathcal{O}((d_Xd_Y^{d_X}+2W)^{3.5}(d_Xd_Y^{d_X}+2W+1)(d_Xd_Y-d_X+1+2W)+(\sum_{k=0}^{d_X-1}N^k)d_Y)$ when using experimental data, 
$\mathcal{O}(d_Xd_Y^{d_X}+2W)^{3.5}(d_Xd_Y^{d_X}+2W+1)d_Xd_Y+Nd_Xd_Y)$ when using observational data, 
${\cal O}((d_Xd_Y^{d_X}+2W)^{3.5}(d_Xd_Y^{d_X}+2W+1)(2d_Xd_Y-d_X+2W)+(\sum_{k=0}^{d_X-1}N^k)d_Y+Nd_Xd_Y)$ when using both experimental and observational data.
\end{theorem}

Therefore, our bounding problems can be solved in polynomial time.
The simplex method has exponential worst-case complexity as follows:
\begin{theorem}
\label{thm: simplex method}
If we use the simplex method to solve LP, the worst-case computational complexity of solving our bounding problems is on the order of 
$\mathcal{O}((d_Xd_Y^{d_X})^{2(d_Xd_Y-d_X+1+W)}+(\sum_{k=0}^{d_X-1}N^k)d_Y)$ when using experimental data, 
$\mathcal{O}((d_Xd_Y^{d_X})^{2(d_Xd_Y+W)}+Nd_Xd_Y)$ when using observational data, 
$\mathcal{O}((d_Xd_Y^{d_X})^{2(2d_Xd_Y-d_X+W)}+(\sum_{k=0}^{d_X-1}N^k)d_Y+Nd_Xd_Y)$ when using both experimental and observational data.
\end{theorem}

However, in practice, the simplex method rarely exhibits its exponential worst-case behavior and is often faster than interior-point algorithms \citep{Vanderbei2001}.

The calculation of $\lambda'(\hat{{\cal P}'})$ using Theorem \ref{thm: identification_of_joint_PO} requires $\mathcal{O}(d_Y\sum_{k=0}^{d_X-1}N^k +d_Xd_Y)$ computations, while $\lambda(\hat{{\cal P}})$ using Theorem \ref{thm: identification_jointPO_full} requires $\mathcal{O}(d_Xd_YN+d_X^2d_Y^2)$ computations.

{\bf Proof.}
We provide proofs of theorems in this subsection.

\begin{proof}[Proof of Theorem \ref{thm: interior-point method}]
We note that our LP has $d_Xd_Y^{d_X}$ variables, $d_Xd_Y-d_X+1$ (resp. $d_Xd_Y,\ 2d_Xd_Y-d_X$) equality constraints and $2W$ inequality constraints when we have experimental (resp. observational, both experimental and observational) data. 
As stated in Appendix \ref{appA}, we can reformulate our LP to an LP with $d_Xd_Y^{d_X}+2W$ variables and $d_Xd_Y-d_X+1+2W$ (resp. $d_Xd_Y+2W,\ 2d_Xd_Y-d_X+2W$) equality constraints we have experimental (resp. observational, both experimental and observational) data.

For an LP with $\mathfrak{n}$ variables and $\mathfrak{m}$ equality constraints, the computational complexity of solving LP using the interior-point method is on the order of  $\mathcal{O}(\mathfrak{n}^{3.5}\ell)$, where $\ell$ denotes the length of input \citep{Karmarkar1984}. {As for bounding of joint distribution of POs, since the coefficients of constraints and objectives take 0 or 1, $\ell$ is on the order of $\mathcal{O}(\mathfrak{nm+n})$.}
Therefore, the computational complexity of solving our LP is on the order of 
$\mathcal{O}((d_Xd_Y^{d_X}+2W)^{3.5}(d_Xd_Y^{d_X}+2W+1)(d_Xd_Y-d_X+1+2W))$ when using experimental data,
$\mathcal{O}((d_Xd_Y^{d_X}+2W)^{3.5}(d_Xd_Y^{d_X+1}+2W+1)(d_Xd_Y))$ when using observational data,
${\cal O}((d_Xd_Y^{d_X}+2W)^{3.5}(d_Xd_Y^{d_X}+2W+1)(2d_Xd_Y-d_X+2W))$ when using both experimental and observational data.

We also need a calculation to estimate $\hat{\mathbb{P}}(Y_k)$ or $\hat{\mathbb{P}}(X,Y)$ so that we can impose marginal constraints. 
\eqref{emp1} and \eqref{emp2} imply that the amount of calculation to impose marginal constraints is on the order of $(\sum_{k=0}^{d_X-1}N^k)d_Y$, $Nd_Xd_Y$, and $(\sum_{k=0}^{d_X-1}N^k)d_Y+Nd_Xd_Y$, when using experimental data, observational data and both experimental and observational data, respectively.
Since these calculations are independent from solving LPs, the computational complexity of solving our bounding problems is a sum of that of solving LP and that of estimating marginal probabilities.
This proves theorem \ref{thm: interior-point method}.
\end{proof}

\begin{proof}[Proof of Theorem \ref{thm: simplex method}]
We note that our LP has $d_Xd_Y^{d_X}$ variables, $d_Xd_Y-d_X+1$ (resp. $d_Xd_Y,\ 2d_Xd_Y-d_X$) equality constraints and $2W$ inequality constraints when we have experimental (resp. observational, both experimental and observational) data. 
As stated in Appendix \ref{appA}, we can reformulate our LP to an LP with $d_Xd_Y^{d_X}$ variables and $2(d_Xd_Y-d_X+1)+2W$ (resp. $2d_Xd_Y+2W,\ 2(2d_Xd_Y-d_X)+2W$) inequality constraints.

For an LP with $\mathfrak{n}$ variables, and $\mathfrak{m}$ inequality constraints, the computational complexity of solving LP using the simplex method is on the order of  $\mathcal{O}(\mathfrak{n^m})$ in the worst case \citep{RinnooyKan1981}. Then, in our case, the computational complexity of solving our LP is on the order of 
$\mathcal{O}((d_Xd_Y^{d_X})^{2(d_Xd_Y-d_X+1+W)})$ when using experimental data, 
$\mathcal{O}((d_Xd_Y^{d_X})^{2(d_Xd_Y+W)})$ when using observational data, 
$\mathcal{O}((d_Xd_Y^{d_X})^{2(2d_Xd_Y-d_X+W)})$ when using both experimental and observational data.

We also need a calculation to estimate $\hat{\mathbb{P}}(Y_k)$ or $\hat{\mathbb{P}}(X,Y)$ so that we can impose marginal constraints. 
\eqref{emp1} and \eqref{emp2} imply that the amount of calculation to impose marginal constraints is on the order of $(\sum_{k=0}^{d_X-1}N^k)d_Y$, $Nd_Xd_Y$, and $(\sum_{k=0}^{d_X-1}N^k)d_Y+Nd_Xd_Y$, when using experimental data, observational data, and both experimental and observational data, respectively.
Since these calculations are independent from solving LPs, the computational complexity of solving our bounding problems is a sum of that of solving LP and that of estimating marginal probabilities.
This proves theorem \ref{thm: simplex method}.
\end{proof}

\section{NUMERICAL EXPERIMENTS}

\label{appC}

In this section, we provide details of the numerical experiments.

We illustrate how MAs can tighten the bounds on the joint POs/OVs probabilities  and their linear combinations, and how point estimates can be obtained under additional identification assumptions.
We use a computer with AMD Ryzen 7 8840U 3.30GHz processor and 16GB of RAM. 
The source code is available at 
(\url{https://anonymous.4open.science/r/AISTATS2026_bounding-65A3}).

\subsection{Bounding Under Monotonicity Assumption}

\label{subsec: numerical experiment of bounding}

{\bf Baselines.}
Our bounds for the joint POs/OVs probabilities  are compared against those in \citet{ALi2024} when incorporating both experimental and observational data.
Their bounds are given in closed form but are not sharp for PoC involving multiple hypothetical terms (multiple PO terms).

{\bf Settings.}
We set $d_X=d_Y=3$ and the number of parameters is $243$.
We set the values of ${\cal P}$ as presented in Table \ref{tab:placeholder}. 
We evaluate (a) the joint POs probability  $\mathbb{P}(Y_0=0,Y_1=0,Y_2=1)$ using experimental, observational, and combined data, (b) the posterior causal effects $\mathbb{E}[Y_1-Y_0|X=2,Y=2]$ using observational and combined data,
and (c) the second moment of causal effects $\mathbb{E}[(Y_1-Y_0)^2]$ using experimental, observational, and combined data.
We solve the LPs using the Simplex algorithm implemented in the R package ``Rglpk" (\url{https://cran.r-project.org/web/packages/Rglpk/index.html}), which provides an R interface to the GNU Linear Programming Kit (GLPK).
We set the true joint POs/OVs probabilities as specified in Table \ref{tab:placeholder}, 
and then draw an i.i.d. sample of size $N$ for the experimental and observational datasets, respectively. 
We perform 100 simulation runs and report the sample mean along with 95\% confidence intervals (CIs).
We exclude infeasible simulations and report the results of 100 feasible simulations.
Tables~\ref{tab: runningtime_Py0y1y2} and \ref{tab: runningtime_E_y1y0_x2y2} report the running times (in milliseconds) for a single simulation.
The running times are quite fast across all situations.

To compare the ``almost surely" assumptions, we first consider the following MAs: 
\begin{enumerate}[label=(\roman*).]
  \setlength\itemsep{2pt}
  \setlength\parskip{0pt}
  \setlength\parsep{0pt}
\item  No Additional Assumption,
\item  $Y_0 \leq Y_1$ almost surely (Assumption \ref{monomatAS}),
\item  $Y_1 \leq Y_2$ almost surely (Assumption \ref{monomatAS}),
\item  $Y_0 \leq Y_1 \leq Y_2$ almost surely (Assumption \ref{monoAS}),
\end{enumerate}
The setting 
in Table \ref{tab:placeholder} satisfies all of the MAs listed above.
Assumptions (ii) and (iii) are both stronger than Assumption (i) but weaker than Assumption (iv); moreover, there is no ordering relationship between Assumptions (ii) and (iii).
We next additionally compare the probability-limited assumption (Assumption \ref{monoLB}) varying threshold $L$, i.e.,
\begin{enumerate}
\setlength\itemsep{2pt}
\setlength\parskip{0pt}
\setlength\parsep{0pt}
\item[(v).] $ L \leq \mathbb{P}(Y_0 \leq Y_1\leq Y_2)$. 
\end{enumerate}
The assumption becomes stronger as $L$ grows.

We finally compare multiple probability-limited assumptions (Assumption \ref{assum_mono_with_prob_s}) varying thresholds $L_1$, $L_2$, and $L_3$,
\begin{enumerate}
\setlength\itemsep{2pt}
\setlength\parskip{0pt}
\setlength\parsep{0pt}
\item[(vi).] $L_1 \leq \mathbb{P}(0\leq Y_1-Y_0\leq 1)$, $L_2 \leq \mathbb{P}(0\leq Y_2-Y_1\leq 1)$, $L_3 \leq \mathbb{P}(0\leq Y_2-Y_0\leq 1)$,
\end{enumerate}
under the setting in Table \ref{tab:p(y0y1y2xy)_steprestriction}. We have $0\leq Y_1-Y_0\leq 1$, $0\leq Y_2-Y_1\leq 1$, and $0\leq Y_2-Y_0\leq 1$ under this setting. 
As $L_1$, $L_2$, and $L_3$ increase, the assumptions become more restrictive.

{\bf Results.}
First, Table~\ref{tab:results} reports the estimates of the bounds of $\mathbb{P}(Y_0=0, Y_1=0, Y_2=1)$, while Table~\ref{tab:results_condATE} presents the estimates of the bounds of $\mathbb{E}[Y_1-Y_0|X=2,Y=2]$ under almost surely monotonicities with sample size $N=10, 100, 1000, 10000$.
Table \ref{tab:results_2ndMoment} reports the estimates of the bounds of $\mathbb{E}[(Y_1-Y_0)^2]$ under Assumptions (i)--(iv).
In Table~\ref{tab:results}, when using experimental data, Assumption (ii) does not yield tighter bounds than those obtained without additional assumptions. 
In contrast, Assumptions (iii) and (iv) provide tighter bounds compared to the no-assumption case.
When using observational data, Assumptions (i)–(iv) do not yield tighter bounds than those obtained without additional assumptions. 
When using experimental and observational data, Assumption (ii) does not yield tighter bounds than those obtained without additional assumptions. 
In contrast, Assumptions (iii) and (iv) provide tighter bounds compared to the no-assumption case.
The bounds by \citep{Li2024} are the same as our bounds without additional assumptions.
However, these bounds are not sharp, as demonstrated in Table \ref{tab: Li-Pearl_is_not_tight}.
In Table~\ref{tab:results_condATE}, when using only observational data, all assumptions do not yield tighter bounds than those obtained without additional assumptions. 
When using experimental and observational data,
Assumption (iii) does not yield tighter bounds than those obtained without additional assumptions. 
In contrast, Assumptions (ii) and (iv) provide tighter bounds compared to the no-assumption case.
In Table \ref{tab:results_2ndMoment}, we also observe that Assumptions (ii) and (iv) provide tighter bounds, while Assumption (iii) does not yield tighter bounds than the case with no additional MAs.
All 95\% CIs, both for the estimates under the identification assumptions and for the bounds, become narrower as the sample size increases.
When $N=10000$, all the estimates are accurate.

Next, Table \ref{tab:results_theta_exp} (resp. Table \ref{tab:results_theta_exp2}, Table \ref{tab:results_theta_2ndMoment}) reports the estimates of the values of $\mathbb{P}(Y_0=0, Y_1=0, Y_2=1)$ (resp. $\mathbb{E}[Y_1-Y_0|X=2,Y=2]$, $\mathbb{E}[(Y_1-Y_0)^2]$), varying the threshold $L$ of Assumption (v) when $N=1000$.
The bounds become tight as the threshold $L$ grows large,
and it is visualized in Figure \ref{fig:placeholder} (resp. Figure \ref{fig:placeholder2}, Figure \ref{fig:EY1Y0^2}).
Table \ref{tab:L1L2L3_prob} reports the estimates of the values of $\mathbb{P}(Y_0=0, Y_1=0, Y_2=1)$, varying the thresholds $L_1,L_2$ and $L_3$ of Assumption (vi) when $N=1000$.
Tables \ref{tab:L1L2L3_expect_obs} and \ref{tab:L1L2L3_expect_both} (resp. Tables \ref{tab:L1L2L3_2moment_expboth} and \ref{tab:L1L2L3_2moment_obs}) report the estimates of the values of $\mathbb{E}[Y_1-Y_0|X=2,Y=2]$ (resp. $\mathbb{E}[(Y_1-Y_0)^2]$), varying the thresholds $L_1,L_2$ and $L_3$ of Assumption (vi) when $N=1000$.
As $L_1$, $L_2$, and $L_3$ increase, the bounds become tight.

We provide details of the experimental setting for the bounding problem under MAs.
We set the value $\mathbb{P}(Y_0,Y_1,Y_2,X,Y)$ as in Table \ref{tab:placeholder}.
These values satisfy the condition $Y_0\leq Y_1\leq Y_2$.

\newpage

\begin{table}[H]
    \centering
    \caption{Setting of simulations for Assumptions (i)-(v) in Section \ref{subsec: numerical experiment of bounding}}
    \label{tab:placeholder}    
    {\small
\begin{tabular}{c|ccc|ccc|ccc}
    \hline
        $X=0,Y=0$ & \multicolumn{3}{c|}{$Y_0=0$} & \multicolumn{3}{c|}{$Y_0=1$} & \multicolumn{3}{c}{$Y_0=2$}\\
        \hline
          & $Y_1=0$ & $Y_1=1$ & $Y_1=2$ & $Y_1=0$ & $Y_1=1$ & $Y_1=2$ & $Y_1=0$ & $Y_1=1$ & $Y_1=2$ \\
         \hline
        $Y_2=0$ & 1/40 & 0 & 0 & 0 & 0 & 0 & 0 & 0 & 0  \\ 
        $Y_2=1$ & 1/40 & 1/40 & 0 & 0 & 0 & 0 & 0 & 0 & 0 \\
        $Y_2=2$ & 1/40 & 1/40 & 1/40 & 0 & 0 & 0 & 0 & 0 & 0  \\
        \hline
    \end{tabular}   
    \\
     \vspace{1mm}
    \begin{tabular}{c|ccc|ccc|ccc}
    \hline
        $X=0,Y=1$ & \multicolumn{3}{c|}{$Y_0=0$} & \multicolumn{3}{c|}{$Y_0=1$} & \multicolumn{3}{c}{$Y_0=2$}\\
        \hline
          & $Y_1=0$ & $Y_1=1$ & $Y_1=2$ & $Y_1=0$ & $Y_1=1$ & $Y_1=2$ & $Y_1=0$ & $Y_1=1$ & $Y_1=2$ \\
         \hline
        $Y_2=0$ & 0 & 0 & 0 & 0 & 0 & 0 & 0 & 0 & 0  \\ 
        $Y_2=1$ & 0 & 0 & 0 & 0 & 1/30 & 0 & 0 & 0 & 0 \\
        $Y_2=2$ & 0 & 0 & 0 & 0 & 1/30 & 1/30 & 0 & 0 & 0  \\
        \hline
    \end{tabular}
    \\
     \vspace{1mm}
    \begin{tabular}{c|ccc|ccc|ccc}
    \hline
        $X=0,Y=2$ & \multicolumn{3}{c|}{$Y_0=0$} & \multicolumn{3}{c|}{$Y_0=1$} & \multicolumn{3}{c}{$Y_0=2$}\\
        \hline
          & $Y_1=0$ & $Y_1=1$ & $Y_1=2$ & $Y_1=0$ & $Y_1=1$ & $Y_1=2$ & $Y_1=0$ & $Y_1=1$ & $Y_1=2$ \\
         \hline
        $Y_2=0$ & 0 & 0 & 0 & 0 & 0 & 0 & 0 & 0 & 0  \\ 
        $Y_2=1$ & 0 & 0 & 0 & 0 & 0 & 0 & 0 & 0 & 0 \\
        $Y_2=2$ & 0 & 0 & 0 & 0 & 0 & 0 & 0 & 0 & 1/10  \\
        \hline
    \end{tabular}
    \\
     \vspace{1mm}
    \begin{tabular}{c|ccc|ccc|ccc}
    \hline
        $X=1,Y=0$ & \multicolumn{3}{c|}{$Y_0=0$} & \multicolumn{3}{c|}{$Y_0=1$} & \multicolumn{3}{c}{$Y_0=2$}\\
        \hline
          & $Y_1=0$ & $Y_1=1$ & $Y_1=2$ & $Y_1=0$ & $Y_1=1$ & $Y_1=2$ & $Y_1=0$ & $Y_1=1$ & $Y_1=2$ \\
         \hline
        $Y_2=0$ & 1/30 & 0 & 0 & 0 & 0 & 0 & 0 & 0 & 0  \\ 
        $Y_2=1$ & 1/30 & 0 & 0 & 0 & 0 & 0 & 0 & 0 & 0 \\
        $Y_2=2$ & 1/30 & 0 & 0 & 0 & 0 & 0 & 0 & 0 & 0  \\
        \hline
    \end{tabular}
    \\
     \vspace{1mm}
    \begin{tabular}{c|ccc|ccc|ccc}
    \hline
        $X=1,Y=1$ & \multicolumn{3}{c|}{$Y_0=0$} & \multicolumn{3}{c|}{$Y_0=1$} & \multicolumn{3}{c}{$Y_0=2$}\\
        \hline
          & $Y_1=0$ & $Y_1=1$ & $Y_1=2$ & $Y_1=0$ & $Y_1=1$ & $Y_1=2$ & $Y_1=0$ & $Y_1=1$ & $Y_1=2$ \\
         \hline
        $Y_2=0$ & 0 & 0 & 0 & 0 & 0 & 0 & 0 & 0 & 0  \\ 
        $Y_2=1$ & 0 & 1/20 & 0 & 0 & 1/20 & 0 & 0 & 0 & 0 \\
        $Y_2=2$ & 0 & 1/20 & 0 & 0 & 1/20 & 0 & 0 & 0 & 0 \\
        \hline
    \end{tabular}
    \\
     \vspace{1mm}
    \begin{tabular}{c|ccc|ccc|ccc}
    \hline
        $X=1,Y=2$ & \multicolumn{3}{c|}{$Y_0=0$} & \multicolumn{3}{c|}{$Y_0=1$} & \multicolumn{3}{c}{$Y_0=2$}\\
        \hline
          & $Y_1=0$ & $Y_1=1$ & $Y_1=2$ & $Y_1=0$ & $Y_1=1$ & $Y_1=2$ & $Y_1=0$ & $Y_1=1$ & $Y_1=2$ \\
         \hline
        $Y_2=0$ & 0 & 0 & 0 & 0 & 0 & 0 & 0 & 0 & 0  \\ 
        $Y_2=1$ & 0 & 0 & 0 & 0 & 0 & 0 & 0 & 0 & 0 \\
        $Y_2=2$ & 0 & 0 & 1/30 & 0 & 0 & 1/30 & 0 & 0 & 1/30  \\
        \hline
    \end{tabular}
    \\
     \vspace{1mm}
    \begin{tabular}{c|ccc|ccc|ccc}
    \hline
        $X=2,Y=0$ & \multicolumn{3}{c|}{$Y_0=0$} & \multicolumn{3}{c|}{$Y_0=1$} & \multicolumn{3}{c}{$Y_0=2$}\\
        \hline
          & $Y_1=0$ & $Y_1=1$ & $Y_1=2$ & $Y_1=0$ & $Y_1=1$ & $Y_1=2$ & $Y_1=0$ & $Y_1=1$ & $Y_1=2$ \\
         \hline
        $Y_2=0$ & 1/20 & 0 & 0 & 0 & 0 & 0 & 0 & 0 & 0  \\ 
        $Y_2=1$ & 0 & 0 & 0 & 0 & 0 & 0 & 0 & 0 & 0 \\
        $Y_2=2$ & 0 & 0 & 0 & 0 & 0 & 0 & 0 & 0 & 0  \\
        \hline
    \end{tabular}
    \\
     \vspace{1mm}
    \begin{tabular}{c|ccc|ccc|ccc}
    \hline
        $X=2,Y=1$ & \multicolumn{3}{c|}{$Y_0=0$} & \multicolumn{3}{c|}{$Y_0=1$} & \multicolumn{3}{c}{$Y_0=2$}\\
        \hline
          & $Y_1=0$ & $Y_1=1$ & $Y_1=2$ & $Y_1=0$ & $Y_1=1$ & $Y_1=2$ & $Y_1=0$ & $Y_1=1$ & $Y_1=2$ \\
         \hline
        $Y_2=0$ & 0 & 0 & 0 & 0 & 0 & 0 & 0 & 0 & 0  \\ 
        $Y_2=1$ & 1/30 & 1/30 & 0 & 0 & 1/30 & 0 & 0 & 0 & 0 \\
        $Y_2=2$ & 0 & 0 & 0 & 0 & 0 & 0 & 0 & 0 & 0  \\
        \hline
    \end{tabular}
    \\
     \vspace{1mm}
    \begin{tabular}{c|ccc|ccc|ccc}
    \hline
        $X=2,Y=2$ & \multicolumn{3}{c|}{$Y_0=0$} & \multicolumn{3}{c|}{$Y_0=1$} & \multicolumn{3}{c}{$Y_0=2$}\\
        \hline
          & $Y_1=0$ & $Y_1=1$ & $Y_1=2$ & $Y_1=0$ & $Y_1=1$ & $Y_1=2$ & $Y_1=0$ & $Y_1=1$ & $Y_1=2$ \\
         \hline
        $Y_2=0$ & 0 & 0 & 0 & 0 & 0 & 0 & 0 & 0 & 0  \\ 
        $Y_2=1$ & 0 & 0 & 0 & 0 & 0 & 0 & 0 & 0 & 0 \\
        $Y_2=2$ & 1/60 & 1/60 & 1/60 & 0 & 1/60 & 1/60 & 0 & 0 & 1/60  \\
        \hline
    \end{tabular}
    }
\end{table}

\newpage
\begin{table}[H]
    \centering
    \caption{Settings of $\mathbb{P}(Y_1,Y_0,Y_2,X,Y)$ for comparing bounds based on \citep{ALi2024} and LP solver. When $x\ne y$, we set $\mathbb{P}(Y_0,Y_1,Y_2,X=x,Y=y)=0$.}
    \label{tab: Li-Pearl_comparison_settings}
    \small{
    \begin{tabular}{c|ccc|ccc|ccc}
    \hline
        $X=0,Y=0$ & \multicolumn{3}{c|}{$Y_0=0$} & \multicolumn{3}{c|}{$Y_0=1$} & \multicolumn{3}{c}{$Y_0=2$}\\
        \hline
          & $Y_1=0$ & $Y_1=1$ & $Y_1=2$ & $Y_1=0$ & $Y_1=1$ & $Y_1=2$ & $Y_1=0$ & $Y_1=1$ & $Y_1=2$ \\
         \hline
        $Y_2=0$ & 0.011 & 0.010 & 0.042 & 0 & 0 & 0 & 0 & 0 & 0  \\ 
        $Y_2=1$ & 0.045 & 0.059 & 0.069 & 0 & 0 & 0 & 0 & 0 & 0 \\
        $Y_2=2$ & 0.017 & 0.041 & 0.047 & 0 & 0 & 0 & 0 & 0 & 0  \\
        \hline
    \end{tabular}
    \vspace{1mm}
    
    \begin{tabular}{c|ccc|ccc|ccc}
    \hline
        $X=1,Y=1$ & \multicolumn{3}{c|}{$Y_0=0$} & \multicolumn{3}{c|}{$Y_0=1$} & \multicolumn{3}{c}{$Y_0=2$}\\
        \hline
          & $Y_1=0$ & $Y_1=1$ & $Y_1=2$ & $Y_1=0$ & $Y_1=1$ & $Y_1=2$ & $Y_1=0$ & $Y_1=1$ & $Y_1=2$ \\
         \hline
        $Y_2=0$ & 0 & 0.030 & 0 & 0 & 0.000 & 0 & 0 & 0.000 & 0  \\ 
        $Y_2=1$ & 0 & 0.112 & 0 & 0 & 0.013 & 0 & 0 & 0.033 & 0 \\
        $Y_2=2$ & 0 & 0.111 & 0 & 0 & 0.017 & 0 & 0 & 0.003 & 0  \\
        \hline
    \end{tabular}
    \vspace{1mm}
    
    \begin{tabular}{c|ccc|ccc|ccc}
    \hline
        $X=2,Y=2$ & \multicolumn{3}{c|}{$Y_0=0$} & \multicolumn{3}{c|}{$Y_0=1$} & \multicolumn{3}{c}{$Y_0=2$}\\
        \hline
          & $Y_1=0$ & $Y_1=1$ & $Y_1=2$ & $Y_1=0$ & $Y_1=1$ & $Y_1=2$ & $Y_1=0$ & $Y_1=1$ & $Y_1=2$ \\
         \hline
        $Y_2=0$ & 0 & 0 & 0 & 0 & 0 & 0 & 0 & 0 & 0  \\ 
        $Y_2=1$ & 0 & 0 & 0 & 0 & 0 & 0 & 0 & 0 & 0 \\
        $Y_2=2$ & 0.060 & 0.062 & 0.052 & 0.007 & 0.018 & 0.006 & 0.121 & 0.004 & 0.009  \\
        \hline
    \end{tabular}
    }
    
\end{table}

\newpage
\begin{table}[H]
    \centering
    \caption{Setting of simulations for Assumption (vi) in Section \ref{subsec: numerical experiment of bounding}}
    \label{tab:p(y0y1y2xy)_steprestriction}
    {\small
    \begin{tabular}{c|ccc|ccc|ccc}
    \hline
        $X=0,Y=0$ & \multicolumn{3}{c|}{$Y_0=0$} & \multicolumn{3}{c|}{$Y_0=1$} & \multicolumn{3}{c}{$Y_0=2$}\\
        \hline
          & $Y_1=0$ & $Y_1=1$ & $Y_1=2$ & $Y_1=0$ & $Y_1=1$ & $Y_1=2$ & $Y_1=0$ & $Y_1=1$ & $Y_1=2$ \\
         \hline
        $Y_2=0$ & 1/20 & 0 & 0 & 0 & 0 & 0 & 0 & 0 & 0  \\ 
        $Y_2=1$ & 1/20 & 1/20 & 0 & 0 & 0 & 0 & 0 & 0 & 0 \\
        $Y_2=2$ & 0 & 0 & 0 & 0 & 0 & 0 & 0 & 0 & 0  \\
        \hline
    \end{tabular}   
    \\
     \vspace{1mm}
    \begin{tabular}{c|ccc|ccc|ccc}
    \hline
        $X=0,Y=1$ & \multicolumn{3}{c|}{$Y_0=0$} & \multicolumn{3}{c|}{$Y_0=1$} & \multicolumn{3}{c}{$Y_0=2$}\\
        \hline
          & $Y_1=0$ & $Y_1=1$ & $Y_1=2$ & $Y_1=0$ & $Y_1=1$ & $Y_1=2$ & $Y_1=0$ & $Y_1=1$ & $Y_1=2$ \\
         \hline
        $Y_2=0$ & 0 & 0 & 0 & 0 & 0 & 0 & 0 & 0 & 0  \\ 
        $Y_2=1$ & 0 & 0 & 0 & 0 & 1/30 & 0 & 0 & 0 & 0 \\
        $Y_2=2$ & 0 & 0 & 0 & 0 & 1/30 & 1/30 & 0 & 0 & 0  \\
        \hline
    \end{tabular}
    \\
     \vspace{1mm}
    \begin{tabular}{c|ccc|ccc|ccc}
    \hline
        $X=0,Y=2$ & \multicolumn{3}{c|}{$Y_0=0$} & \multicolumn{3}{c|}{$Y_0=1$} & \multicolumn{3}{c}{$Y_0=2$}\\
        \hline
          & $Y_1=0$ & $Y_1=1$ & $Y_1=2$ & $Y_1=0$ & $Y_1=1$ & $Y_1=2$ & $Y_1=0$ & $Y_1=1$ & $Y_1=2$ \\
         \hline
        $Y_2=0$ & 0 & 0 & 0 & 0 & 0 & 0 & 0 & 0 & 0  \\ 
        $Y_2=1$ & 0 & 0 & 0 & 0 & 0 & 0 & 0 & 0 & 0 \\
        $Y_2=2$ & 0 & 0 & 0 & 0 & 0 & 0 & 0 & 0 & 1/10  \\
        \hline
    \end{tabular}
    \\
     \vspace{1mm}
    \begin{tabular}{c|ccc|ccc|ccc}
    \hline
        $X=1,Y=0$ & \multicolumn{3}{c|}{$Y_0=0$} & \multicolumn{3}{c|}{$Y_0=1$} & \multicolumn{3}{c}{$Y_0=2$}\\
        \hline
          & $Y_1=0$ & $Y_1=1$ & $Y_1=2$ & $Y_1=0$ & $Y_1=1$ & $Y_1=2$ & $Y_1=0$ & $Y_1=1$ & $Y_1=2$ \\
         \hline
        $Y_2=0$ & 1/20 & 0 & 0 & 0 & 0 & 0 & 0 & 0 & 0  \\ 
        $Y_2=1$ & 1/20 & 0 & 0 & 0 & 0 & 0 & 0 & 0 & 0 \\
        $Y_2=2$ & 0 & 0 & 0 & 0 & 0 & 0 & 0 & 0 & 0  \\
        \hline
    \end{tabular}
    \\
     \vspace{1mm}
    \begin{tabular}{c|ccc|ccc|ccc}
    \hline
        $X=1,Y=1$ & \multicolumn{3}{c|}{$Y_0=0$} & \multicolumn{3}{c|}{$Y_0=1$} & \multicolumn{3}{c}{$Y_0=2$}\\
        \hline
          & $Y_1=0$ & $Y_1=1$ & $Y_1=2$ & $Y_1=0$ & $Y_1=1$ & $Y_1=2$ & $Y_1=0$ & $Y_1=1$ & $Y_1=2$ \\
         \hline
        $Y_2=0$ & 0 & 0 & 0 & 0 & 0 & 0 & 0 & 0 & 0  \\ 
        $Y_2=1$ & 0 & 1/15 & 0 & 0 & 1/15 & 0 & 0 & 0 & 0 \\
        $Y_2=2$ & 0 & 0 & 0 & 0 & 1/15 & 0 & 0 & 0 & 0 \\
        \hline
    \end{tabular}
    \\
     \vspace{1mm}
    \begin{tabular}{c|ccc|ccc|ccc}
    \hline
        $X=1,Y=2$ & \multicolumn{3}{c|}{$Y_0=0$} & \multicolumn{3}{c|}{$Y_0=1$} & \multicolumn{3}{c}{$Y_0=2$}\\
        \hline
          & $Y_1=0$ & $Y_1=1$ & $Y_1=2$ & $Y_1=0$ & $Y_1=1$ & $Y_1=2$ & $Y_1=0$ & $Y_1=1$ & $Y_1=2$ \\
         \hline
        $Y_2=0$ & 0 & 0 & 0 & 0 & 0 & 0 & 0 & 0 & 0  \\ 
        $Y_2=1$ & 0 & 0 & 0 & 0 & 0 & 0 & 0 & 0 & 0 \\
        $Y_2=2$ & 0 & 0 & 0 & 0 & 0 & 1/20 & 0 & 0 & 1/20  \\
        \hline
    \end{tabular}
    \\
     \vspace{1mm}
    \begin{tabular}{c|ccc|ccc|ccc}
    \hline
        $X=2,Y=0$ & \multicolumn{3}{c|}{$Y_0=0$} & \multicolumn{3}{c|}{$Y_0=1$} & \multicolumn{3}{c}{$Y_0=2$}\\
        \hline
          & $Y_1=0$ & $Y_1=1$ & $Y_1=2$ & $Y_1=0$ & $Y_1=1$ & $Y_1=2$ & $Y_1=0$ & $Y_1=1$ & $Y_1=2$ \\
         \hline
        $Y_2=0$ & 1/20 & 0 & 0 & 0 & 0 & 0 & 0 & 0 & 0  \\ 
        $Y_2=1$ & 0 & 0 & 0 & 0 & 0 & 0 & 0 & 0 & 0 \\
        $Y_2=2$ & 0 & 0 & 0 & 0 & 0 & 0 & 0 & 0 & 0  \\
        \hline
    \end{tabular}
    \\
     \vspace{1mm}
    \begin{tabular}{c|ccc|ccc|ccc}
    \hline
        $X=2,Y=1$ & \multicolumn{3}{c|}{$Y_0=0$} & \multicolumn{3}{c|}{$Y_0=1$} & \multicolumn{3}{c}{$Y_0=2$}\\
        \hline
          & $Y_1=0$ & $Y_1=1$ & $Y_1=2$ & $Y_1=0$ & $Y_1=1$ & $Y_1=2$ & $Y_1=0$ & $Y_1=1$ & $Y_1=2$ \\
         \hline
        $Y_2=0$ & 0 & 0 & 0 & 0 & 0 & 0 & 0 & 0 & 0  \\ 
        $Y_2=1$ & 1/30 & 1/30 & 0 & 0 & 1/30 & 0 & 0 & 0 & 0 \\
        $Y_2=2$ & 0 & 0 & 0 & 0 & 0 & 0 & 0 & 0 & 0  \\
        \hline
    \end{tabular}
    \\
     \vspace{1mm}
    \begin{tabular}{c|ccc|ccc|ccc}
    \hline
        $X=2,Y=2$ & \multicolumn{3}{c|}{$Y_0=0$} & \multicolumn{3}{c|}{$Y_0=1$} & \multicolumn{3}{c}{$Y_0=2$}\\
        \hline
          & $Y_1=0$ & $Y_1=1$ & $Y_1=2$ & $Y_1=0$ & $Y_1=1$ & $Y_1=2$ & $Y_1=0$ & $Y_1=1$ & $Y_1=2$ \\
         \hline
        $Y_2=0$ & 0 & 0 & 0 & 0 & 0 & 0 & 0 & 0 & 0  \\ 
        $Y_2=1$ & 0 & 0 & 0 & 0 & 0 & 0 & 0 & 0 & 0 \\
        $Y_2=2$ & 0 & 0 & 0 & 0 & 1/30 & 1/30 & 0 & 0 & 1/30  \\
        \hline
    \end{tabular}
    }
\end{table}

\newpage
\begin{table}[H]
    \centering
    \caption{Running time of solving maximization and minimization problem of $\mathbb{P}(Y_0=0,Y_1=0,Y_2=1)$ for one time in milliseconds. Each entry is the mean over 100 repetitions.}
    \label{tab: runningtime_Py0y1y2}
   \begin{tabular}{c|cccc}
   \multicolumn{5}{c}{Exp.}\\
    \hline
    MAs & $N=10$ & $N=100$ & $N=1000$ & $N=10000$ \\
    \hline\hline
    (i) & 3 & 3 & 3 & 3 \\
    (ii) & 3 & 3 & 3 & 3 \\
    (iii) & 3 & 3 & 3 & 3 \\
    (iv) & 3 & 4 & 3 & 3 \\
    \hline
    \end{tabular}
    \\
    \vspace{1mm}
    \begin{tabular}{c|cccc}
    \multicolumn{5}{c}{Obs.}\\
    \hline
    MAs & $N=10$ & $N=100$ & $N=1000$ & $N=10000$ \\
    \hline\hline
    (i) & 3 & 3 & 2 & 3 \\
    (ii)  & 3 & 3 & 3 & 2 \\
    (iii) & 3 & 2 & 3 & 2 \\
    (iv) & 2 & 2 & 4 & 2 \\
    \hline
    \end{tabular}
    \\
    \vspace{1mm}
    \begin{tabular}{c|cccc}
    \multicolumn{5}{c}{Both}\\
    \hline
    MAs & $N=10$ & $N=100$ & $N=1000$ & $N=10000$ \\
    \hline\hline
    (i) & 3 & 3 & 3 & 3 \\
    (ii) & 4 & 3 & 3 & 3 \\
    (iii) & 3 & 3 & 3 & 3 \\
    (iv) & 3 & 3 & 3 & 2 \\
    \hline
    \end{tabular}
\end{table}

\begin{table}[H]
    \centering
    \caption{Running time of solving maximization and minimization problem of $\mathbb{E}[Y_1-Y_0|X=2,Y=2]$ for one time in milliseconds. Each entry is the mean over 100 repetitions.}
    \label{tab: runningtime_E_y1y0_x2y2}
   \begin{tabular}{c|cccc}
    \multicolumn{5}{c}{Exp.}\\
    \hline
    MAs & $N=10$ & $N=100$ & $N=1000$ & $N=10000$ \\
    \hline\hline
    (i) & 2 & 3 & 2 & 3 \\
    (ii) & 3 & 3 & 3 & 3 \\
    (iii) & 3 & 3 & 2 & 2 \\
    (iv) & 3 & 2 & 2 & 2 \\
    \hline
    \end{tabular}
    \\
    \vspace{1mm}
    \begin{tabular}{c|cccc}
    \multicolumn{5}{c}{Obs.}\\
    \hline
    MAs & $N=10$ & $N=100$ & $N=1000$ & $N=10000$ \\
    \hline\hline
    (i) & 2 & 2 & 2 & 2 \\
    (ii) & 2 & 2 & 2 & 2 \\
    (iii) & 2 & 2 & 2 & 2 \\
    (iv) & 2 & 2 & 2 & 3 \\
    \hline
    \end{tabular}
    \\
    \vspace{1mm}
    \begin{tabular}{c|cccc}
    \multicolumn{5}{c}{Both}\\
    \hline
    MAs & $N=10$ & $N=100$ & $N=1000$ & $N=10000$ \\
    \hline\hline
    (i) & 2 & 3 & 3 & 3 \\
    (ii) & 3 & 3 & 3 & 2 \\
    (iii) & 3 & 3 & 3 & 3 \\
    (iv) & 3 & 3 & 3 & 3 \\
    \hline
    \end{tabular}
\end{table}

\newpage
\begin{table}[H]
    \centering
    \caption{Results of the estimates of the bounds of $\mathbb{P}(Y_0=0,Y_1=0,Y_2=1)$.
    We report the means and 95\% confidence intervals (CIs). The true value of $\mathbb{P}(Y_0=0,Y_1=0,Y_2=1)$ is 0.092.
    }
    \label{tab:results}
    \scalebox{0.9}{
    \begin{tabular}{cc|cccc}
    \hline
    MAs & Bounds (Exp.) & $N=10$ & $N=100$ & $N=1000$ & $N=10000$ \\
    \hline
    \hline
    (i)  & LB & 0.000 [0.000,0.000] & 0.000 [0.000,0.000] & 0.000 [0.000,0.000] & 0.000 [0.000,0.000] \\
    (i)  & UB & 0.224 [0.100,0.452] & 0.265 [0.200,0.340] & 0.274 [0.245,0.304] & 0.275 [0.268,0.284] \\
    \hline
    (ii) & LB & 0.002 [0.000,0.000] & 0.000 [0.000,0.000] & 0.000 [0.000,0.000] & 0.000 [0.000,0.000] \\
    (ii) & UB & 0.224 [0.100,0.452] & 0.265 [0.200,0.340] & 0.274 [0.245,0.304] & 0.275 [0.268,0.284] \\
    \hline
    (iii) & LB & 0.000 [0.000,0.000] & 0.000 [0.000,0.000] & 0.000 [0.000,0.000] & 0.000 [0.000,0.000] \\
    (iii) & UB & 0.163 [0.000,0.400] & 0.167 [0.105,0.235] & 0.166 [0.144,0.186] & 0.167 [0.161,0.175] \\
    \hline
    (iv) & LB & 0.027 [0.000,0.152] & 0.000 [0.000,0.000] & 0.000 [0.000,0.000] & 0.000 [0.000,0.000] \\
    (iv) & UB & 0.163 [0.000,0.400] & 0.167 [0.105,0.235] & 0.166 [0.144,0.186] & 0.167 [0.161,0.175] \\
    \hline
    \end{tabular}
    }
    \\
     \vspace{1mm}
    \scalebox{0.9}{
    \begin{tabular}{cc|cccc}
    \hline
    MAs & Bounds (Obs.) & $N=10$ & $N=100$ & $N=1000$ & $N=10000$ \\
    \hline
    \hline
    (i)  & LB & 0.000 [0.000,0.000] & 0.000 [0.000,0.000] & 0.000 [0.000,0.000] & 0.000 [0.000,0.000] \\
    (i)  & UB & 0.350 [0.100,0.700] & 0.350 [0.265,0.435] & 0.349 [0.318,0.383] & 0.350 [0.339,0.358] \\
    \hline
    (ii) & LB & 0.000 [0.000,0.000] & 0.000 [0.000,0.000] & 0.000 [0.000,0.000] & 0.000 [0.000,0.000] \\
    (ii) & UB & 0.350 [0.100,0.700] & 0.350 [0.265,0.435] & 0.349 [0.318,0.383] & 0.350 [0.339,0.358] \\
    \hline
    (iii) & LB & 0.000 [0.000,0.000] & 0.000 [0.000,0.000] & 0.000 [0.000,0.000] & 0.000 [0.000,0.000] \\
    (iii) & UB & 0.350 [0.100,0.700] & 0.350 [0.265,0.435] & 0.349 [0.318,0.383] & 0.350 [0.339,0.358] \\
    \hline
    (iv) & LB & 0.000 [0.000,0.000] & 0.000 [0.000,0.000] & 0.000 [0.000,0.000] & 0.000 [0.000,0.000] \\
    (iv) & UB & 0.350 [0.100,0.700] & 0.350 [0.265,0.435] & 0.349 [0.318,0.383] & 0.350 [0.339,0.358] \\
    \hline
    \end{tabular}
    }
    \\
     \vspace{1mm}
    \scalebox{0.9}{
    \begin{tabular}{cc|cccc}
    \hline
    MAs & Bounds (Both) & $N=10$ & $N=100$ & $N=1000$ & $N=10000$ \\
    \hline
    \hline
    (i) Li--Pearl & LB & 0.000 [0.000,0.000] & 0.000 [0.000,0.000] & 0.000 [0.000,0.000] & 0.000 [0.000,0.000] \\
    (i) Li--Pearl & UB & 0.205 [0.000,0.400] & 0.263 [0.200,0.335] & 0.274 [0.245,0.304] & 0.275 [0.268,0.284] \\
    \hline
    (i)  & LB & 0.000 [0.000,0.000] & 0.000 [0.000,0.000] & 0.000 [0.000,0.000] & 0.000 [0.000,0.000] \\
    (i)  & UB & 0.205 [0.000,0.400] & 0.262 [0.200,0.335] & 0.274 [0.245,0.304] & 0.275 [0.268,0.284] \\
    \hline
    (ii) & LB & 0.015 [0.000,0.100] & 0.000 [0.000,0.005] & 0.000 [0.000,0.000] & 0.000 [0.000,0.000] \\
    (ii) & UB & 0.204 [0.000,0.400] & 0.262 [0.200,0.335] & 0.274 [0.245,0.304] & 0.275 [0.268,0.284] \\
    \hline
    (iii) & LB & 0.003 [0.000,0.000] & 0.000 [0.000,0.000] & 0.000 [0.000,0.000] & 0.000 [0.000,0.000] \\
    (iii) & UB & 0.156 [0.000,0.400] & 0.166 [0.105,0.235] & 0.166 [0.144,0.186] & 0.167 [0.161,0.175] \\
    \hline
    (iv) & LB & 0.047 [0.000,0.200] & 0.002 [0.000,0.030] & 0.000 [0.000,0.000] & 0.000 [0.000,0.000] \\
    (iv) & UB & 0.151 [0.000,0.400] & 0.166 [0.105,0.235] & 0.166 [0.144,0.186] & 0.167 [0.161,0.175] \\
    \hline
    \end{tabular}
    }
\end{table}

\begin{table}[H]
    \centering
    \caption{Results of the estimates of the bounds of $\mathbb{E}[Y_1-Y_0|X=2,Y=2]$.
    We report the means and 95\% confidence intervals (CIs). The true value of $\mathbb{E}[Y_1-Y_0|X=2,Y=2]$ is 0.667.}
    \label{tab:results_condATE}
    \scalebox{0.9}{
    \begin{tabular}{cc|cccc}
    \hline
    MAs & Bounds (Obs.) & $N=10$ & $N=100$ & $N=1000$ & $N=10000$ \\
    \hline
    \hline
    (i)  & LB & -2.000 [-2.000,-2.000] & -2.000 [-2.000,-2.000] & -2.000 [-2.000,-2.000] & -2.000 [-2.000,-2.000] \\
    (i)  & UB & 2.000 [2.000,2.000] & 2.000 [2.000,2.000] & 2.000 [2.000,2.000] & 2.000 [2.000,2.000] \\
    \hline
    (ii) & LB & 0.000 [0.000,0.000] & 0.000 [0.000,0.000] & 0.000 [0.000,0.000] & 0.000 [0.000,0.000] \\
    (ii) & UB & 2.000 [2.000,2.000] & 2.000 [2.000,2.000] & 2.000 [2.000,2.000] & 2.000 [2.000,2.000] \\
    \hline
    (iii) & LB & -2.000 [-2.000,-2.000] & -2.000 [-2.000,-2.000] & -2.000 [-2.000,-2.000] & -2.000 [-2.000,-2.000] \\
    (iii) & UB & 2.000 [2.000,2.000] & 2.000 [2.000,2.000] & 2.000 [2.000,2.000] & 2.000 [2.000,2.000] \\
    \hline
    (iv) & LB & 0.000 [0.000,0.000] & 0.000 [0.000,0.000] & 0.000 [0.000,0.000] & 0.000 [0.000,0.000] \\
    (iv) & UB & 2.000 [2.000,2.000] & 2.000 [2.000,2.000] & 2.000 [2.000,2.000] & 2.000 [2.000,2.000] \\
    \hline
    \end{tabular}
    }
    \\
     \vspace{1mm}
    \scalebox{0.9}{
    \begin{tabular}{cc|cccc}
    \hline
    MAs & Bounds (Both) & $N=10$ & $N=100$ & $N=1000$ & $N=10000$ \\
    \hline
    \hline
    (i)  & LB & -1.073 [-2.000,0.417] & -1.515 [-2.000,-0.956] & -1.508 [-1.812,-1.251] & -1.501 [-1.582,-1.425] \\
    (i)  & UB & 1.879 [1.000,2.000] & 1.996 [2.000,2.000] & 2.000 [2.000,2.000] & 2.000 [2.000,2.000] \\
    \hline
    (ii) & LB & 0.030 [0.000,0.417] & 0.000 [0.000,0.000] & 0.000 [0.000,0.000] & 0.000 [0.000,0.000] \\
    (ii) & UB & 1.083 [0.000,2.000] & 1.822 [1.116,2.000] & 1.963 [1.702,2.000] & 1.999 [1.987,2.000] \\
    \hline
    (iii) & LB & -0.743 [-2.000,1.000] & -1.423 [-2.000,-0.703] & -1.496 [-1.806,-1.215] & -1.501 [-1.582,-1.425] \\
    (iii) & UB & 1.879 [1.000,2.000] & 1.996 [2.000,2.000] & 2.000 [2.000,2.000] & 2.000 [2.000,2.000] \\
    \hline
    (iv) & LB & 0.128 [0.000,1.000] & 0.000 [0.000,0.000] & 0.000 [0.000,0.000] & 0.000 [0.000,0.000] \\
    (iv) & UB & 0.976 [0.000,2.000] & 1.709 [1.032,2.000] & 1.961 [1.702,2.000] & 1.999 [1.987,2.000] \\
    \hline
    \end{tabular}
    }
    \\
     \vspace{1mm}
\end{table}

\newpage
\begin{table}[H]
    \centering
    \caption{Results of the estimates of the bounds of $\mathbb{E}[(Y_1-Y_0)^2]$.
    We report the means and 95\% confidence intervals (CIs). The true value of $\mathbb{E}[(Y_1-Y_0)^2]$ is 0.583.}
    \label{tab:results_2ndMoment}
    \scalebox{0.9}{
    \begin{tabular}{cc|cccc}
    \hline
    MAs & Bounds (Exp.) & $N=10$ & $N=100$ & $N=1000$ & $N=10000$ \\
    \hline
    \hline
    (i)  & LB & 0.417 [0.047,0.952] & 0.441 [0.325,0.570] & 0.432 [0.393,0.472] & 0.432 [0.422,0.443] \\
    (i)  & UB & 2.003 [0.900,3.052] & 2.176 [1.803,2.485] & 2.189 [2.077,2.313] & 2.197 [2.151,2.240] \\
    \hline
    (ii) & LB & 0.417 [0.047,0.952] & 0.441 [0.325,0.570] & 0.432 [0.393,0.472] & 0.432 [0.422,0.443] \\
    (ii) & UB & 0.629 [0.048,1.300] & 0.764 [0.539,1.016] & 0.745 [0.659,0.816] & 0.747 [0.725,0.771] \\
    \hline
    (iii) & LB & 0.417 [0.047,0.952] & 0.441 [0.325,0.570] & 0.432 [0.393,0.472] & 0.432 [0.422,0.443] \\
    (iii) & UB & 2.003 [0.900,3.052] & 2.176 [1.803,2.485] & 2.189 [2.077,2.313] & 2.197 [2.151,2.240] \\
    \hline
    (iv) & LB & 0.417 [0.047,0.952] & 0.441 [0.325,0.570] & 0.432 [0.393,0.472] & 0.432 [0.422,0.443] \\
    (iv) & UB & 0.629 [0.048,1.300] & 0.764 [0.539,1.016] & 0.745 [0.659,0.816] & 0.747 [0.725,0.771] \\
    \hline
    \end{tabular}
    }
    \\
    \vspace{1mm}
    \scalebox{0.9}{
    \begin{tabular}{cc|cccc}
    \hline
    MAs & Bounds (Obs.) & $N=10$ & $N=100$ & $N=1000$ & $N=10000$ \\
    \hline
    \hline
    (i)  & LB & 0.000 [0.000,0.000] & 0.000 [0.000,0.000] & 0.000 [0.000,0.000] & 0.000 [0.000,0.000] \\
    (i)  & UB & 3.136 [2.200,3.700] & 3.094 [2.814,3.370] & 3.100 [2.998,3.175] & 3.102 [3.078,3.125] \\
    \hline
    (ii) & LB & 0.000 [0.000,0.000] & 0.000 [0.000,0.000] & 0.000 [0.000,0.000] & 0.000 [0.000,0.000] \\
    (ii) & UB & 2.284 [1.300,3.100] & 2.278 [1.965,2.625] & 2.299 [2.193,2.414] & 2.301 [2.274,2.330] \\
    \hline
    (iii) & LB & 0.000 [0.000,0.000] & 0.000 [0.000,0.000] & 0.000 [0.000,0.000] & 0.000 [0.000,0.000] \\
    (iii) & UB & 3.136 [2.200,3.700] & 3.094 [2.814,3.370] & 3.100 [2.998,3.175] & 3.102 [3.078,3.125] \\
    \hline
    (iv) & LB & 0.000 [0.000,0.000] & 0.000 [0.000,0.000] & 0.000 [0.000,0.000] & 0.000 [0.000,0.000] \\
    (iv) & UB & 1.848 [0.947,2.857] & 1.806 [1.460,2.201] & 1.800 [1.718,1.904] & 1.801 [1.769,1.831] \\
    \hline
    \end{tabular}
    }
    \\
     \vspace{1mm}
    \scalebox{0.9}{
    \begin{tabular}{cc|cccc}
    \hline
    MAs & Bounds (Both) & $N=10$ & $N=100$ & $N=1000$ & $N=10000$ \\
    \hline
    \hline
    (i)  & LB & 0.435 [0.047,1.100] & 0.442 [0.325,0.570] & 0.432 [0.393,0.472] & 0.432 [0.422,0.443] \\
    (i)  & UB & 1.913 [0.848,2.952] & 2.170 [1.803,2.485] & 2.189 [2.077,2.313] & 2.197 [2.151,2.240] \\
    \hline
    (ii) & LB & 0.435 [0.048,1.100] & 0.442 [0.325,0.570] & 0.432 [0.393,0.472] & 0.432 [0.422,0.443] \\
    (ii) & UB & 0.617 [0.047,1.300] & 0.762 [0.539,1.016] & 0.745 [0.659,0.816] & 0.747 [0.725,0.771] \\
    \hline
    (iii) & LB & 0.447 [0.047,1.100] & 0.445 [0.325,0.590] & 0.432 [0.393,0.472] & 0.432 [0.422,0.443] \\
    (iii) & UB & 1.811 [0.848,2.700] & 2.125 [1.784,2.425] & 2.186 [2.073,2.307] & 2.197 [2.151,2.240] \\
    \hline
    (iv) & LB & 0.447 [0.047,1.100] & 0.445 [0.325,0.590] & 0.432 [0.393,0.472] & 0.432 [0.422,0.443] \\
    (iv) & UB & 0.613 [0.047,1.300] & 0.762 [0.519,1.016] & 0.745 [0.659,0.816] & 0.747 [0.725,0.771] \\
    \hline
    \end{tabular}
    }
\end{table}

\begin{table}[H]
    \centering
    \caption{The bound of $\mathbb{P}(Y_0=0,Y_1=1,Y_2=2)$ based on Theorem 11 in \citep{ALi2024} and LP solver under the setting in Table \ref{tab: Li-Pearl_comparison_settings}. The true value of $\mathbb{P}(Y_0=0,Y_1=1,Y_2=2)$ is 0.214.
    }
    \label{tab: Li-Pearl_is_not_tight}
    \scalebox{0.9}{
    \begin{tabular}{c|cccc}
    \hline
    Bounds (Both) & $N=10$ & $N=100$ & $N=1000$ & $N=10000$  \\
    \hline
    \hline
    LB (\citep{ALi2024}) & 0.056 [0.000,0.300] & 0.003 [0.000,0.025] & 0.000 [0.000,0.000] & 0.000 [0.000,0.000] \\
    UB (\citep{ALi2024}) & 0.426 [0.100,0.700] & 0.501 [0.410,0.570] & 0.517 [0.485,0.549] & 0.514 [0.505,0.526] \\
    \hline
    LB & 0.056 [0.000,0.300] & 0.003 [0.000,0.025] & 0.000 [0.000,0.000] & 0.000 [0.000,0.000] \\
    UB & 0.389 [0.100,0.600] & 0.414 [0.305,0.508] & 0.426 [0.387,0.451] & 0.426 [0.411,0.435] \\
    \hline
    \end{tabular}
    }
\end{table}

\newpage
\begin{table}[H]
    \centering
    \caption{Results of the estimates of the bounds of $\mathbb{P}(Y_0=0,Y_1=0,Y_2=1)$ under  assumption $L \leq \mathbb{P}(Y_0\leq Y_1\leq Y_2)$ and $N=1000$.
    We report the means and 95\% confidence intervals (CIs).
    The true value of $\mathbb{P}(Y_0=0,Y_1=0,Y_2=1)$ is 0.092.
    }
    \label{tab:results_theta_exp}
    \scalebox{0.9}{
    \begin{tabular}{cc|cc}
    \hline
   Data & $L$ & LB & UB \\
    \hline\hline
    Obs & (All) & 0.000 [0.000, 0.000] & 0.351 [0.326, 0.374] \\
    \hline
    Exp. \& Both & 0.00 & 0.000 [0.000, 0.000] & 0.274 [0.250, 0.301] \\
    Exp. \& Both & 0.70 & 0.000 [0.000, 0.000] & 0.274 [0.250, 0.301] \\
    Exp. \& Both & 0.75 & 0.000 [0.000, 0.000] & 0.274 [0.250, 0.301] \\
    Exp. \& Both & 0.80 & 0.000 [0.000, 0.000] & 0.274 [0.250, 0.301] \\
    Exp. \& Both & 0.85 & 0.000 [0.000, 0.000] & 0.274 [0.250, 0.301] \\
    Exp. \& Both & 0.86 & 0.000 [0.000, 0.000] & 0.274 [0.250, 0.301] \\
    Exp. \& Both & 0.87 & 0.000 [0.000, 0.000] & 0.274 [0.250, 0.300] \\
    Exp. \& Both & 0.88 & 0.000 [0.000, 0.000] & 0.273 [0.250, 0.294] \\
    Exp. \& Both & 0.89 & 0.000 [0.000, 0.000] & 0.270 [0.250, 0.291] \\
    Exp. \& Both & 0.90 & 0.000 [0.000, 0.000] & 0.264 [0.245, 0.286] \\
    Exp. \& Both & 0.91 & 0.000 [0.000, 0.000] & 0.255 [0.235, 0.276] \\
    Exp. \& Both & 0.92 & 0.000 [0.000, 0.000] & 0.245 [0.225, 0.266] \\
    Exp. \& Both & 0.93 & 0.000 [0.000, 0.000] & 0.235 [0.215, 0.256] \\
    Exp. \& Both & 0.94 & 0.000 [0.000, 0.000] & 0.225 [0.205, 0.246] \\
    Exp. \& Both & 0.95 & 0.000 [0.000, 0.000] & 0.215 [0.195, 0.236] \\
    Exp. \& Both & 0.96 & 0.000 [0.000, 0.000] & 0.205 [0.185, 0.226] \\
    Exp. \& Both & 0.97 & 0.000 [0.000, 0.000] & 0.195 [0.175, 0.216] \\
    Exp. \& Both & 0.98 & 0.000 [0.000, 0.000] & 0.185 [0.165, 0.206] \\
    Exp. \& Both & 0.99 & 0.000 [0.000, 0.000] & 0.175 [0.155, 0.196] \\
    Exp. \& Both & 1.00 & 0.000 [0.000, 0.000] & 0.165 [0.145, 0.186] \\
    \hline
    \end{tabular}
    }
\end{table}

\begin{figure}[H]
\centering
\input{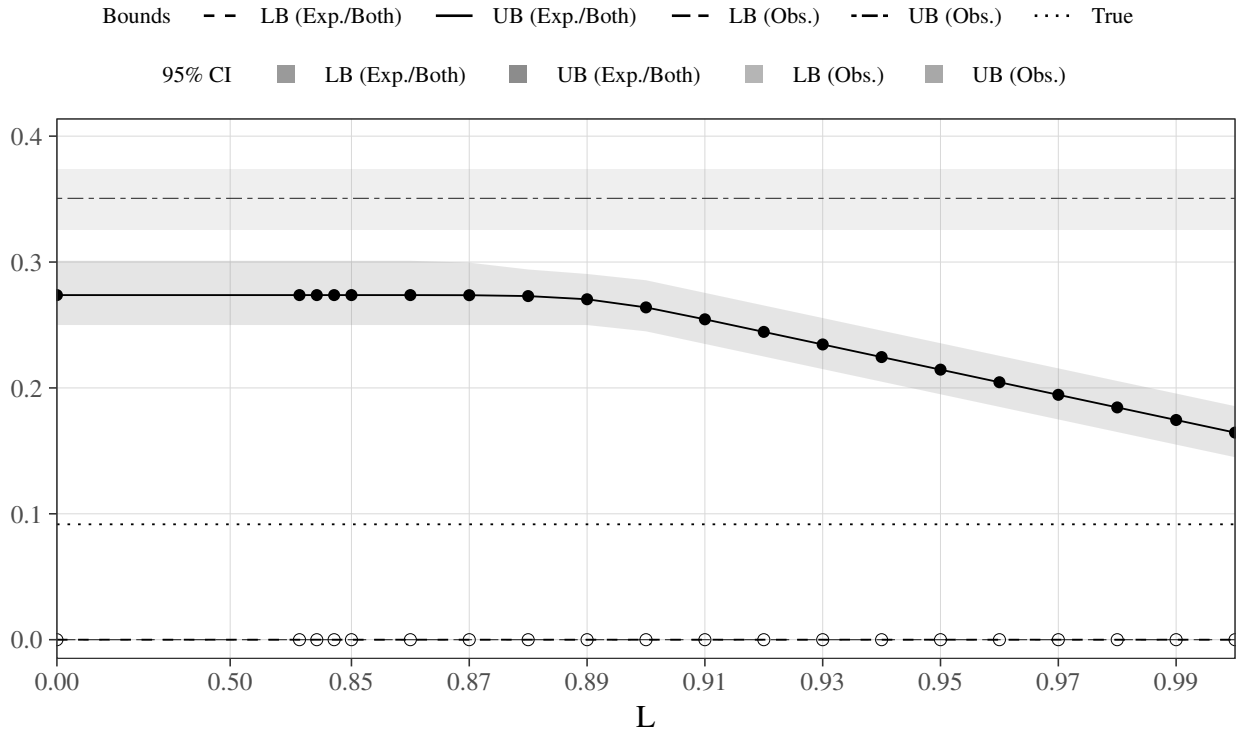}
\caption{Results of the estimates of the bounds of $\mathbb{P}(Y_0=0,Y_1=0,Y_2=1)$ as $L$ varies. The x-axis shows the value of $L$, and the y-axis shows the mean estimates of the bounds of $\mathbb{P}(Y_0=0,Y_1=0,Y_2=1)$.}
\label{fig:placeholder}
\end{figure}

\newpage
\begin{table}[H]
\centering
\caption{Results of the estimates of the bounds of $\mathbb{E}[Y_1-Y_0|X=2,Y=2]$ under  assumption $L \leq \mathbb{P}(Y_0\leq Y_1\leq Y_2)$ and $N=1000$.
We report the means and 95\% confidence intervals (CIs).
The true value of $\mathbb{E}[Y_1-Y_0|X=2,Y=2]$ is 0.667.
}
\label{tab:results_theta_exp2}
\scalebox{0.9}{
\begin{tabular}{cc|cc}
\hline
Data & $L$ & LB & UB \\
\hline\hline
Obs & (All) & -2.000 [-2.000, -2.000] & 2.000 [2.000, 2.000] \\
\hline
Both & 0.00 & -1.509 [-1.850, -1.267] & 2.000 [2.000, 2.000] \\
Both & 0.70 & -1.509 [-1.850, -1.267] & 2.000 [2.000, 2.000] \\
Both & 0.75 & -1.509 [-1.850, -1.267] & 2.000 [2.000, 2.000] \\
Both & 0.80 & -1.509 [-1.850, -1.267] & 2.000 [2.000, 2.000] \\
Both & 0.85 & -1.509 [-1.850, -1.267] & 2.000 [2.000, 2.000] \\
Both & 0.86 & -1.509 [-1.850, -1.267] & 2.000 [2.000, 2.000] \\
Both & 0.87 & -1.509 [-1.850, -1.267] & 2.000 [2.000, 2.000] \\
Both & 0.88 & -1.508 [-1.850, -1.255] & 2.000 [2.000, 2.000] \\
Both & 0.89 & -1.503 [-1.850, -1.216] & 2.000 [2.000, 2.000] \\
Both & 0.90 & -1.481 [-1.821, -1.174] & 2.000 [2.000, 2.000] \\
Both & 0.91 & -1.414 [-1.794, -1.112] & 2.000 [2.000, 2.000] \\
Both & 0.92 & -1.318 [-1.696, -1.020] & 2.000 [2.000, 2.000] \\
Both & 0.93 & -1.210 [-1.565, -0.922] & 2.000 [2.000, 2.000] \\
Both & 0.94 & -1.095 [-1.356, -0.825] & 2.000 [2.000, 2.000] \\
Both & 0.95 & -0.956 [-1.160, -0.726] & 2.000 [2.000, 2.000] \\
Both & 0.96 & -0.792 [-0.958, -0.631] & 1.999 [2.000, 2.000] \\
Both & 0.97 & -0.605 [-0.728, -0.515] & 1.997 [1.977, 2.000] \\
Both & 0.98 & -0.405 [-0.485, -0.351] & 1.995 [1.930, 2.000] \\
Both & 0.99 & -0.203 [-0.243, -0.175] & 1.988 [1.836, 2.000] \\
Both & 1.00 & 0.000 [0.000, 0.000] & 1.971 [1.743, 2.000] \\
\hline
\end{tabular}
    }
\end{table}

\newpage
\begin{figure}[H]
    \centering
    \input{E_Y1_Y0_X_2_Y_2__Nexp1000_Nobs1000_M100_brokenx_new2}
    \caption{Results of the estimates of the bounds of $\mathbb{E}[Y_1-Y_0|X=2,Y=2]$ as $L$ varies. The x-axis shows the value of $L$, and the y-axis shows the mean estimates of the bounds of $\mathbb{E}[Y_1-Y_0|X=2,Y=2]$.}
    \label{fig:placeholder2}
\end{figure}

\newpage
\begin{table}[H]
\centering
\caption{Results of the estimates of the bounds of $\mathbb{E}[(Y_1-Y_0)^2]$ under  assumption $L \leq \mathbb{P}(Y_0\leq Y_1\leq Y_2)$ and $N=1000$.
We report the means and 95\% confidence intervals (CIs).
The true value of $\mathbb{E}[(Y_1-Y_0)^2]$ is 0.583.
}
\label{tab:results_theta_2ndMoment}
\scalebox{0.9}{
\begin{tabular}{cc|cc}
\hline
Data & $L$ & LB & UB \\
\hline\hline
Exp. & 0.00 & 0.435 [0.392, 0.474] & 2.196 [2.073, 2.339] \\
Exp. & 0.70 & 0.435 [0.392, 0.474] & 2.195 [2.073, 2.331] \\
Exp. & 0.75 & 0.435 [0.392, 0.474] & 2.141 [2.003, 2.272] \\
Exp. & 0.80 & 0.435 [0.392, 0.474] & 2.041 [1.903, 2.172] \\
Exp. & 0.85 & 0.435 [0.392, 0.474] & 1.921 [1.803, 2.024] \\
Exp. & 0.86 & 0.435 [0.392, 0.474] & 1.863 [1.775, 1.952] \\
Exp. & 0.87 & 0.435 [0.392, 0.474] & 1.791 [1.698, 1.876] \\
Exp. & 0.88 & 0.435 [0.392, 0.474] & 1.713 [1.621, 1.796] \\
Exp. & 0.89 & 0.435 [0.392, 0.474] & 1.633 [1.541, 1.716] \\
Exp. & 0.90 & 0.435 [0.392, 0.474] & 1.553 [1.461, 1.636] \\
Exp. & 0.91 & 0.435 [0.392, 0.474] & 1.473 [1.381, 1.556] \\
Exp. & 0.92 & 0.435 [0.392, 0.474] & 1.393 [1.301, 1.476] \\
Exp. & 0.93 & 0.435 [0.392, 0.474] & 1.313 [1.221, 1.396] \\
Exp. & 0.94 & 0.435 [0.392, 0.474] & 1.233 [1.141, 1.316] \\
Exp. & 0.95 & 0.435 [0.392, 0.474] & 1.153 [1.061, 1.236] \\
Exp. & 0.96 & 0.435 [0.392, 0.474] & 1.073 [0.981, 1.156] \\
Exp. & 0.97 & 0.435 [0.392, 0.474] & 0.993 [0.901, 1.076] \\
Exp. & 0.98 & 0.435 [0.392, 0.474] & 0.913 [0.821, 0.996] \\
Exp. & 0.99 & 0.435 [0.392, 0.474] & 0.833 [0.741, 0.916] \\
Exp. & 1.00 & 0.435 [0.392, 0.474] & 0.753 [0.661, 0.836] \\
\hline
\end{tabular}
}
\hspace{5mm}
\scalebox{0.9}{
\begin{tabular}{cc|cc}
\hline
Data & $L$ & LB & UB \\
\hline\hline
Obs. & 0.00 & 0.000 [0.000, 0.000] & 3.104 [3.026, 3.190] \\
Obs. & 0.70 & 0.000 [0.000, 0.000] & 2.945 [2.864, 3.030] \\
Obs. & 0.75 & 0.000 [0.000, 0.000] & 2.788 [2.691, 2.880] \\
Obs. & 0.80 & 0.000 [0.000, 0.000] & 2.593 [2.491, 2.695] \\
Obs. & 0.85 & 0.000 [0.000, 0.000] & 2.393 [2.291, 2.495] \\
Obs. & 0.86 & 0.000 [0.000, 0.000] & 2.353 [2.251, 2.455] \\
Obs. & 0.87 & 0.000 [0.000, 0.000] & 2.313 [2.211, 2.415] \\
Obs. & 0.88 & 0.000 [0.000, 0.000] & 2.273 [2.171, 2.375] \\
Obs. & 0.89 & 0.000 [0.000, 0.000] & 2.233 [2.131, 2.335] \\
Obs. & 0.90 & 0.000 [0.000, 0.000] & 2.193 [2.091, 2.295] \\
Obs. & 0.91 & 0.000 [0.000, 0.000] & 2.153 [2.051, 2.255] \\
Obs. & 0.92 & 0.000 [0.000, 0.000] & 2.113 [2.011, 2.215] \\
Obs. & 0.93 & 0.000 [0.000, 0.000] & 2.073 [1.971, 2.175] \\
Obs. & 0.94 & 0.000 [0.000, 0.000] & 2.033 [1.931, 2.135] \\
Obs. & 0.95 & 0.000 [0.000, 0.000] & 1.993 [1.891, 2.095] \\
Obs. & 0.96 & 0.000 [0.000, 0.000] & 1.953 [1.851, 2.055] \\
Obs. & 0.97 & 0.000 [0.000, 0.000] & 1.913 [1.811, 2.015] \\
Obs. & 0.98 & 0.000 [0.000, 0.000] & 1.873 [1.771, 1.975] \\
Obs. & 0.99 & 0.000 [0.000, 0.000] & 1.833 [1.731, 1.935] \\
Obs. & 1.00 & 0.000 [0.000, 0.000] & 1.793 [1.691, 1.895] \\
\hline
\end{tabular}
}
\\
\vspace{5mm}
\scalebox{0.9}{
\begin{tabular}{cc|cc}
\hline
Data & $L$ & LB & UB \\
\hline\hline
Both & 0.00 & 0.435 [0.392, 0.474] & 2.196 [2.073, 2.339] \\
Both & 0.70 & 0.435 [0.392, 0.474] & 2.195 [2.073, 2.331] \\
Both & 0.75 & 0.435 [0.392, 0.474] & 2.139 [2.003, 2.265] \\
Both & 0.80 & 0.435 [0.392, 0.474] & 2.039 [1.903, 2.165] \\
Both & 0.85 & 0.435 [0.392, 0.474] & 1.918 [1.803, 2.003] \\
Both & 0.86 & 0.435 [0.392, 0.474] & 1.863 [1.775, 1.948] \\
Both & 0.87 & 0.435 [0.392, 0.474] & 1.791 [1.698, 1.876] \\
Both & 0.88 & 0.435 [0.392, 0.474] & 1.713 [1.621, 1.796] \\
Both & 0.89 & 0.435 [0.392, 0.474] & 1.633 [1.541, 1.716] \\
Both & 0.90 & 0.435 [0.392, 0.474] & 1.553 [1.461, 1.636] \\
Both & 0.91 & 0.435 [0.392, 0.474] & 1.473 [1.381, 1.556] \\
Both & 0.92 & 0.435 [0.392, 0.474] & 1.393 [1.301, 1.476] \\
Both & 0.93 & 0.435 [0.392, 0.474] & 1.313 [1.221, 1.396] \\
Both & 0.94 & 0.435 [0.392, 0.474] & 1.233 [1.141, 1.316] \\
Both & 0.95 & 0.435 [0.392, 0.474] & 1.153 [1.061, 1.236] \\
Both & 0.96 & 0.435 [0.392, 0.474] & 1.073 [0.981, 1.156] \\
Both & 0.97 & 0.435 [0.392, 0.474] & 0.993 [0.901, 1.076] \\
Both & 0.98 & 0.435 [0.392, 0.474] & 0.913 [0.821, 0.996] \\
Both & 0.99 & 0.435 [0.392, 0.474] & 0.833 [0.741, 0.916] \\
Both & 1.00 & 0.435 [0.392, 0.474] & 0.753 [0.661, 0.836] \\
\hline
\end{tabular}
    }
\end{table}

\newpage
\begin{figure}[H]
    \centering
    \input{E_Y1_Y0_2__Nexp1000_Nobs1000_M100_brokenx_ALLcases}
    \caption{Results of the estimates of the bounds of $\mathbb{E}[(Y_1-Y_0)^2]$ as $L$ varies. The x-axis shows the value of $L$, and the y-axis shows the mean estimates of the bounds of $\mathbb{E}[(Y_1-Y_0)^2]$.}
    \label{fig:EY1Y0^2}
\end{figure}

\newpage
\begin{table}[H]
    \centering
    \caption{Results of the estimates of bounds of $\mathbb{P}(Y_0=0,Y_1=0,Y_2=1)$ under Assumption (vi) and $N=1000$ with experimental data / with both experimental and observational data.
    The true value of $\mathbb{P}(Y_0=0,Y_1=0,Y_2=1)$ is 0.133.
    }
    \label{tab:L1L2L3_prob}
    \scalebox{0.85}{
\begin{tabular}{cc|c|c|c|c|c}
    \multicolumn{7}{c}{$L_3 = 0.900$} \\
    \hline
    & & $L_2 = 0.900$ & $L_2 = 0.925$ & $L_2 = 0.950$ & $L_2 = 0.975$ & $L_2 = 1.000$ \\
    \hline
    $L_1 = 0.900$ & LB & 0.000 [0.000, 0.000] & 0.000 [0.000, 0.000] & 0.001 [0.000, 0.008] & 0.009 [0.000, 0.033] & 0.033 [0.012, 0.058] \\
    $L_1 = 0.900$ & UB & 0.233 [0.212, 0.258] & 0.208 [0.187, 0.233] & 0.183 [0.162, 0.208] & 0.158 [0.137, 0.183] & 0.133 [0.112, 0.158] \\
        \hline
    $L_1 = 0.925$ & LB & 0.000 [0.000, 0.000] & 0.001 [0.000, 0.008] & 0.009 [0.000, 0.033] & 0.033 [0.012, 0.058] & 0.058 [0.037, 0.083] \\
    $L_1 = 0.925$ & UB & 0.233 [0.212, 0.258] & 0.208 [0.187, 0.233] & 0.183 [0.162, 0.208] & 0.158 [0.137, 0.183] & 0.133 [0.112, 0.158] \\
        \hline
    $L_1 = 0.950$ & LB & 0.001 [0.000, 0.008] & 0.009 [0.000, 0.033] & 0.033 [0.012, 0.058] & 0.058 [0.037, 0.083] & 0.083 [0.062, 0.108] \\
    $L_1 = 0.950$ & UB & 0.233 [0.212, 0.258] & 0.208 [0.187, 0.233] & 0.183 [0.162, 0.208] & 0.158 [0.137, 0.183] & 0.133 [0.112, 0.158] \\
        \hline
    $L_1 = 0.975$ & LB & 0.009 [0.000, 0.033] & 0.033 [0.012, 0.058] & 0.058 [0.037, 0.083] & 0.083 [0.062, 0.108] & 0.108 [0.087, 0.133] \\
    $L_1 = 0.975$ & UB & 0.233 [0.212, 0.258] & 0.208 [0.187, 0.233] & 0.183 [0.162, 0.208] & 0.158 [0.137, 0.183] & 0.133 [0.112, 0.158] \\
        \hline
    $L_1 = 1.000$ & LB & 0.033 [0.012, 0.058] & 0.058 [0.037, 0.083] & 0.083 [0.062, 0.108] & 0.108 [0.087, 0.133] & 0.133 [0.112, 0.158] \\
    $L_1 = 1.000$ & UB & 0.233 [0.212, 0.258] & 0.208 [0.187, 0.233] & 0.183 [0.162, 0.208] & 0.158 [0.137, 0.183] & 0.133 [0.112, 0.158] \\
        \hline
    \end{tabular}

    }
    \\
    \vspace{1mm}
    \scalebox{0.85}{
    \begin{tabular}{cc|c|c|c|c|c}
    \multicolumn{7}{c}{$L_3 = 0.925$} \\
    \hline
    & & $L_2 = 0.900$ & $L_2 = 0.925$ & $L_2 = 0.950$ & $L_2 = 0.975$ & $L_2 = 1.000$ \\
    \hline
    $L_1 = 0.900$ & LB & 0.000 [0.000, 0.000] & 0.000 [0.000, 0.000] & 0.001 [0.000, 0.008] & 0.009 [0.000, 0.033] & 0.033 [0.012, 0.058] \\
    $L_1 = 0.900$ & UB & 0.233 [0.212, 0.258] & 0.208 [0.187, 0.233] & 0.183 [0.162, 0.208] & 0.158 [0.137, 0.183] & 0.133 [0.112, 0.158] \\
        \hline
    $L_1 = 0.925$ & LB & 0.001 [0.000, 0.008] & 0.001 [0.000, 0.008] & 0.009 [0.000, 0.033] & 0.033 [0.012, 0.058] & 0.058 [0.037, 0.083] \\
    $L_1 = 0.925$ & UB & 0.233 [0.212, 0.258] & 0.208 [0.187, 0.233] & 0.183 [0.162, 0.208] & 0.158 [0.137, 0.183] & 0.133 [0.112, 0.158] \\
        \hline
    $L_1 = 0.950$ & LB & 0.009 [0.000, 0.033] & 0.009 [0.000, 0.033] & 0.033 [0.012, 0.058] & 0.058 [0.037, 0.083] & 0.083 [0.062, 0.108] \\
    $L_1 = 0.950$ & UB & 0.233 [0.212, 0.258] & 0.208 [0.187, 0.233] & 0.183 [0.162, 0.208] & 0.158 [0.137, 0.183] & 0.133 [0.112, 0.158] \\
        \hline
    $L_1 = 0.975$ & LB & 0.033 [0.012, 0.058] & 0.033 [0.012, 0.058] & 0.058 [0.037, 0.083] & 0.083 [0.062, 0.108] & 0.108 [0.087, 0.133] \\
    $L_1 = 0.975$ & UB & 0.233 [0.212, 0.258] & 0.208 [0.187, 0.233] & 0.183 [0.162, 0.208] & 0.158 [0.137, 0.183] & 0.133 [0.112, 0.158] \\
        \hline
    $L_1 = 1.000$ & LB & 0.058 [0.037, 0.083] & 0.058 [0.037, 0.083] & 0.083 [0.062, 0.108] & 0.108 [0.087, 0.133] & 0.133 [0.112, 0.158] \\
    $L_1 = 1.000$ & UB & 0.233 [0.212, 0.258] & 0.208 [0.187, 0.233] & 0.183 [0.162, 0.208] & 0.158 [0.137, 0.183] & 0.133 [0.112, 0.158] \\
        \hline
    \end{tabular}

    }
    \\
    \vspace{1mm}
    \scalebox{0.85}{
    \begin{tabular}{cc|c|c|c|c|c}
    \multicolumn{7}{c}{$L_3 = 0.950$} \\
    \hline
    & & $L_2 = 0.900$ & $L_2 = 0.925$ & $L_2 = 0.950$ & $L_2 = 0.975$ & $L_2 = 1.000$ \\
    \hline
    $L_1 = 0.900$ & LB & 0.001 [0.000, 0.008] & 0.001 [0.000, 0.008] & 0.001 [0.000, 0.008] & 0.009 [0.000, 0.033] & 0.033 [0.012, 0.058] \\
    $L_1 = 0.900$ & UB & 0.233 [0.212, 0.258] & 0.208 [0.187, 0.233] & 0.183 [0.162, 0.208] & 0.158 [0.137, 0.183] & 0.133 [0.112, 0.158] \\
        \hline
    $L_1 = 0.925$ & LB & 0.009 [0.000, 0.033] & 0.009 [0.000, 0.033] & 0.009 [0.000, 0.033] & 0.033 [0.012, 0.058] & 0.058 [0.037, 0.083] \\
    $L_1 = 0.925$ & UB & 0.233 [0.212, 0.258] & 0.208 [0.187, 0.233] & 0.183 [0.162, 0.208] & 0.158 [0.137, 0.183] & 0.133 [0.112, 0.158] \\
        \hline
    $L_1 = 0.950$ & LB & 0.033 [0.012, 0.058] & 0.033 [0.012, 0.058] & 0.033 [0.012, 0.058] & 0.058 [0.037, 0.083] & 0.083 [0.062, 0.108] \\
    $L_1 = 0.950$ & UB & 0.233 [0.212, 0.258] & 0.208 [0.187, 0.233] & 0.183 [0.162, 0.208] & 0.158 [0.137, 0.183] & 0.133 [0.112, 0.158] \\
        \hline
    $L_1 = 0.975$ & LB & 0.058 [0.037, 0.083] & 0.058 [0.037, 0.083] & 0.058 [0.037, 0.083] & 0.083 [0.062, 0.108] & 0.108 [0.087, 0.133] \\
    $L_1 = 0.975$ & UB & 0.233 [0.212, 0.258] & 0.208 [0.187, 0.233] & 0.183 [0.162, 0.208] & 0.158 [0.137, 0.183] & 0.133 [0.112, 0.158] \\
        \hline
    $L_1 = 1.000$ & LB & 0.083 [0.062, 0.108] & 0.083 [0.062, 0.108] & 0.083 [0.062, 0.108] & 0.108 [0.087, 0.133] & 0.133 [0.112, 0.158] \\
    $L_1 = 1.000$ & UB & 0.233 [0.212, 0.258] & 0.208 [0.187, 0.233] & 0.183 [0.162, 0.208] & 0.158 [0.137, 0.183] & 0.133 [0.112, 0.158] \\
        \hline
    \end{tabular}
    }
    \\
    \vspace{1mm}
    \scalebox{0.85}{
    \begin{tabular}{cc|c|c|c|c|c}
    \multicolumn{7}{c}{$L_3 = 0.975$} \\
    \hline
    & & $L_2 = 0.900$ & $L_2 = 0.925$ & $L_2 = 0.950$ & $L_2 = 0.975$ & $L_2 = 1.000$ \\
    \hline
    $L_1 = 0.900$ & LB & 0.009 [0.000, 0.033] & 0.009 [0.000, 0.033] & 0.009 [0.000, 0.033] & 0.009 [0.000, 0.033] & 0.033 [0.012, 0.058] \\
    $L_1 = 0.900$ & UB & 0.233 [0.212, 0.258] & 0.208 [0.187, 0.233] & 0.183 [0.162, 0.208] & 0.158 [0.137, 0.183] & 0.133 [0.112, 0.158] \\
        \hline
    $L_1 = 0.925$ & LB & 0.033 [0.012, 0.058] & 0.033 [0.012, 0.058] & 0.033 [0.012, 0.058] & 0.033 [0.012, 0.058] & 0.058 [0.037, 0.083] \\
    $L_1 = 0.925$ & UB & 0.233 [0.212, 0.258] & 0.208 [0.187, 0.233] & 0.183 [0.162, 0.208] & 0.158 [0.137, 0.183] & 0.133 [0.112, 0.158] \\
        \hline
    $L_1 = 0.950$ & LB & 0.058 [0.037, 0.083] & 0.058 [0.037, 0.083] & 0.058 [0.037, 0.083] & 0.058 [0.037, 0.083] & 0.083 [0.062, 0.108] \\
    $L_1 = 0.950$ & UB & 0.233 [0.212, 0.258] & 0.208 [0.187, 0.233] & 0.183 [0.162, 0.208] & 0.158 [0.137, 0.183] & 0.133 [0.112, 0.158] \\
        \hline
    $L_1 = 0.975$ & LB & 0.083 [0.062, 0.108] & 0.083 [0.062, 0.108] & 0.083 [0.062, 0.108] & 0.083 [0.062, 0.108] & 0.108 [0.087, 0.133] \\
    $L_1 = 0.975$ & UB & 0.233 [0.212, 0.258] & 0.208 [0.187, 0.233] & 0.183 [0.162, 0.208] & 0.158 [0.137, 0.183] & 0.133 [0.112, 0.158] \\
        \hline
    $L_1 = 1.000$ & LB & 0.108 [0.087, 0.133] & 0.108 [0.087, 0.133] & 0.108 [0.087, 0.133] & 0.108 [0.087, 0.133] & 0.133 [0.112, 0.158] \\
    $L_1 = 1.000$ & UB & 0.233 [0.212, 0.258] & 0.208 [0.187, 0.233] & 0.183 [0.162, 0.208] & 0.158 [0.137, 0.183] & 0.133 [0.112, 0.158] \\
        \hline
    \end{tabular}
    }
    \\
    \vspace{1mm}
    \scalebox{0.85}{
    \begin{tabular}{cc|c|c|c|c|c}
    \multicolumn{7}{c}{$L_3 = 1.000$} \\
    \hline
    & & $L_2 = 0.900$ & $L_2 = 0.925$ & $L_2 = 0.950$ & $L_2 = 0.975$ & $L_2 = 1.000$ \\
    \hline
    $L_1 = 0.900$ & LB & 0.033 [0.012, 0.058] & 0.033 [0.012, 0.058] & 0.033 [0.012, 0.058] & 0.033 [0.012, 0.058] & 0.033 [0.012, 0.058] \\
    $L_1 = 0.900$ & UB & 0.233 [0.212, 0.258] & 0.208 [0.187, 0.233] & 0.183 [0.162, 0.208] & 0.158 [0.137, 0.183] & 0.133 [0.112, 0.158] \\
        \hline
    $L_1 = 0.925$ & LB & 0.058 [0.037, 0.083] & 0.058 [0.037, 0.083] & 0.058 [0.037, 0.083] & 0.058 [0.037, 0.083] & 0.058 [0.037, 0.083] \\
    $L_1 = 0.925$ & UB & 0.233 [0.212, 0.258] & 0.208 [0.187, 0.233] & 0.183 [0.162, 0.208] & 0.158 [0.137, 0.183] & 0.133 [0.112, 0.158] \\
        \hline
    $L_1 = 0.950$ & LB & 0.083 [0.062, 0.108] & 0.083 [0.062, 0.108] & 0.083 [0.062, 0.108] & 0.083 [0.062, 0.108] & 0.083 [0.062, 0.108] \\
    $L_1 = 0.950$ & UB & 0.233 [0.212, 0.258] & 0.208 [0.187, 0.233] & 0.183 [0.162, 0.208] & 0.158 [0.137, 0.183] & 0.133 [0.112, 0.158] \\
        \hline
    $L_1 = 0.975$ & LB & 0.108 [0.087, 0.133] & 0.108 [0.087, 0.133] & 0.108 [0.087, 0.133] & 0.108 [0.087, 0.133] & 0.108 [0.087, 0.133] \\
    $L_1 = 0.975$ & UB & 0.233 [0.212, 0.258] & 0.208 [0.187, 0.233] & 0.183 [0.162, 0.208] & 0.158 [0.137, 0.183] & 0.133 [0.112, 0.158] \\
        \hline
    $L_1 = 1.000$ & LB & 0.133 [0.112, 0.158] & 0.133 [0.112, 0.158] & 0.133 [0.112, 0.158] & 0.133 [0.112, 0.158] & 0.133 [0.112, 0.158] \\
    $L_1 = 1.000$ & UB & 0.233 [0.212, 0.258] & 0.208 [0.187, 0.233] & 0.183 [0.162, 0.208] & 0.158 [0.137, 0.183] & 0.133 [0.112, 0.158] \\
        \hline
    \end{tabular}

    }
    \\
    \vspace{2mm}
    \scalebox{0.9}{
       Obs.:\quad  LB: 0.000 [0.000,0.000]\quad UB: 0.350 [0.322, 0.380]
       }
\end{table}

\newpage
\begin{table}[H]
    \centering
    \caption{Results of the estimates of bounds of $\mathbb{E}[Y_1-Y_0|X=2,Y=2]$ under Assumption (vi) and $N=1000$ with observational data.
    The true value of $\mathbb{E}[Y_1-Y_0|X=2,Y=2]$ is 0.333.
    }
    \label{tab:L1L2L3_expect_obs}
    \scalebox{0.85}{
    \begin{tabular}{cc|c|c|c|c|c}
    \multicolumn{7}{c}{$L_3 = 0.900$} \\
    \hline
    & & $L_2 = 0.900$ & $L_2 = 0.925$ & $L_2 = 0.950$ & $L_2 = 0.975$ & $L_2 = 1.000$ \\
    \hline
    $L_1 = 0.900$ & LB & -1.929 [-2.000, -1.667] & -1.713 [-1.909, -1.458] & -1.464 [-1.606, -1.250] & -1.214 [-1.303, -1.042] & -0.965 [-1.000, -0.833] \\
    $L_1 = 0.900$ & UB & 1.965 [1.833, 2.000] & 1.965 [1.833, 2.000] & 1.965 [1.833, 2.000] & 1.965 [1.833, 2.000] & 1.965 [1.833, 2.000] \\
        \hline
    $L_1 = 0.925$ & LB & -1.498 [-1.819, -1.250] & -1.498 [-1.819, -1.250] & -1.248 [-1.516, -1.042] & -0.998 [-1.213, -0.833] & -0.749 [-0.909, -0.625] \\
    $L_1 = 0.925$ & UB & 1.749 [1.625, 1.909] & 1.749 [1.625, 1.909] & 1.749 [1.625, 1.909] & 1.749 [1.625, 1.909] & 1.749 [1.625, 1.909] \\
        \hline
    $L_1 = 0.950$ & LB & -0.998 [-1.213, -0.833] & -0.998 [-1.213, -0.833] & -0.998 [-1.213, -0.833] & -0.749 [-0.909, -0.625] & -0.499 [-0.606, -0.417] \\
    $L_1 = 0.950$ & UB & 1.499 [1.417, 1.606] & 1.499 [1.417, 1.606] & 1.499 [1.417, 1.606] & 1.499 [1.417, 1.606] & 1.499 [1.417, 1.606] \\
        \hline
    $L_1 = 0.975$ & LB & -0.499 [-0.606, -0.417] & -0.499 [-0.606, -0.417] & -0.499 [-0.606, -0.417] & -0.499 [-0.606, -0.417] & -0.250 [-0.303, -0.208] \\
    $L_1 = 0.975$ & UB & 1.250 [1.208, 1.303] & 1.250 [1.208, 1.303] & 1.250 [1.208, 1.303] & 1.250 [1.208, 1.303] & 1.250 [1.208, 1.303] \\
        \hline
    $L_1 = 1.000$ & LB & 0.000 [0.000, 0.000] & 0.000 [0.000, 0.000] & 0.000 [0.000, 0.000] & 0.000 [0.000, 0.000] & 0.000 [0.000, 0.000] \\
    $L_1 = 1.000$ & UB & 1.000 [1.000, 1.000] & 1.000 [1.000, 1.000] & 1.000 [1.000, 1.000] & 1.000 [1.000, 1.000] & 1.000 [1.000, 1.000] \\
        \hline
    \end{tabular}
    }
    \\
    \vspace{1mm}
    \scalebox{0.85}{
    \begin{tabular}{cc|c|c|c|c|c}
    \multicolumn{7}{c}{$L_3 = 0.925$} \\
    \hline
    & & $L_2 = 0.900$ & $L_2 = 0.925$ & $L_2 = 0.950$ & $L_2 = 0.975$ & $L_2 = 1.000$ \\
    \hline
    $L_1 = 0.900$ & LB & -1.929 [-2.000, -1.667] & -1.713 [-1.909, -1.458] & -1.464 [-1.606, -1.250] & -1.214 [-1.303, -1.042] & -0.965 [-1.000, -0.833] \\
    $L_1 = 0.900$ & UB & 1.749 [1.625, 1.909] & 1.749 [1.625, 1.909] & 1.749 [1.625, 1.909] & 1.749 [1.625, 1.909] & 1.749 [1.625, 1.909] \\
        \hline
    $L_1 = 0.925$ & LB & -1.498 [-1.819, -1.250] & -1.498 [-1.819, -1.250] & -1.248 [-1.516, -1.042] & -0.998 [-1.213, -0.833] & -0.749 [-0.909, -0.625] \\
    $L_1 = 0.925$ & UB & 1.749 [1.625, 1.909] & 1.749 [1.625, 1.909] & 1.749 [1.625, 1.909] & 1.749 [1.625, 1.909] & 1.749 [1.625, 1.909] \\
        \hline
    $L_1 = 0.950$ & LB & -0.998 [-1.213, -0.833] & -0.998 [-1.213, -0.833] & -0.998 [-1.213, -0.833] & -0.749 [-0.909, -0.625] & -0.499 [-0.606, -0.417] \\
    $L_1 = 0.950$ & UB & 1.499 [1.417, 1.606] & 1.499 [1.417, 1.606] & 1.499 [1.417, 1.606] & 1.499 [1.417, 1.606] & 1.499 [1.417, 1.606] \\
        \hline
    $L_1 = 0.975$ & LB & -0.499 [-0.606, -0.417] & -0.499 [-0.606, -0.417] & -0.499 [-0.606, -0.417] & -0.499 [-0.606, -0.417] & -0.250 [-0.303, -0.208] \\
    $L_1 = 0.975$ & UB & 1.250 [1.208, 1.303] & 1.250 [1.208, 1.303] & 1.250 [1.208, 1.303] & 1.250 [1.208, 1.303] & 1.250 [1.208, 1.303] \\
        \hline
    $L_1 = 1.000$ & LB & 0.000 [0.000, 0.000] & 0.000 [0.000, 0.000] & 0.000 [0.000, 0.000] & 0.000 [0.000, 0.000] & 0.000 [0.000, 0.000] \\
    $L_1 = 1.000$ & UB & 1.000 [1.000, 1.000] & 1.000 [1.000, 1.000] & 1.000 [1.000, 1.000] & 1.000 [1.000, 1.000] & 1.000 [1.000, 1.000] \\
        \hline
    \end{tabular}
    }
    \\
    \vspace{1mm}
    \scalebox{0.85}{
    \begin{tabular}{cc|c|c|c|c|c}
    \multicolumn{7}{c}{$L_3 = 0.950$} \\
    \hline
    & & $L_2 = 0.900$ & $L_2 = 0.925$ & $L_2 = 0.950$ & $L_2 = 0.975$ & $L_2 = 1.000$ \\
    \hline
    $L_1 = 0.900$ & LB & -1.929 [-2.000, -1.667] & -1.713 [-1.909, -1.458] & -1.464 [-1.606, -1.250] & -1.214 [-1.303, -1.042] & -0.965 [-1.000, -0.833] \\
    $L_1 = 0.900$ & UB & 1.499 [1.417, 1.606] & 1.499 [1.417, 1.606] & 1.499 [1.417, 1.606] & 1.499 [1.417, 1.606] & 1.499 [1.417, 1.606] \\
        \hline
    $L_1 = 0.925$ & LB & -1.498 [-1.819, -1.250] & -1.498 [-1.819, -1.250] & -1.248 [-1.516, -1.042] & -0.998 [-1.213, -0.833] & -0.749 [-0.909, -0.625] \\
    $L_1 = 0.925$ & UB & 1.499 [1.417, 1.606] & 1.499 [1.417, 1.606] & 1.499 [1.417, 1.606] & 1.499 [1.417, 1.606] & 1.499 [1.417, 1.606] \\
        \hline
    $L_1 = 0.950$ & LB & -0.998 [-1.213, -0.833] & -0.998 [-1.213, -0.833] & -0.998 [-1.213, -0.833] & -0.749 [-0.909, -0.625] & -0.499 [-0.606, -0.417] \\
    $L_1 = 0.950$ & UB & 1.499 [1.417, 1.606] & 1.499 [1.417, 1.606] & 1.499 [1.417, 1.606] & 1.499 [1.417, 1.606] & 1.499 [1.417, 1.606] \\
        \hline
    $L_1 = 0.975$ & LB & -0.499 [-0.606, -0.417] & -0.499 [-0.606, -0.417] & -0.499 [-0.606, -0.417] & -0.499 [-0.606, -0.417] & -0.250 [-0.303, -0.208] \\
    $L_1 = 0.975$ & UB & 1.250 [1.208, 1.303] & 1.250 [1.208, 1.303] & 1.250 [1.208, 1.303] & 1.250 [1.208, 1.303] & 1.250 [1.208, 1.303] \\
        \hline
    $L_1 = 1.000$ & LB & 0.000 [0.000, 0.000] & 0.000 [0.000, 0.000] & 0.000 [0.000, 0.000] & 0.000 [0.000, 0.000] & 0.000 [0.000, 0.000] \\
    $L_1 = 1.000$ & UB & 1.000 [1.000, 1.000] & 1.000 [1.000, 1.000] & 1.000 [1.000, 1.000] & 1.000 [1.000, 1.000] & 1.000 [1.000, 1.000] \\
        \hline
    \end{tabular}
    }
    \\
    \vspace{1mm}
    \scalebox{0.85}{
    \begin{tabular}{cc|c|c|c|c|c}
    \multicolumn{7}{c}{$L_3 = 0.975$} \\
    \hline
    & & $L_2 = 0.900$ & $L_2 = 0.925$ & $L_2 = 0.950$ & $L_2 = 0.975$ & $L_2 = 1.000$ \\
    \hline
    $L_1 = 0.900$ & LB & -1.929 [-2.000, -1.667] & -1.713 [-1.909, -1.458] & -1.464 [-1.606, -1.250] & -1.214 [-1.303, -1.042] & -0.965 [-1.000, -0.833] \\
    $L_1 = 0.900$ & UB & 1.250 [1.208, 1.303] & 1.250 [1.208, 1.303] & 1.250 [1.208, 1.303] & 1.250 [1.208, 1.303] & 1.250 [1.208, 1.303] \\
        \hline
    $L_1 = 0.925$ & LB & -1.498 [-1.819, -1.250] & -1.498 [-1.819, -1.250] & -1.248 [-1.516, -1.042] & -0.998 [-1.213, -0.833] & -0.749 [-0.909, -0.625] \\
    $L_1 = 0.925$ & UB & 1.250 [1.208, 1.303] & 1.250 [1.208, 1.303] & 1.250 [1.208, 1.303] & 1.250 [1.208, 1.303] & 1.250 [1.208, 1.303] \\
        \hline
    $L_1 = 0.950$ & LB & -0.998 [-1.213, -0.833] & -0.998 [-1.213, -0.833] & -0.998 [-1.213, -0.833] & -0.749 [-0.909, -0.625] & -0.499 [-0.606, -0.417] \\
    $L_1 = 0.950$ & UB & 1.250 [1.208, 1.303] & 1.250 [1.208, 1.303] & 1.250 [1.208, 1.303] & 1.250 [1.208, 1.303] & 1.250 [1.208, 1.303] \\
        \hline
    $L_1 = 0.975$ & LB & -0.499 [-0.606, -0.417] & -0.499 [-0.606, -0.417] & -0.499 [-0.606, -0.417] & -0.499 [-0.606, -0.417] & -0.250 [-0.303, -0.208] \\
    $L_1 = 0.975$ & UB & 1.250 [1.208, 1.303] & 1.250 [1.208, 1.303] & 1.250 [1.208, 1.303] & 1.250 [1.208, 1.303] & 1.250 [1.208, 1.303] \\
        \hline
    $L_1 = 1.000$ & LB & 0.000 [0.000, 0.000] & 0.000 [0.000, 0.000] & 0.000 [0.000, 0.000] & 0.000 [0.000, 0.000] & 0.000 [0.000, 0.000] \\
    $L_1 = 1.000$ & UB & 1.000 [1.000, 1.000] & 1.000 [1.000, 1.000] & 1.000 [1.000, 1.000] & 1.000 [1.000, 1.000] & 1.000 [1.000, 1.000] \\
        \hline
    \end{tabular}
    }
    \\
    \vspace{1mm}
    \scalebox{0.85}{
    \begin{tabular}{cc|c|c|c|c|c}
    \multicolumn{7}{c}{$L_3 = 1.000$} \\
    \hline
    & & $L_2 = 0.900$ & $L_2 = 0.925$ & $L_2 = 0.950$ & $L_2 = 0.975$ & $L_2 = 1.000$ \\
    \hline
    $L_1 = 0.900$ & LB & -1.929 [-2.000, -1.667] & -1.713 [-1.909, -1.458] & -1.464 [-1.606, -1.250] & -1.214 [-1.303, -1.042] & -0.965 [-1.000, -0.833] \\
    $L_1 = 0.900$ & UB & 1.000 [1.000, 1.000] & 1.000 [1.000, 1.000] & 1.000 [1.000, 1.000] & 1.000 [1.000, 1.000] & 1.000 [1.000, 1.000] \\
        \hline
    $L_1 = 0.925$ & LB & -1.498 [-1.819, -1.250] & -1.498 [-1.819, -1.250] & -1.248 [-1.516, -1.042] & -0.998 [-1.213, -0.833] & -0.749 [-0.909, -0.625] \\
    $L_1 = 0.925$ & UB & 1.000 [1.000, 1.000] & 1.000 [1.000, 1.000] & 1.000 [1.000, 1.000] & 1.000 [1.000, 1.000] & 1.000 [1.000, 1.000] \\
        \hline
    $L_1 = 0.950$ & LB & -0.998 [-1.213, -0.833] & -0.998 [-1.213, -0.833] & -0.998 [-1.213, -0.833] & -0.749 [-0.909, -0.625] & -0.499 [-0.606, -0.417] \\
    $L_1 = 0.950$ & UB & 1.000 [1.000, 1.000] & 1.000 [1.000, 1.000] & 1.000 [1.000, 1.000] & 1.000 [1.000, 1.000] & 1.000 [1.000, 1.000] \\
        \hline
    $L_1 = 0.975$ & LB & -0.499 [-0.606, -0.417] & -0.499 [-0.606, -0.417] & -0.499 [-0.606, -0.417] & -0.499 [-0.606, -0.417] & -0.250 [-0.303, -0.208] \\
    $L_1 = 0.975$ & UB & 1.000 [1.000, 1.000] & 1.000 [1.000, 1.000] & 1.000 [1.000, 1.000] & 1.000 [1.000, 1.000] & 1.000 [1.000, 1.000] \\
        \hline
    $L_1 = 1.000$ & LB & 0.000 [0.000, 0.000] & 0.000 [0.000, 0.000] & 0.000 [0.000, 0.000] & 0.000 [0.000, 0.000] & 0.000 [0.000, 0.000] \\
    $L_1 = 1.000$ & UB & 1.000 [1.000, 1.000] & 1.000 [1.000, 1.000] & 1.000 [1.000, 1.000] & 1.000 [1.000, 1.000] & 1.000 [1.000, 1.000] \\
        \hline
    \end{tabular}
    }
\end{table}

\newpage
\begin{table}[H]
    \centering
    \caption{Results of the estimates of bounds of $\mathbb{E}[Y_1-Y_0|X=2,Y=2]$ under Assumption (vi) and $N=1000$ with both experimental and observational data.
    The true value of $\mathbb{E}[Y_1-Y_0|X=2,Y=2]$ is 0.333.
    }
    \label{tab:L1L2L3_expect_both}
    \scalebox{0.85}{
    \begin{tabular}{cc|c|c|c|c|c}
    \multicolumn{7}{c}{$L_3 = 0.900$} \\
    \hline
    & & $L_2 = 0.900$ & $L_2 = 0.925$ & $L_2 = 0.950$ & $L_2 = 0.975$ & $L_2 = 1.000$ \\
    \hline
    $L_1 = 0.900$ & LB & -1.758 [-2.000, -1.436] & -1.556 [-1.835, -1.241] & -1.307 [-1.546, -1.022] & -1.057 [-1.268, -0.795] & -0.808 [-1.000, -0.564] \\
    $L_1 = 0.900$ & UB & 1.926 [1.767, 2.000] & 1.926 [1.767, 2.000] & 1.926 [1.767, 2.000] & 1.926 [1.767, 2.000] & 1.917 [1.752, 2.000] \\
        \hline
    $L_1 = 0.925$ & LB & -1.475 [-1.769, -1.248] & -1.461 [-1.769, -1.198] & -1.216 [-1.467, -0.998] & -0.966 [-1.164, -0.788] & -0.716 [-0.862, -0.564] \\
    $L_1 = 0.925$ & UB & 1.744 [1.625, 1.909] & 1.744 [1.625, 1.909] & 1.744 [1.625, 1.909] & 1.744 [1.625, 1.909] & 1.727 [1.551, 1.909] \\
        \hline
    $L_1 = 0.950$ & LB & -0.998 [-1.213, -0.833] & -0.998 [-1.213, -0.833] & -0.998 [-1.213, -0.833] & -0.749 [-0.909, -0.625] & -0.499 [-0.606, -0.417] \\
    $L_1 = 0.950$ & UB & 1.499 [1.417, 1.606] & 1.499 [1.417, 1.606] & 1.499 [1.417, 1.606] & 1.499 [1.417, 1.606] & 1.482 [1.321, 1.606] \\
        \hline
    $L_1 = 0.975$ & LB & -0.499 [-0.606, -0.417] & -0.499 [-0.606, -0.417] & -0.499 [-0.606, -0.417] & -0.499 [-0.606, -0.417] & -0.250 [-0.303, -0.208] \\
    $L_1 = 0.975$ & UB & 1.250 [1.208, 1.303] & 1.250 [1.208, 1.303] & 1.250 [1.208, 1.303] & 1.249 [1.208, 1.303] & 1.232 [1.092, 1.303] \\
        \hline
    $L_1 = 1.000$ & LB & 0.000 [0.000, 0.000] & 0.000 [0.000, 0.000] & 0.000 [0.000, 0.000] & 0.000 [0.000, 0.000] & 0.000 [0.000, 0.000] \\
    $L_1 = 1.000$ & UB & 1.000 [1.000, 1.000] & 1.000 [1.000, 1.000] & 1.000 [1.000, 1.000] & 0.999 [1.000, 1.000] & 0.983 [0.844, 1.000] \\
        \hline
    \end{tabular}
    }
    \\
    \vspace{1mm}
    \scalebox{0.85}{
    \begin{tabular}{cc|c|c|c|c|c}
    \multicolumn{7}{c}{$L_3 = 0.925$} \\
    \hline
    & & $L_2 = 0.900$ & $L_2 = 0.925$ & $L_2 = 0.950$ & $L_2 = 0.975$ & $L_2 = 1.000$ \\
    \hline
    $L_1 = 0.900$ & LB & -1.758 [-2.000, -1.436] & -1.556 [-1.835, -1.241] & -1.307 [-1.546, -1.022] & -1.057 [-1.268, -0.795] & -0.808 [-1.000, -0.564] \\
    $L_1 = 0.900$ & UB & 1.745 [1.625, 1.909] & 1.745 [1.625, 1.909] & 1.745 [1.625, 1.909] & 1.745 [1.625, 1.909] & 1.745 [1.625, 1.909] \\
        \hline
    $L_1 = 0.925$ & LB & -1.475 [-1.769, -1.248] & -1.461 [-1.769, -1.198] & -1.216 [-1.467, -0.998] & -0.966 [-1.164, -0.788] & -0.716 [-0.862, -0.564] \\
    $L_1 = 0.925$ & UB & 1.704 [1.540, 1.909] & 1.704 [1.540, 1.909] & 1.704 [1.540, 1.909] & 1.704 [1.540, 1.909] & 1.695 [1.530, 1.909] \\
        \hline
    $L_1 = 0.950$ & LB & -0.998 [-1.213, -0.833] & -0.998 [-1.213, -0.833] & -0.998 [-1.213, -0.833] & -0.749 [-0.909, -0.625] & -0.499 [-0.606, -0.417] \\
    $L_1 = 0.950$ & UB & 1.496 [1.416, 1.606] & 1.496 [1.416, 1.606] & 1.496 [1.416, 1.606] & 1.495 [1.413, 1.606] & 1.479 [1.321, 1.606] \\
        \hline
    $L_1 = 0.975$ & LB & -0.499 [-0.606, -0.417] & -0.499 [-0.606, -0.417] & -0.499 [-0.606, -0.417] & -0.499 [-0.606, -0.417] & -0.250 [-0.303, -0.208] \\
    $L_1 = 0.975$ & UB & 1.250 [1.208, 1.303] & 1.250 [1.208, 1.303] & 1.250 [1.208, 1.303] & 1.249 [1.208, 1.303] & 1.232 [1.092, 1.303] \\
        \hline
    $L_1 = 1.000$ & LB & 0.000 [0.000, 0.000] & 0.000 [0.000, 0.000] & 0.000 [0.000, 0.000] & 0.000 [0.000, 0.000] & 0.000 [0.000, 0.000] \\
    $L_1 = 1.000$ & UB & 1.000 [1.000, 1.000] & 1.000 [1.000, 1.000] & 1.000 [1.000, 1.000] & 0.999 [1.000, 1.000] & 0.983 [0.844, 1.000] \\
        \hline
    \end{tabular}
    }
    \\
    \vspace{1mm}
    \scalebox{0.85}{
    \begin{tabular}{cc|c|c|c|c|c}
    \multicolumn{7}{c}{$L_3 = 0.950$} \\
    \hline
    & & $L_2 = 0.900$ & $L_2 = 0.925$ & $L_2 = 0.950$ & $L_2 = 0.975$ & $L_2 = 1.000$ \\
    \hline
    $L_1 = 0.900$ & LB & -1.758 [-2.000, -1.436] & -1.556 [-1.835, -1.241] & -1.307 [-1.546, -1.022] & -1.057 [-1.268, -0.795] & -0.808 [-1.000, -0.564] \\
    $L_1 = 0.900$ & UB & 1.499 [1.417, 1.606] & 1.499 [1.417, 1.606] & 1.499 [1.417, 1.606] & 1.499 [1.417, 1.606] & 1.499 [1.417, 1.606] \\
        \hline
    $L_1 = 0.925$ & LB & -1.475 [-1.769, -1.248] & -1.461 [-1.769, -1.198] & -1.216 [-1.467, -0.998] & -0.966 [-1.164, -0.788] & -0.716 [-0.862, -0.564] \\
    $L_1 = 0.925$ & UB & 1.496 [1.416, 1.606] & 1.496 [1.416, 1.606] & 1.496 [1.416, 1.606] & 1.496 [1.416, 1.606] & 1.495 [1.413, 1.606] \\
        \hline
    $L_1 = 0.950$ & LB & -0.998 [-1.213, -0.833] & -0.998 [-1.213, -0.833] & -0.998 [-1.213, -0.833] & -0.749 [-0.909, -0.625] & -0.499 [-0.606, -0.417] \\
    $L_1 = 0.950$ & UB & 1.454 [1.314, 1.606] & 1.454 [1.314, 1.606] & 1.454 [1.314, 1.606] & 1.454 [1.314, 1.606] & 1.446 [1.292, 1.606] \\
        \hline
    $L_1 = 0.975$ & LB & -0.499 [-0.606, -0.417] & -0.499 [-0.606, -0.417] & -0.499 [-0.606, -0.417] & -0.499 [-0.606, -0.417] & -0.250 [-0.303, -0.208] \\
    $L_1 = 0.975$ & UB & 1.246 [1.200, 1.303] & 1.246 [1.200, 1.303] & 1.246 [1.200, 1.303] & 1.246 [1.186, 1.303] & 1.229 [1.092, 1.303] \\
        \hline
    $L_1 = 1.000$ & LB & 0.000 [0.000, 0.000] & 0.000 [0.000, 0.000] & 0.000 [0.000, 0.000] & 0.000 [0.000, 0.000] & 0.000 [0.000, 0.000] \\
    $L_1 = 1.000$ & UB & 1.000 [1.000, 1.000] & 1.000 [1.000, 1.000] & 1.000 [1.000, 1.000] & 0.999 [1.000, 1.000] & 0.983 [0.844, 1.000] \\
        \hline
    \end{tabular}
    }
    \\
    \vspace{1mm}
    \scalebox{0.85}{
    \begin{tabular}{cc|c|c|c|c|c}
    \multicolumn{7}{c}{$L_3 = 0.975$} \\
    \hline
    & & $L_2 = 0.900$ & $L_2 = 0.925$ & $L_2 = 0.950$ & $L_2 = 0.975$ & $L_2 = 1.000$ \\
    \hline
    $L_1 = 0.900$ & LB & -1.758 [-2.000, -1.436] & -1.556 [-1.835, -1.241] & -1.307 [-1.546, -1.022] & -1.057 [-1.268, -0.795] & -0.808 [-1.000, -0.564] \\
    $L_1 = 0.900$ & UB & 1.250 [1.208, 1.303] & 1.250 [1.208, 1.303] & 1.250 [1.208, 1.303] & 1.250 [1.208, 1.303] & 1.250 [1.208, 1.303] \\
        \hline
    $L_1 = 0.925$ & LB & -1.475 [-1.769, -1.248] & -1.461 [-1.769, -1.198] & -1.216 [-1.467, -0.998] & -0.966 [-1.164, -0.788] & -0.716 [-0.862, -0.564] \\
    $L_1 = 0.925$ & UB & 1.250 [1.208, 1.303] & 1.250 [1.208, 1.303] & 1.250 [1.208, 1.303] & 1.250 [1.208, 1.303] & 1.250 [1.208, 1.303] \\
        \hline
    $L_1 = 0.950$ & LB & -0.998 [-1.213, -0.833] & -0.998 [-1.213, -0.833] & -0.998 [-1.213, -0.833] & -0.749 [-0.909, -0.625] & -0.499 [-0.606, -0.417] \\
    $L_1 = 0.950$ & UB & 1.246 [1.200, 1.303] & 1.246 [1.200, 1.303] & 1.246 [1.200, 1.303] & 1.246 [1.200, 1.303] & 1.246 [1.186, 1.303] \\
        \hline
    $L_1 = 0.975$ & LB & -0.499 [-0.606, -0.417] & -0.499 [-0.606, -0.417] & -0.499 [-0.606, -0.417] & -0.499 [-0.606, -0.417] & -0.250 [-0.303, -0.208] \\
    $L_1 = 0.975$ & UB & 1.205 [1.077, 1.303] & 1.205 [1.077, 1.303] & 1.205 [1.077, 1.303] & 1.205 [1.077, 1.303] & 1.196 [1.046, 1.303] \\
        \hline
    $L_1 = 1.000$ & LB & 0.000 [0.000, 0.000] & 0.000 [0.000, 0.000] & 0.000 [0.000, 0.000] & 0.000 [0.000, 0.000] & 0.000 [0.000, 0.000] \\
    $L_1 = 1.000$ & UB & 0.993 [0.908, 1.000] & 0.993 [0.908, 1.000] & 0.993 [0.908, 1.000] & 0.993 [0.908, 1.000] & 0.976 [0.836, 1.000] \\
        \hline
    \end{tabular}%
    }
    
\end{table}

\newpage
\begin{table}[H]
    \centering
    \caption{Results of the estimates of bounds of $\mathbb{E}[(Y_1-Y_0)^2]$ under Assumption (vi) and $N=1000$ with experimental data / with both experimental and observational data.
    The true value of $\mathbb{E}[(Y_1-Y_0)^2]$ is 0.267.
    }
    \label{tab:L1L2L3_2moment_expboth}
    \scalebox{0.85}{
    %
    }
    
\end{table}

\newpage
\begin{table}[H]
    \centering
    \caption{Results of the estimates of bounds of $\mathbb{E}[(Y_1-Y_0)^2]$ under Assumption (vi) and $N=1000$ with observational data.
    The true value of $\mathbb{E}[(Y_1-Y_0)^2]$ is 0.267.
    }
    \label{tab:L1L2L3_2moment_obs}
    \scalebox{0.85}{
    %
    }
    
\end{table}

\newpage
\subsection{Under Identification Assumptions}
\label{appC2}

Next, we illustrate the properties of the point estimator under an identification assumption.
There are no baselines because no other identification assumptions exist for the joint POs/OVs probabilities  with discrete treatments and outcomes.

{\bf Settings.}
We set the values of ${\cal P}$ as presented in Table \ref{tab:p(y0y1y2xy)_exo}.
For this purpose, we impose the following identification assumptions.
\begin{enumerate}[label=(\roman*).]
  \setlength\itemsep{2pt}
  \setlength\parskip{0pt}
  \setlength\parsep{0pt}
\item[(I).]  $0\leq Y_s-Y_t\leq 1$ almost surely for any $s>t$ (Assumption \ref{MITE_D}) for experimental data,
\item[(II).]  $0\leq Y_s-Y_t\leq 1$ almost surely for any $s>t$ (Assumption \ref{MITE_D}) and the exogeneity (Assumption \ref{exo}) for observational data.
\end{enumerate}
The setting in Table \ref{tab:p(y0y1y2xy)_exo} satisfies the above identification assumptions.

{\bf Results.}
Table \ref{tab:iden_results} (resp. \ref{tab:iden_results2}, \ref{tab:iden_results_2ndmoment}) reports the estimates of the values of $\mathbb{P}(Y_0=0, Y_1=0, Y_2=1)$ (resp. $\mathbb{E}[Y_1-Y_0|X=2,Y=2]$, $\mathbb{E}[(Y_1-Y_0)^2]$) with sample size $N=10, 100, 1000, 10000$.
All estimators converge to the ground truth, and their 95\% CIs become tighter as the sample size increases.
When the sample size is relatively large (e.g., $N=1000$ or $N=10000$), the estimates are reliable.

\newpage
\begin{table}[H]
    \centering
    \caption{Setting of Simulations}
    \label{tab:p(y0y1y2xy)_exo}
    {\small
    \begin{tabular}{c|ccc|ccc|ccc}
    \hline
        $X=0,Y=0$ & \multicolumn{3}{c|}{$Y_0=0$} & \multicolumn{3}{c|}{$Y_0=1$} & \multicolumn{3}{c}{$Y_0=2$}\\
        \hline
          & $Y_1=0$ & $Y_1=1$ & $Y_1=2$ & $Y_1=0$ & $Y_1=1$ & $Y_1=2$ & $Y_1=0$ & $Y_1=1$ & $Y_1=2$ \\
         \hline
        $Y_2=0$ & 1/21 & 0 & 0 & 0 & 0 & 0 & 0 & 0 & 0  \\ 
        $Y_2=1$ & 1/21 & 1/21 & 0 & 0 & 0 & 0 & 0 & 0 & 0 \\
        $Y_2=2$ & 0 & 0 & 0 & 0 & 0 & 0 & 0 & 0 & 0  \\
        \hline
    \end{tabular}   
    \\
     \vspace{1mm}
    \begin{tabular}{c|ccc|ccc|ccc}
    \hline
        $X=0,Y=1$ & \multicolumn{3}{c|}{$Y_0=0$} & \multicolumn{3}{c|}{$Y_0=1$} & \multicolumn{3}{c}{$Y_0=2$}\\
        \hline
          & $Y_1=0$ & $Y_1=1$ & $Y_1=2$ & $Y_1=0$ & $Y_1=1$ & $Y_1=2$ & $Y_1=0$ & $Y_1=1$ & $Y_1=2$ \\
         \hline
        $Y_2=0$ & 0 & 0 & 0 & 0 & 0 & 0 & 0 & 0 & 0  \\ 
        $Y_2=1$ & 0 & 0 & 0 & 0 & 1/21 & 0 & 0 & 0 & 0 \\
        $Y_2=2$ & 0 & 0 & 0 & 0 & 1/21 & 1/21 & 0 & 0 & 0  \\
        \hline
    \end{tabular}
    \\
     \vspace{1mm}
    \begin{tabular}{c|ccc|ccc|ccc}
    \hline
        $X=0,Y=2$ & \multicolumn{3}{c|}{$Y_0=0$} & \multicolumn{3}{c|}{$Y_0=1$} & \multicolumn{3}{c}{$Y_0=2$}\\
        \hline
          & $Y_1=0$ & $Y_1=1$ & $Y_1=2$ & $Y_1=0$ & $Y_1=1$ & $Y_1=2$ & $Y_1=0$ & $Y_1=1$ & $Y_1=2$ \\
         \hline
        $Y_2=0$ & 0 & 0 & 0 & 0 & 0 & 0 & 0 & 0 & 0  \\ 
        $Y_2=1$ & 0 & 0 & 0 & 0 & 0 & 0 & 0 & 0 & 0 \\
        $Y_2=2$ & 0 & 0 & 0 & 0 & 0 & 0 & 0 & 0 & 1/21  \\
        \hline
    \end{tabular}
    \\
     \vspace{1mm}
    \begin{tabular}{c|ccc|ccc|ccc}
    \hline
        $X=1,Y=0$ & \multicolumn{3}{c|}{$Y_0=0$} & \multicolumn{3}{c|}{$Y_0=1$} & \multicolumn{3}{c}{$Y_0=2$}\\
        \hline
          & $Y_1=0$ & $Y_1=1$ & $Y_1=2$ & $Y_1=0$ & $Y_1=1$ & $Y_1=2$ & $Y_1=0$ & $Y_1=1$ & $Y_1=2$ \\
         \hline
        $Y_2=0$ & 1/21 & 0 & 0 & 0 & 0 & 0 & 0 & 0 & 0  \\ 
        $Y_2=1$ & 1/21 & 0 & 0 & 0 & 0 & 0 & 0 & 0 & 0 \\
        $Y_2=2$ & 0 & 0 & 0 & 0 & 0 & 0 & 0 & 0 & 0  \\
        \hline
    \end{tabular}
    \\
     \vspace{1mm}
    \begin{tabular}{c|ccc|ccc|ccc}
    \hline
        $X=1,Y=1$ & \multicolumn{3}{c|}{$Y_0=0$} & \multicolumn{3}{c|}{$Y_0=1$} & \multicolumn{3}{c}{$Y_0=2$}\\
        \hline
          & $Y_1=0$ & $Y_1=1$ & $Y_1=2$ & $Y_1=0$ & $Y_1=1$ & $Y_1=2$ & $Y_1=0$ & $Y_1=1$ & $Y_1=2$ \\
         \hline
        $Y_2=0$ & 0 & 0 & 0 & 0 & 0 & 0 & 0 & 0 & 0  \\ 
        $Y_2=1$ & 0 & 1/21 & 0 & 0 & 1/21 & 0 & 0 & 0 & 0 \\
        $Y_2=2$ & 0 & 0 & 0 & 0 & 1/21 & 0 & 0 & 0 & 0 \\
        \hline
    \end{tabular}
    \\
     \vspace{1mm}
    \begin{tabular}{c|ccc|ccc|ccc}
    \hline
        $X=1,Y=2$ & \multicolumn{3}{c|}{$Y_0=0$} & \multicolumn{3}{c|}{$Y_0=1$} & \multicolumn{3}{c}{$Y_0=2$}\\
        \hline
          & $Y_1=0$ & $Y_1=1$ & $Y_1=2$ & $Y_1=0$ & $Y_1=1$ & $Y_1=2$ & $Y_1=0$ & $Y_1=1$ & $Y_1=2$ \\
         \hline
        $Y_2=0$ & 0 & 0 & 0 & 0 & 0 & 0 & 0 & 0 & 0  \\ 
        $Y_2=1$ & 0 & 0 & 0 & 0 & 0 & 0 & 0 & 0 & 0 \\
        $Y_2=2$ & 0 & 0 & 0 & 0 & 0 & 1/21 & 0 & 0 & 1/21  \\
        \hline
    \end{tabular}
    \\
     \vspace{1mm}
    \begin{tabular}{c|ccc|ccc|ccc}
    \hline
        $X=2,Y=0$ & \multicolumn{3}{c|}{$Y_0=0$} & \multicolumn{3}{c|}{$Y_0=1$} & \multicolumn{3}{c}{$Y_0=2$}\\
        \hline
          & $Y_1=0$ & $Y_1=1$ & $Y_1=2$ & $Y_1=0$ & $Y_1=1$ & $Y_1=2$ & $Y_1=0$ & $Y_1=1$ & $Y_1=2$ \\
         \hline
        $Y_2=0$ & 1/21 & 0 & 0 & 0 & 0 & 0 & 0 & 0 & 0  \\ 
        $Y_2=1$ & 0 & 0 & 0 & 0 & 0 & 0 & 0 & 0 & 0 \\
        $Y_2=2$ & 0 & 0 & 0 & 0 & 0 & 0 & 0 & 0 & 0  \\
        \hline
    \end{tabular}
    \\
     \vspace{1mm}
    \begin{tabular}{c|ccc|ccc|ccc}
    \hline
        $X=2,Y=1$ & \multicolumn{3}{c|}{$Y_0=0$} & \multicolumn{3}{c|}{$Y_0=1$} & \multicolumn{3}{c}{$Y_0=2$}\\
        \hline
          & $Y_1=0$ & $Y_1=1$ & $Y_1=2$ & $Y_1=0$ & $Y_1=1$ & $Y_1=2$ & $Y_1=0$ & $Y_1=1$ & $Y_1=2$ \\
         \hline
        $Y_2=0$ & 0 & 0 & 0 & 0 & 0 & 0 & 0 & 0 & 0  \\ 
        $Y_2=1$ & 1/21 & 1/21 & 0 & 0 & 1/21 & 0 & 0 & 0 & 0 \\
        $Y_2=2$ & 0 & 0 & 0 & 0 & 0 & 0 & 0 & 0 & 0  \\
        \hline
    \end{tabular}
    \\
     \vspace{1mm}
    \begin{tabular}{c|ccc|ccc|ccc}
    \hline
        $X=2,Y=2$ & \multicolumn{3}{c|}{$Y_0=0$} & \multicolumn{3}{c|}{$Y_0=1$} & \multicolumn{3}{c}{$Y_0=2$}\\
        \hline
          & $Y_1=0$ & $Y_1=1$ & $Y_1=2$ & $Y_1=0$ & $Y_1=1$ & $Y_1=2$ & $Y_1=0$ & $Y_1=1$ & $Y_1=2$ \\
         \hline
        $Y_2=0$ & 0 & 0 & 0 & 0 & 0 & 0 & 0 & 0 & 0  \\ 
        $Y_2=1$ & 0 & 0 & 0 & 0 & 0 & 0 & 0 & 0 & 0 \\
        $Y_2=2$ & 0 & 0 & 0 & 0 & 1/21 & 1/21 & 0 & 0 & 1/21  \\
        \hline
    \end{tabular}
    }
\end{table}

\begin{table}[H]
    \centering
    \caption{Results of the estimates of $\mathbb{P}(Y_0=0,Y_1=0,Y_2=1)$ under identification assumptions.
    We report the means and 95\% confidence intervals (CIs).
    The ground truth of $\mathbb{P}(Y_0=0,Y_1=0,Y_2=1)$ is 0.143.
    }
    \label{tab:iden_results}
    \scalebox{0.9}{
    \begin{tabular}{cc|cccc}
    \hline 
    MAs & Data & $N=10$ & $N=100$
    & $N=1000$
    & $N=10000$ \\
     \hline \hline
    (I) &  Exp. & 0.142 [0.000,0.353] & 0.145 [0.085,0.225] & 0.142 [0.120,0.166] & 0.143 [0.135,0.150] \\
    \hline 
    (II) &  Obs.& 0.190 [0.000,0.500] & 0.143 [0.020,0.309] & 0.152 [0.085,0.226] & 0.144 [0.128,0.157] \\
    \hline 
    \end{tabular}
    }
\end{table}

\newpage
\begin{table}[H]
    \centering
    \caption{Results of the estimates of $\mathbb{E}[Y_1-Y_0|X=2,Y=2]$ under identification assumptions.
    We report the means and 95\% confidence intervals (CIs).
    The ground truth of $\mathbb{E}[Y_1-Y_0|X=2,Y=2]$ is 0.333.
    }
    \label{tab:iden_results2}
    \scalebox{0.9}{
    \begin{tabular}{cc|cccc}
    \hline 
    MAs & Data & $N=10$ & $N=100$
    & $N=1000$
    & $N=10000$ \\
     \hline \hline
    (II)  &  Obs. & 0.431 [0.000,1.000] & 0.332 [0.042,0.700] & 0.321 [0.202,0.465] & 0.334 [0.288,0.382] \\
    \hline 
    \end{tabular}
    }
\end{table}

\begin{table}[H]
    \centering
    \caption{Results of the estimates of $\mathbb{E}[(Y_1-Y_0)^2]$ under identification assumptions.
    We report the means and 95\% confidence intervals (CIs).
    The ground truth of $\mathbb{E}[(Y_1-Y_0)^2]$ is 0.286.
    }
    \label{tab:iden_results_2ndmoment}
    \scalebox{0.9}{
    \begin{tabular}{cc|cccc}
    \hline 
    MAs & Data & $N=10$ & $N=100$
    & $N=1000$
    & $N=10000$ \\
     \hline \hline
    (I) &  Exp. & 0.313 [0.000,0.600] & 0.279 [0.200,0.350] & 0.288 [0.264,0.314] & 0.285 [0.278,0.294] \\
    \hline 
    (II) &  Obs.& 0.395 [0.000,1.000] & 0.303 [0.077,0.561] & 0.285 [0.169,0.388] & 0.282 [0.238,0.319] \\
    \hline 
    \end{tabular}
    }
\end{table}

\end{document}